\documentclass{article}
\usepackage{fullpage}

\usepackage[utf8]{inputenc} 
\usepackage[T1]{fontenc}    
\usepackage{hyperref}       
\usepackage{url}            
\usepackage{booktabs}       
\usepackage{algorithm}
\usepackage{amsfonts}       
\usepackage{nicefrac}       
\usepackage{microtype}      
\usepackage[utf8]{inputenc} 
\usepackage[T1]{fontenc}    
\usepackage{hyperref}       
\usepackage{booktabs}       
\usepackage{amsfonts}       
\usepackage{nicefrac}       
\usepackage{microtype}      
\usepackage{lipsum}
\usepackage{graphicx}
\usepackage{float}
\usepackage{comment}

\usepackage{rotating}
\usepackage{amsmath}
\newcounter{ToDo}
\usepackage{subcaption}
\usepackage{amsthm}
\newtheorem{thm}{Theorem} 
\newtheorem{lem}{Lemma}
\newtheorem{prop}{Proposition}

\newtheorem{defn}{Definition}
\newtheorem{rem}{Remark}
\usepackage{hyperref}
\usepackage{xcolor}
\usepackage[normalem]{ulem}

\usepackage{algorithmic}
\usepackage{pgfplots}
\usepgfplotslibrary{patchplots}
\usetikzlibrary{patterns, positioning, arrows}
\pgfplotsset{compat=1.15}
\usepackage[symbol]{footmisc}
\date{}

\newcounter{Note}
\definecolor{blue-violet}{rgb}{0.54, 0.17, 0.89}
\definecolor{mygreen}{rgb}{0.0, 0.5, 0.0}
\definecolor{awesome}{rgb}{1.0, 0.13, 0.32}
\definecolor{bostonuniversityred}{rgb}{0.8, 0.0, 0.0}

\usepackage[normalem]{ulem}
\usepackage{color}
\newcounter{guocomm}

\title{A Discussion On the Validity of Manifold Learning}

\author{%
  Dai Shi\footnote{School of Computer, Data and Mathematical Sciences, Western Sydney University, 20195423@student.westernsydney.edu.au  } \quad Andi Han\footnote{Discipline of Business Analytics, University of Sydney. Email: \{andi.han, junbin.gao\}@sydney.edu.au. \protect\label{usyd}} \quad  Yi Guo\footnote{Corresponding Author, School of Computer, Data and Mathematical Sciences, Y.Guo@westernsydney.edu.au} \quad  Junbin Gao\footref{usyd}
  
}

\begin{document}

\maketitle

\begin{abstract}
Dimensionality reduction (DR) and manifold learning (ManL) have been applied extensively in many machine learning tasks, including signal processing, speech recognition, and neuroinformatics.However, the understanding of whether DR and ManL models can generate valid learning results remains unclear. In this work, we investigate the validity of learning results of some widely used DR and ManL methods through the chart mapping function of a manifold. We identify a fundamental problem of these methods: the mapping functions induced by these methods violate the basic settings of manifolds, and hence they are not learning manifold in the mathematical sense. To address this problem, we provide a provably correct algorithm called fixed points Laplacian mapping (FPLM), that has the geometric guarantee to find a valid manifold representation (up to a homeomorphism). Combining one additional condition (orientation preserving), we discuss a sufficient condition for an algorithm to be bijective for any $d$-simplex decomposition result on a $d$-manifold.  However,  constructing such a mapping function and its computational method satisfying these conditions is still an open problem in mathematics.
\end{abstract}

\section{Introduction}\label{introduction}

    

Dimensionality reduction (DR) and manifold learning (ManL) have been widely applied in many machine learning tasks, such as signal processing \cite{rui2016dimensionality}, speech recognition \cite{van2009dimensionality}, neuroinformatics \cite{mwangi2014review} and bioinformatics \cite{zou2016novel}. The reason is the well-accepted manifold assumption: \textit{the observed data distribute in a non-linear low dimensional manifold embedded in a high dimensional ambient space}. The main focus of various ManL is to learn the manifold structure from the sampled data along with DR methods to obtain the latent space representation of the manifold. Many methods have been proposed in recent decades. For example, if the manifold is linear, some classic methods such as Principal Component Analysis (PCA) \cite{wold1987principal} and Multi-Dimensional Scaling (MDS) \cite{cox2008multidimensional} can be efficiently applied. When the manifold is non-linear, several embedding methods can be used such as Isomap \cite{balasubramanian2002isomap}, Local Linear Embedding (LLE) \cite{roweis2000nonlinear}, Laplacian Eigenmap (LE) \cite{belkin2003laplacian}, Local Tangent Space Alignment (LTSA) \cite{zhang2004principal} and  Hessian Eigenmaps \cite{donoho2003hessian} to name a few. 
Although ManL and DR methods are important pre-processing in machine learning and widely used, unfortunately, the understanding of what results they produce is largely missing. It is not clear whether these methods generate valid latent space representation when examined under the mathematical definition of a manifold. As subsequent learning process is built on ManL and DR, it is crucial to investigate the mathematical validity of the outputs generated from these methods so that the entire learning algorithm can be well understood and interpreted.

\paragraph{Manifold Structure}
Given a $d$-dimensional manifold $\mathcal{M}$ embedded in $\mathbb R^l$ ($l >d$) covered by a set of open sets $\mathcal{M} \subset \bigcup_{\alpha} U_{\alpha}$. For each set $ U_{\alpha}$, there is a homeoomorphism $\psi_{\alpha}: U_{\alpha} \rightarrow \mathbb R^d$. The pair $(U_{\alpha},\psi_{\alpha})$ forms a chart. The image of the chart map is deemed to present the manifold strucuture \cite{guillemin2010differential}. So by definition of homeomorphism, the chart map has to be bijective, i.e. one-to-one and onto, mapping a neighbourhood of a manifold to latent space. On top of that, one can require chart map to preserve other geometric aspects of manifold such as angles \cite{courant2005dirichlet} and distance \cite{maceachren1987sampling} if it is possible. To obtain manifold local structures, one common way used in most of ManL/DR algorithms is applying K nearest neighborhood (KNN) given the observed data is discretely sampled from the manifold. The estimated local structure is then used to infer the global coordinates of all data points by minimizing some loss function. A more advanced method is to approximate manifold by simplex decomposition \cite{boissonnat2018delaunay}, such as surface triangulation for 2-manifolds (surfaces) and tetrahedralization for 3-manifolds. These methods generate a piece-wise linear approximation of $\mathcal M$ from the sample so that, under appropriate sampling conditions, the manifold approximation quality can be guaranteed. Many achievements have been made along with this line \cite{cazals2006delaunay} \cite{boissonnat2018delaunay} \cite{maglo2012progressive}.
For example, Boissonnat et al. \cite{boissonnat2018delaunay} developed an algorithm named tangential complex (TC) to decompose manifold with $d$-simplices.  The reconstructed manifold generated from TC is isotopic to the original manifold if the density of data points is higher enough. 
In this paper, we use these methods to generate manifold structures.




\paragraph{Observation and Motivation}
Any chart map bijectively maps between the manifold to a latent space locally. Given DR/ManL methods claim to learn the latent representation, therefore, it is vital to ensure that those learnt representations are bijective or at least one-to-one, i.e. injective globally. Now the central problem is how to check a map between sets of discrete points is bijective/injective. For 1D-manifold, there is natural ordering if it is considered as a curve and hence the bijectivity is easily checked by looking at the order of the representations in $\mathbb R$. However, when $d>1$, there is no such ordering. Therefore, the crux is the implementation of such ordering for any $d$. Our solution to this is simplex preserving. Given a $d$-simplex decomposition on $\mathcal M$ in $\mathbb R^l$,  a bijective mapping $\mathbb R^l \rightarrow \mathbb R^d$ must preserve simplex structure, i.e. their integrity, connectivity and neighboring relations. For example, given a triangulation on a 2-manifold, a bijective map over the entire triangulation should be one-to-one over each triangle, edge, and point, no degeneracy and overlapping of triangles and edges. Existence of any degeneracy or overlap violates the required bijectivity/injectivity and hence indicates deviation from learning manifold properly.  

Following the above line of thoughts, we investigate the learning performance of commonly used DR/ManL algorithms and models. We observe that none of those methods could preserve the simplex structure even though we directly provide adjacency information from the simplex decomposition rather than generating it by their own. For example, the results of these methods on 2-manifolds usually contain a large number of edge crosses, indicating that the map is not one-to-one. Figure~\ref{failed_result} shows a simple example of embedding result with/without line segments intersections.




\begin{figure}[H]
     \centering
     \subfloat[Paraboloid  scatter]{\includegraphics[width = 0.2\textwidth, height = 0.18\textwidth]{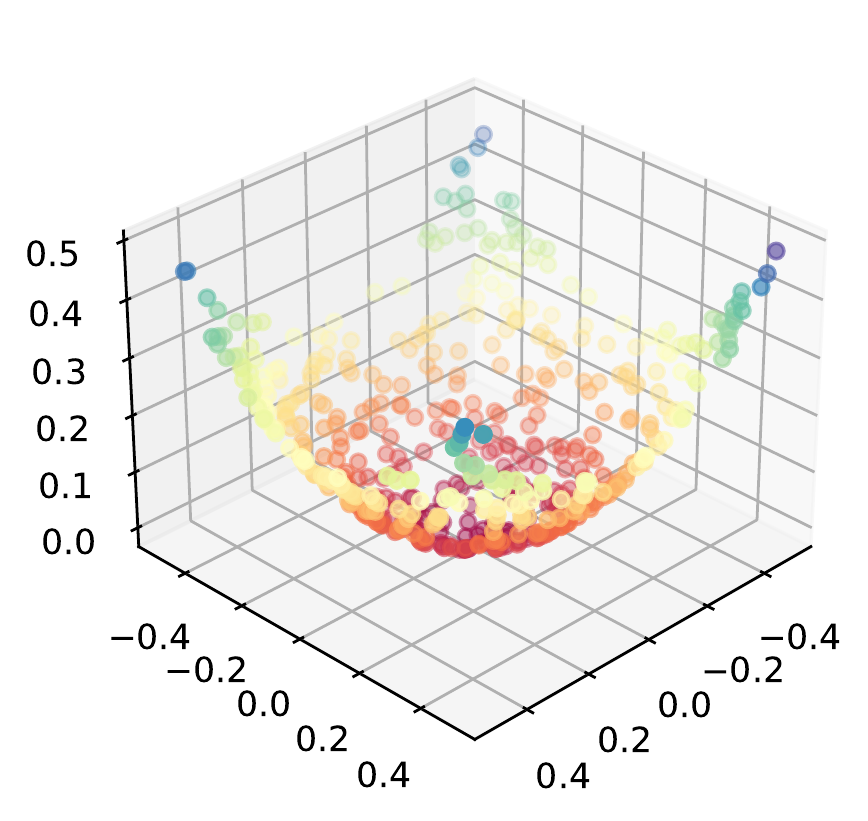}} \;\;\;
     \subfloat[Triangulation on paraboloid]{\includegraphics[width =0.2\textwidth, height=0.18\textwidth]{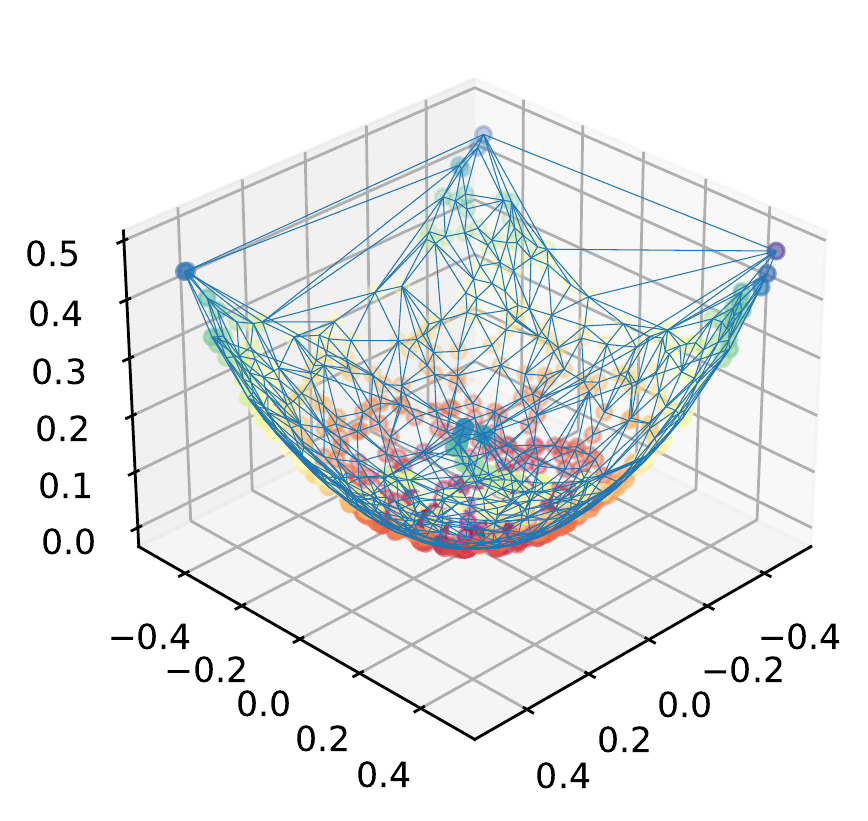}}\;\;\;
     \subfloat[Failed embedding]{\includegraphics[width =0.2\textwidth, height = 0.18\textwidth]{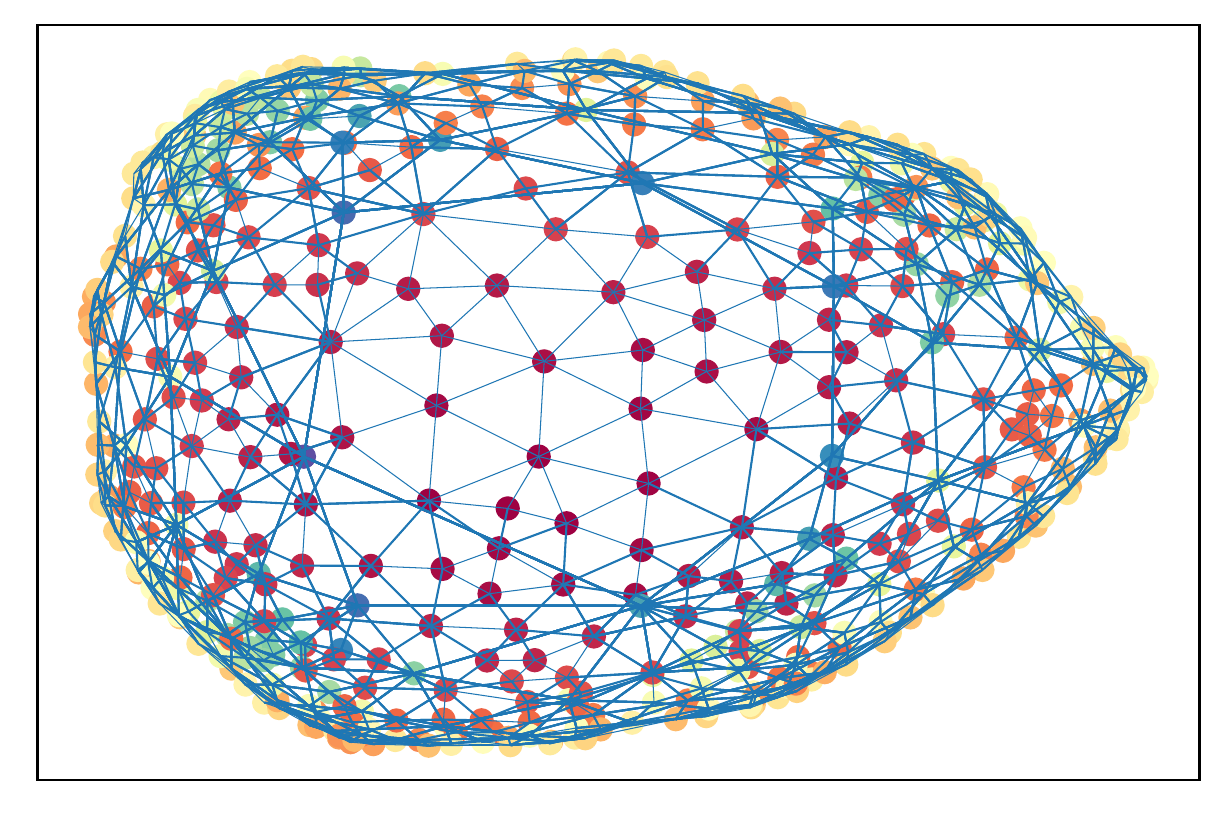}}\;\;\;
     \subfloat[Successful embedding]{\includegraphics[width = 0.2\textwidth, height = 0.18\textwidth]{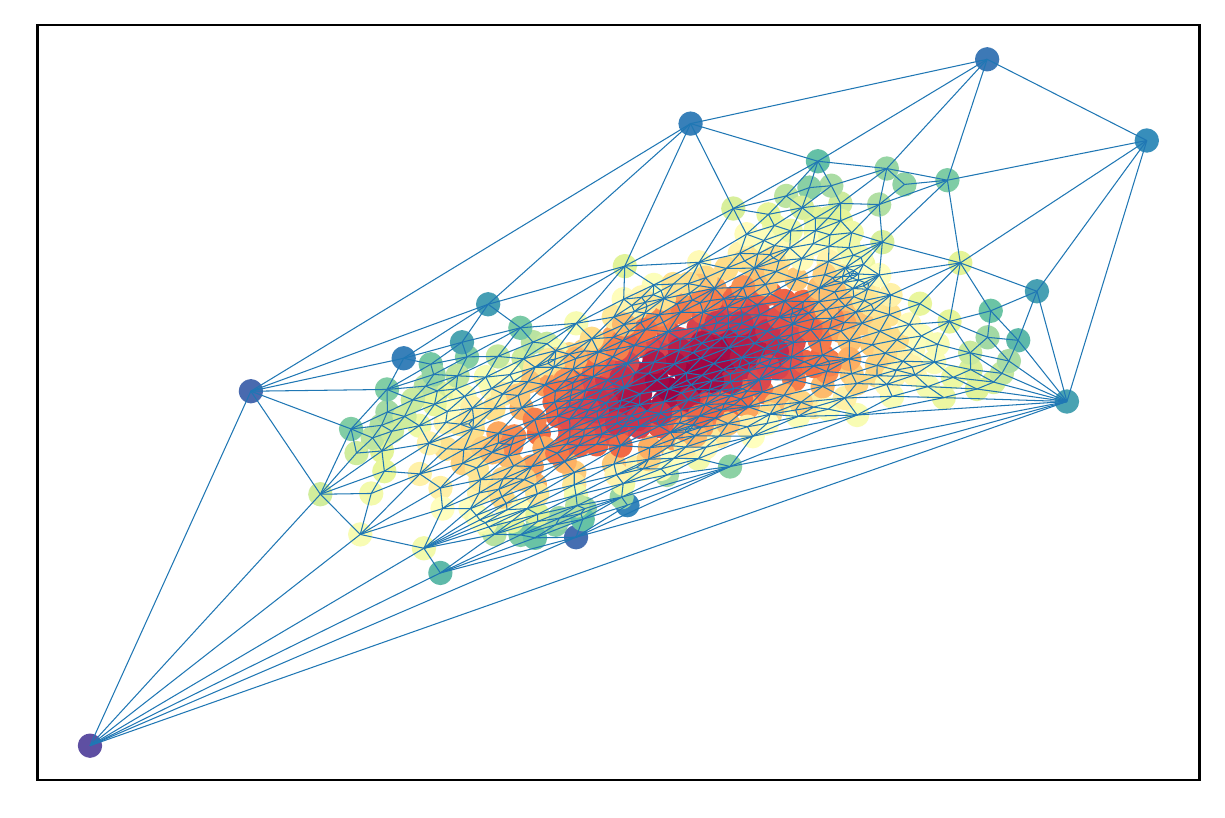}} \;\;\;
\caption{Two embedding algorithms on Paraboloid. One failed (c), Another one succeed(d).}   
\label{failed_result}
\end{figure}



We discover that these methods failed the bijectivity test that is essential in manifold definition. Nonetheless they may work for particular type of manifolds. For example, Isomap provides a correct embedding only if the manifold itself is isometric to an open set in the latent space, and LLE only reconstructs topological balls \cite{chen2011locally}. While the type of manifold is not known beforehand, it is important to have a geometrically correct algorithm which guarantees bijectivity. To address this problem, we propose a new method that is proved to be bijective (when restricted to codomain of the mapping) at least for 2-manifolds. The theory we use is Tutte's planar embedding theorem, which states that every planar graph has a convex representation in $\mathbb R^2$. The theorem was later generalized in  \cite{floater2003one} that proves under some additional requirements, a convex combination map that maps a triangulation from $\mathbb R^2$ to $\mathbb R^2$ is one-to-one. Our algorithm is applicable to any $d$-manifolds (without genus) with facet-to-facet tessellation. Most importantly, it has theoretical guarantee (refer to the Appendix) to produce a valid latent space representation for 2-manifolds deferring from the true latent space representation by a homeomorphism. 




\paragraph{Contributions.}
In summary, the main contributions of this paper is listed below.  
\begin{enumerate}
    \item We propose a method to validate manifold learning algorithms from the viewpoint of the definition of manifolds. Instead of examining the cluster quality, we focus on the bijectivity of the induced mapping function, which has to be homeomorphic to be the chart map of the manifold. 
    
    \item By using this method, we identify a fundamental problem in some prominent methods: the mapping function induced by these methods is not bijective, and in turn, violates the basic settings of manifolds.
    
    \item We offer a provably correct algorithm called fixed point Laplacian mapping (FPLM) to learn the manifold. This method has geometric guarantee to find a valid latent space representation (up to a homeomorphism).
    
    
    \item By generalizing the previous embedding theorem, we make our algorithm adaptive to any  non-degenerate edge-to-edge tessellation of 2-manifolds and most of 3-manifolds with and without boundary. We also discuss a sufficient condition to ensure a mapping that is always bijective to any manifold mesh generated by $d$-simplex decomposition.


\end{enumerate}

\paragraph{Organization.}
The remainder of the paper is organized as follows. In Section 2, we introduce the basic concepts used in this paper. FPLM will be introduced in Section 3 with the analysis and geometric guarantees discussed in Section 4. Section 5 provides the
experimental results to validate our claims. We discuss the strength and limitation of our methods in Section 6, and conclude this paper in Section 7 with some possible extensions.

\section{Definitions and Preliminaries} \label{Definition and Preliminaries}



In this paper, $\mathcal M$ is an orientable connected manifold with intrinsic dimension $d$  embedded in $\mathbb R^l$ and a finite sample $\mathbf X =\{\mathbf x_1, \mathbf x_2,..\mathbf x_N\} \subset \mathcal M$.  
The chart map $\psi$ of the manifold is a continuous bijective function that maps a neighbourhood of the manifold $\mathcal M$ to a subset in $\mathbb R^d$.  Its inverse, $\psi^{-1}$, is deemed to generate the structure of $\mathcal M$ \cite{guillemin2010differential}. 
We also assume that $\mathcal M$ is embedded in the ambient space with dimension at least $d+1$ {\em without self-intersection}. It is known that the Klein bottle in $\mathbb R^3$ is not an embedding \cite{whitney1944singularities}, merely an immersion, and hence will not be considered in this paper.

\subsection{d-Simplex decomposition of $\mathcal M$}\label{Triangulation and connectivity}
Let $\mathcal M$ be a $d$-manifold with boundary embedded in $\mathbb R^l$ with discrete sample $\mathbf X$. By a $d$-simplex, we mean a $d$-dimensional polytope that is the convex body formed by its $d+1$ vertices. For example, a 0,1,2-simplex stands for a point, line segment, and triangle. Given a $d$-simplex, we call this simplex \textit{degenerate} if it is less than $d$-dimension. We will further assume that all data points sampled in $\mathcal M$ are in \textit{general position}, i.e. no colinearity among points, or in other words, no extra point inside a simplex. For example, there will be no point inside a triangle or edge in triangulation.
\begin{defn}[d-simplex decomposition of $\mathcal M$]
Let $\mathcal S$ be a finite set of non-degenerate $d$-simplex and let $D_{\mathcal{S}} = \bigcup_{S\in \mathcal S} S$ we will call $\mathcal S$ a $d$-simplex decomposition of $\mathcal M$ if: 

\begin{enumerate}
    \item The intersection of any pair of $d$-simplex can either be empty or a common $\{d-1,d-2,..0\}$ simplex, and $\mathbf X$ are the vertices. 
    
    \item The boundary of $D_{\mathcal S}$, a closed polytope written as $\partial D_{\mathcal S}$, is formed by those ($d-1$)-simplices in $\mathcal S$ that are not shared. 
    \item $D_{\mathcal S}$ is homeomorphic to $\mathcal M$. 
\end{enumerate}
\end{defn}

This $d$-simplex decomposition of $\mathcal M$ is actually the best piece-wise linear manifold approximating $\mathcal M$ one could possibly have given a discrete sample from $\mathcal M$. Note that this differs from normal simplical complex on $\mathbf X$ in $\mathbb R^l$, which would be an $l$-simplex decomposition whose convex hull circumscribes $\mathbf X$. For example, for 2-manifold, 2-simplex decomposition (triangulation) should be applied and tetrahedronization for $d=3$. We now take triangulation as an example to provide a few more definitions. 
Let $D_{\mathcal{T}} = \bigcup_{T\in \mathcal T} T$ be a triangulation $\mathcal T$ in $\mathbb{R}^2$ ($T$ stands for a triangle). Following \cite{floater2003one}, based on the second requirement of simplex decomposition, $D_\mathcal{T}$ will be simply connected with boundary $\partial D_{\mathcal{T}}$.  For the vertices and edges contained in the boundary $\partial D_{\mathcal{T}}$, we call them \textit{boundary} vertices and edges, and otherwise the \textit{interior} vertices and edges.
If an edge with both ends on boundary vertices, we will call it a \textit{dividing edge}. For example, in Figure~\ref{dividing_edge}, the edge $[V,W]$ is a dividing edge.





%

\begin{figure}[H]
\centering
\includegraphics[width =0.3\textwidth, height = 0.18\textwidth]{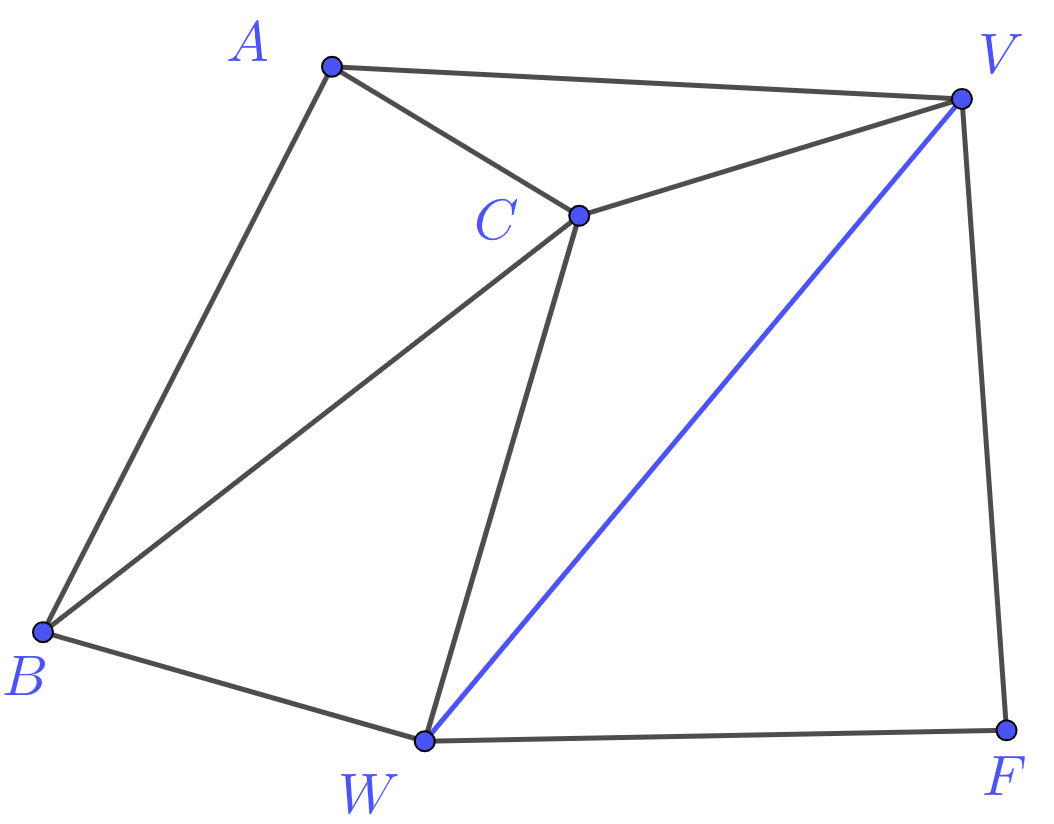}
\caption{Connectivity between triangles and dividing edges}\label{dividing_edge}
\end{figure}

\begin{defn}
We say that a triangulation is strongly connected if it contains no dividing edges.
\end{defn}

As we mentioned in the previous sections, if a mapping bijectively maps the manifold simplex decomposition result to the latent space, then the connectivity between simplices must be preserved. It is easy to see that if there is overlap from simplices, the points inside the overlap will certainly have more than 1 preimage, and hence not injective/bijective. In terms of simplex decomposition, for unknown manifold one can apply Tangential Complex algorithm \cite{boissonnat2014manifold} with required conditions and consistency test. For test purpose, or validity checking for a given DR/ManL algorithm, one can generate those manifolds such that the simplex mesh in latent space is preserved, for example, graph of some function, i.e. $(x_1,\ldots,x_d,f(x_1,\ldots,x_d))$, where $x_i\in \mathbb R$ are latent variables and $f$ is the function, which is a $d$-manifold, a hypersurface in $\mathbb R^{d+1}$.

\section{Fixed Point Laplacian Mapping}\label{Fixed Point Laplacian Eigenmap: FPLM}

By using the validity checking method mentioned above, e.g. the example in Section 1 and many more in experiment section, we realised that those mostly used DR/ManL methods we tested are not bijective, and hence do not really learn a manifold. The question is then, is it possible to design such an algorithm which has bijectivity guarantee, at least for some manifolds? The answer is positive.  

\subsection{Settings}\label{settings}
Based on the previous observation, one necessary condition for bijectivity is that the simplex structure in $\mathcal M$, a graph written as  $\mathcal{G}_{\mathcal{S}}$, is preserved by the mapping. Unfortunately, this is highly nontrivial. Normal neighborhood preserving and alignment ideas often seen in many DR/ManL methods do not work as they lack "hard" enforcement to ensure the preserving results, which is also the reason they fail bijectivity test. We need geometry inspired constraints and/or procedures with bijectivity embedded naturally. 

We start from constructing a (weighted) adjacency matrix $\mathbf A \in \mathbb{R}^{N \times N}$ derived from a simplex decomposition of $\mathcal M$, with the associated degree matrix and Laplacian denoted by $\mathbf D$ and $\mathbf L$. 
The algorithm that we will show below is a two-round procedure with the same optimization performed twice with different constraints each time. We call this optimization fixed-point Laplacian mapping (FPLM), where the fixed points are the constraints. We denote these fixed points as $\mathbf{C}=[\mathbf {c}_1,...,\mathbf{c}_p]^T \in \mathbb{R}^{p\times d}$. We write $\mathbf{P}(\mathbf{C})$ as the simple polytope formed by joining fixed points in $\mathbf C$ as vertices.

\subsection{Fixed-point Laplacian Mapping (FPLM)}
\label{FPLM_algorithm}
FPLM is formulated as follows: 
\begin{equation}\label{foc_FPLM1}
    \min_{\mathbf Y \in \mathbb R^{N \times d}} \, \text{tr}(\mathbf{Y}^T\mathbf{LY}), \qquad \text{ subject to } \, \mathbf y_{i} = \mathbf c_i, i \in [1,p]
\end{equation}
where $\mathbf  c_i\in\mathbb R^d, i\in [1,p]$ are the fixed points. We firstly determine whether $\mathcal{G}_{\mathcal{S}}$ is strongly connected or not (i.e., whether there is a dividing edge). If $\mathcal{G}_{\mathcal{S}}$ is strongly connected, the fixed points in the first round, collected in $\mathbf C_1$, are the images of the vertices from a randomly selected $d$-simplex after reducing its dimensionality. Therefore, $\mathbf C=\mathbf C_1$ in FPLM and $p=d+1$. Note that this step is lossless as $d$-simplex in $\mathbb R^l$ is intrinsical $d$-dimensional and linear. After the first round of FPLM, the boundary of the simplex decomposition, i.e., $\partial D_{\mathcal S}$, will be mapped inside $\mathbf P(\mathbf C_1)$ in $\mathbb R^d$. Recall that the boundary of a $d$-simplex decomposition is the $d-1$ simplicial complex that is not shared. It is straightforward to use the boundary of the simplex decomposition as the boundary to conduct the second round of FPLM. We collect these $p$ vertices in the boundary polytope in $\mathbf C_2$. When $\mathcal{G}_{\mathcal{S}}$ is not strongly connected, we let $n$ be the number of boundary vertices detected from simplex decomposition in $\mathbb R^l$, and construct a $p$-face ($p=n$) convex polytope in $\mathbb R^d$. One example of such convex polytope is the regular $p$-face polytope.

We now summarize two rounds of FPLM below.
\begin{algorithm}[H]
\caption{Two Rounds of FPLM}\label{2sFPLM_algorithm}
\begin{algorithmic}[1]
\STATE \textbf{Input:} Simplex decomposition graph $\mathcal{G}_{\mathcal{S}}$, first round fixed points $\mathbf C_1$. 
\STATE Construct weighted adjacency matrix $\mathbf A$ and its Laplacian $\mathbf L$ from $\mathcal{G}_{\mathcal{S}}$.
\IF{No dividing edge in $\mathcal{G}_{\mathcal{S}}$}
\STATE Obtain first-step $\mathbf Y_1$ by \eqref{foc_FPLM1} using $\mathbf C=\mathbf C_1$. 
\IF{No boundary detected inside $\mathbf P(\mathbf C_1)$}
\RETURN{$\mathbf Y_1$}
\ELSE
\STATE Use the boundary detected as $\mathbf C_2$.
\STATE Obtain second-step $\mathbf Y_2$ by \eqref{foc_FPLM1} using $\mathbf C=\mathbf C_2$.
\RETURN{$\mathbf Y_2$}
\ENDIF
\ELSE
\STATE Find the number of boundary points of $\mathcal{G}_{\mathcal{S}}$ as $p$ and construct a $p$-face convex polytope as $\mathbf C_1$.
\STATE Obtain first step $\mathbf Y_1$ by \eqref{foc_FPLM1} using $\mathbf C=\mathbf C_1$.
\RETURN{$\mathbf Y_1$}
\ENDIF
\end{algorithmic}
\end{algorithm}

\section{Analysis of Algorithm and Geometric Guarantees} \label{Analysis of Algorithm and Geometric Guarantees}
We now justify that the about procedure results in a bijective mapping. A summary of the line of proofs is the following. We first show that the mapping induced from FPLM is a convex combination mapping over simplex decomposition. Then taking 2-manifolds as an example, we prove that the convex combination mapping is one-to-one over entire triangulation. Restricting the mapping to the codomain, the mapping is bijective. We further prove that the procedure is applicable to any 2-manifold which structure is estimated from a non-degenerate edge-to-edge tessellation of polygons. Due to page limit, we only present the central theorems here; see Appendix for detailed proofs and derivations.

\paragraph{Algebraic solution of FPLM and Convex Combination Mapping}
The global minimizer of FPLM $\tilde{\mathbf Y}^*$ is: 
\begin{equation}
    {\mathbf y}_i^* = \sum_{j =1}^{n} \frac{\mathbf A_{ij}}{\mathbf D_{ii}} {\mathbf y}_j^*  = \sum_{j =1}^{n} \lambda_{ij} {\mathbf y}_j^*, \quad \forall{i = 1, ..., n-p}, \label{barycenter_mapping1}
\end{equation} 
Where $\mathbf{D}_{ii}$ is the diagonal of the degree matrix. By the definition of degree matrix we have $\sum_{j=1}^n \lambda_{ij} = 1$, $\forall i$. This shows that every optimal non-fixed point is a convex combination of its neighbours. 
As we mentioned earlier, the connectivity between simplices should remain the same in both image and pre-image of a bijective function over the entire simplex decomposition. Given two simplex decomposition $\mathcal S$ and $\mathcal S'$ of some subsets in $\mathbb R^d$, with some abuse of notation, we call a function $f: \mathcal S \rightarrow \mathcal S'$ a \textit{piece-wise linear function} if it is continuous over entire $D_{\mathcal S'}$ and linear over each simplex. Similarly, we have \textit{piece-wise linear mapping } $\phi : \mathcal G_{\mathcal S} \rightarrow \mathcal S$ to be a mapping taking from $\mathcal M$ simplex decomposition to its latent space where the simplex structure remains. A typical DR/ManL method learns $\phi$ such that $\mathbf y_i=\phi(\mathbf x_i)$. If $\phi$ satisfies \eqref{barycenter_mapping1}, we call $\phi$ a convex combination mapping \cite{floater2003one}. Clearly, FPLM generates a convex combination mapping over $d$-simplex decomposition of the manifold. 

\paragraph{Geometric Guarantees of FPLM}
We present our central theorems for the geometric guarantees of FPLM here for 2-manifolds. 
\begin{thm}\label{thm:2dmfdtriangulationplanar}
For any 2-manifold without genus, the graph induced from any valid triangulation on the manifold is planar. 
\end{thm}

The idea is to prove that the graph induced from a triangulation does not contain Kuratowski subgraph $K_5$ and $K_{3,3}$. We now show the features of $\phi$, which is the convex combination mapping induced from FPLM.

\begin{prop}
FPLM maps all non-fixed points inside the convex hull formed by the fixed points ($\mathbf{P}(\mathbf{C})$).
\end{prop}
We only present a sketch of the proof here. If on contrary, there is a point outside the convex hull of $\mathbf{P}(\mathbf{C})$, then there must be more points outside too due to \eqref{barycenter_mapping1}. For those outside points, find the one on the edge of the convex hull, then it must have more points surrounding it too. Continue this process until all non-fixed points are exhausted. The out-most one will not have a convex hull formed by its neighbors according to supporting hyperplane theorem \cite{boyd2004convex} against the fact that every non-fixed point has to be convex combination of its neighbors. Therefore the assumption is incorrect. Another way to prove this is by direct observation of the minimization from FPLM. 



By applying the conclusion from previous works \cite{kneser1926losung} \cite{floater2003one}, we proved that the first round of FPLM is one-to-one over any strongly connected triangulation (Appendix One). We then explored the convexity of the boundary polygon of $\mathcal{G}_{\mathcal{S}}$ after the first round of FPLM and concluded the following lemma: 



\begin{lem}
Given a strongly connected triangulation $\mathcal T$, $\partial D_{\mathcal S}$ is mapped as a convex hull after the first round of FPLM, and hence the results from algorithm \ref{2sFPLM_algorithm} is one-to-one. 
\end{lem} 

The conclusion from the above Lemma is proved by virtual of Tutte's embedding theorm \cite{tutte1963draw} after we show the convexity of the image of $\partial D_{\mathcal S}$. 
However, when the triangulation is not strongly connected, the first round of FPLM is no longer injective because the dividing edge will be mapped as the boundary edge of the manifold inside the selected triangle in the first round of FPLM. Therefore, we have to directly detect the boundary from $\mathcal{G}_{\mathcal{S}}$ and generate a $p$-side convex polygon in $\mathbb R^2$ so that all dividing edges remain inside the boundary and none of the boundary vertices are then collinear. The following theorem justifies this part in algorithm \ref{2sFPLM_algorithm}.  


\begin{thm}\label{thm:withdividingedge}
Given a triangulation $\mathcal T$ with dividing edges, FPLM with fixed points $\mathbf C$ as vertices of a p-side convex polygon is one-to-one, where $p$ is the number of boundary vertices.
\end{thm}


The above conclusions make FPLM applicable to any 2-manifold (orientable and connected) without genus. However, when $d \geq 3$ we need an extra condition,  orientation preserving property to ensure injectivity. We discuss this in the appendix. 

\section{Experiments} \label{Experiments}
In this section, we investigate the learning performance of widely used state-of-the-art DR/ManL algorithms. The structure of every 2-manifold was generated by applying either Tangential Complex (TC) algorithm \cite{boissonnat2014manifold} or Delaunary/Surface triangulation. The structure of 3-manifold is generated by using Delaunay tetrahedralization algorithm \cite{si2015tetgen} included in Tetgen and TC.
To have a fair comparison between FPLM and other prominent methods, the adjacency information obtained from the simplex decomposition will be used as the input as manifold structure. The number of the line segments crosses will be counted as a measure to evaluate the learning performance for all included models.

\textbf{Experiment setup.} 
The 2-manifolds included in the experiment are: Monkey saddle, Swiss roll, Paraboloid, Twin peaks and Sphere. 
We construct a weighted adjacency matrix from triangulation via rbf kernel function. That is, $ A_{ij} = \exp(-\gamma d_m(\mathbf x_i, \mathbf x_j))$ if $\mathbf x_i$ is connected to $\mathbf x_j$, where we use $l_2$ distance $d_m(\mathbf x, \mathbf y) = \sqrt{\sum_{i=1}^d (x_i - y_i)^2}$ for $\mathbf x, \mathbf y \in \mathbb{R}^d$. For all experiments, we fix $\gamma = 0.1$. 

The settings of other learning algorithms are as follows: for LE and LTSA, 
we use pre-computed $\mathbf A$ as input; for Local Linear Embedding, we use adjacency matrix constructed from the simplex  decomposed graph as input to replace the neighborhood graph; for Isomap, we construct the distance matrix from the simplex decomposed  graph and  distance; for MDS and tSNE, we use default settings. Finally, for Manifold Autoencoder, we construct a neural network with layer $3 \times 64 \times 2 \times 64 \times 3$. Activation function is Relu; dropout layer is considered with $p = 0.2$. Batch normalization is applied to the bottleneck layer. The optimizer is chosen to be ADAM with a learning rate of $0.1$. For every experiment, we run 1000 epochs. All experiments are carried out on a laptop computer running on a 64-bit operating system with Intel Core i5-8350U 1.90GHz CPU and 16G RAM with Python 3.36.

For the manifold with boundary, the second round output of FPLM will be compared with other learning algorithms. Due to the space limit, we will only present the investigation results of swiss roll for the manifold with boundary and 2-Sphere for manifold without boundary. For the rest of the result, please see Appendix. Figures below show the comparison results:

\begin{figure}[H]
     \centering
     \subfloat[]{\includegraphics[width = 0.13\textwidth, height = 0.13\textwidth]{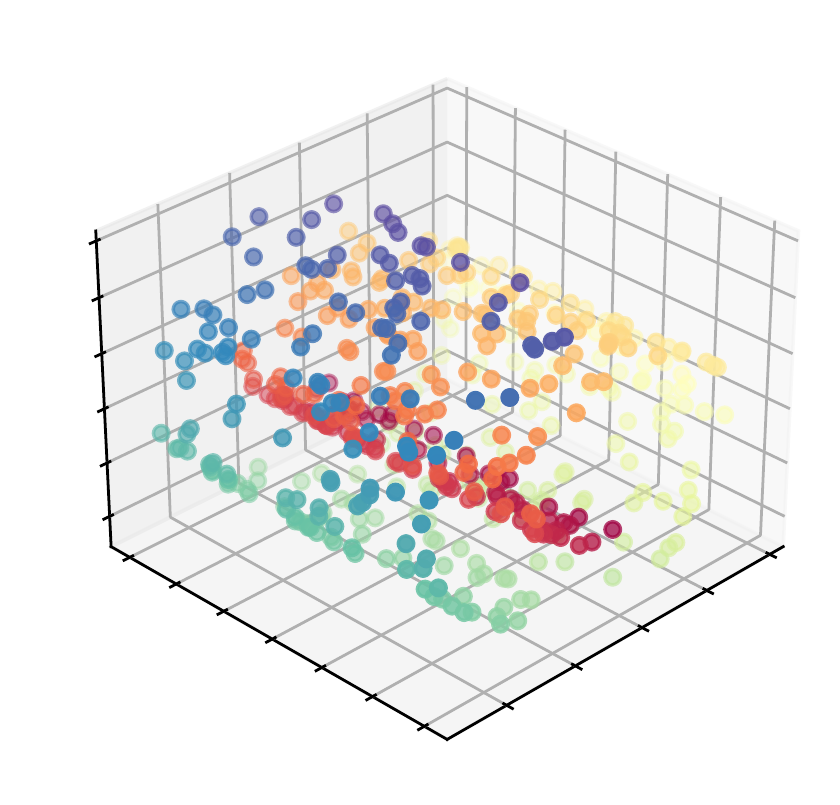}} \;\;\;
     \subfloat[]{\includegraphics[width =0.13\textwidth, height = 0.13\textwidth]{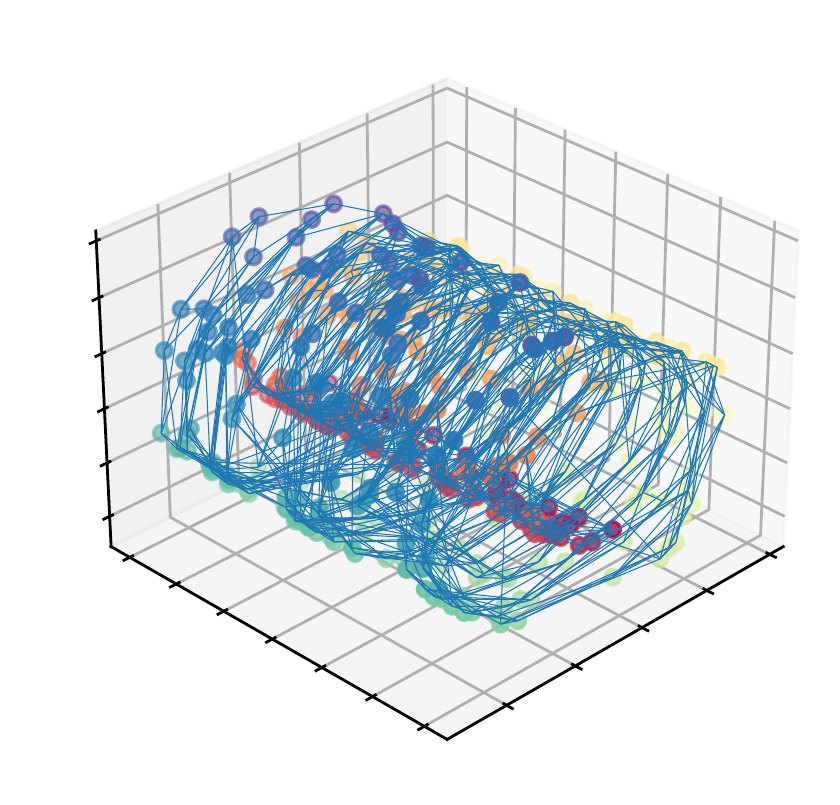}}\;\;\;
     \subfloat[]{\includegraphics[width =0.13\textwidth, height = 0.13\textwidth]{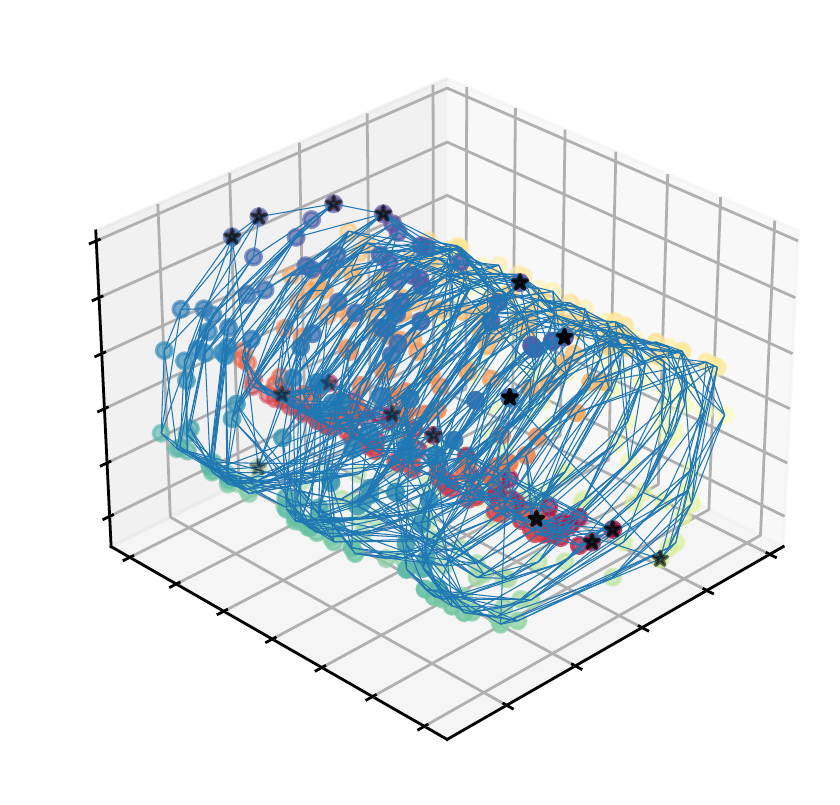}}\;\;\;
     \subfloat[]{\includegraphics[width =0.13\textwidth, height = 0.13\textwidth]{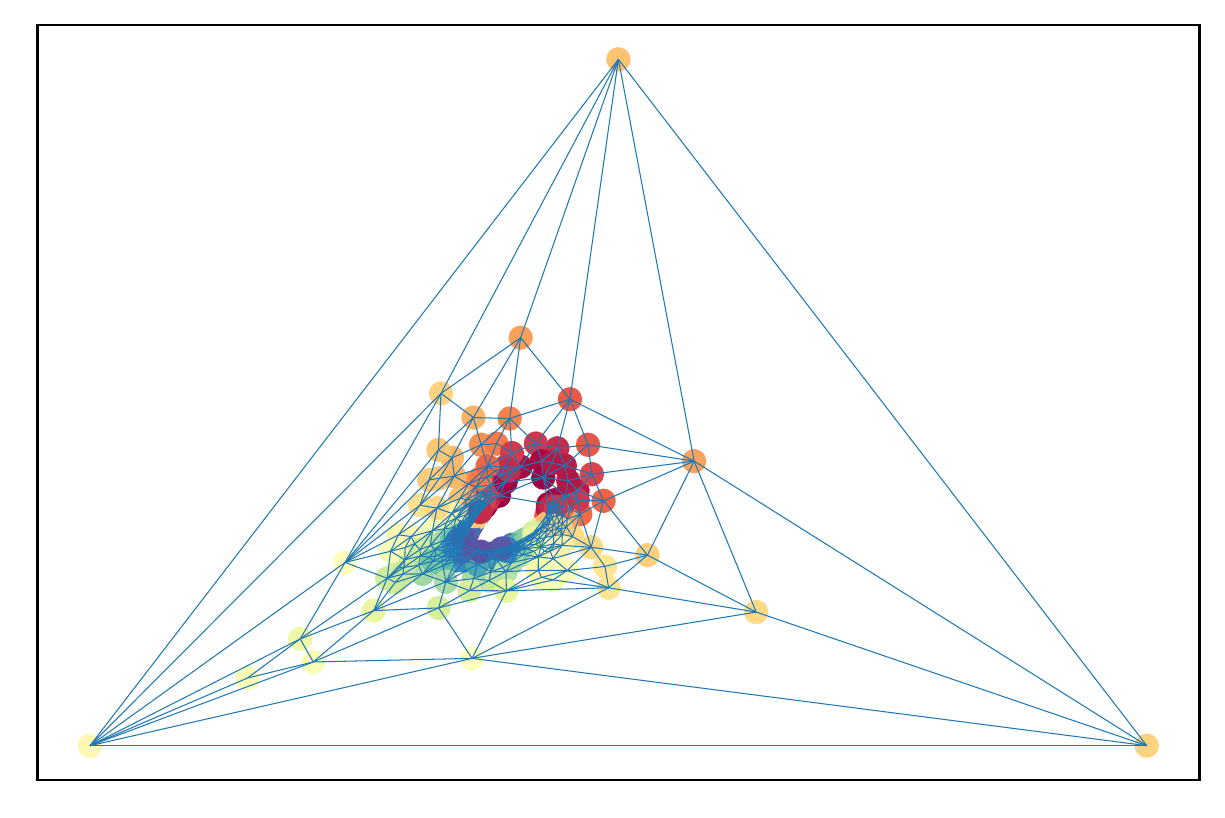}}\;\;\;
     \subfloat[]{\includegraphics[width =0.13\textwidth, height = 0.13\textwidth]{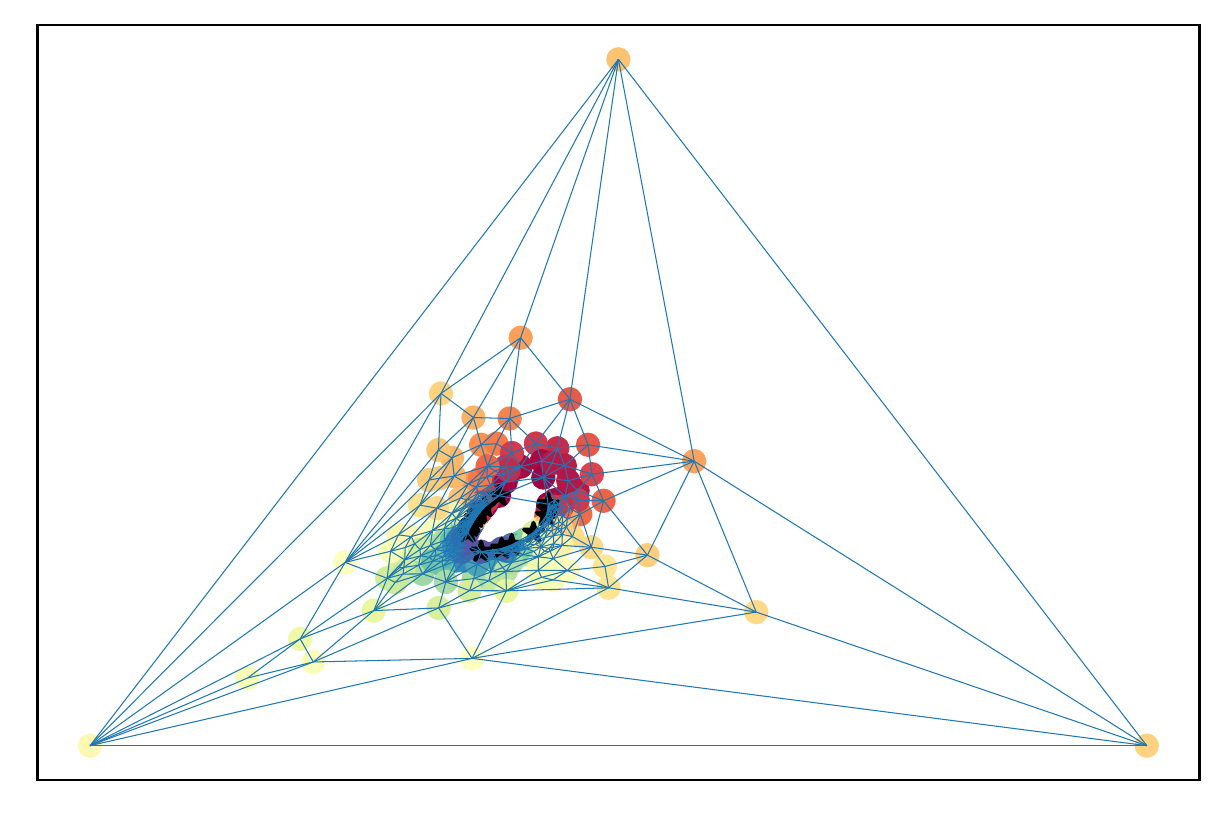}}\;\;\;
     \subfloat[]{\includegraphics[width =0.13\textwidth, height = 0.13\textwidth]{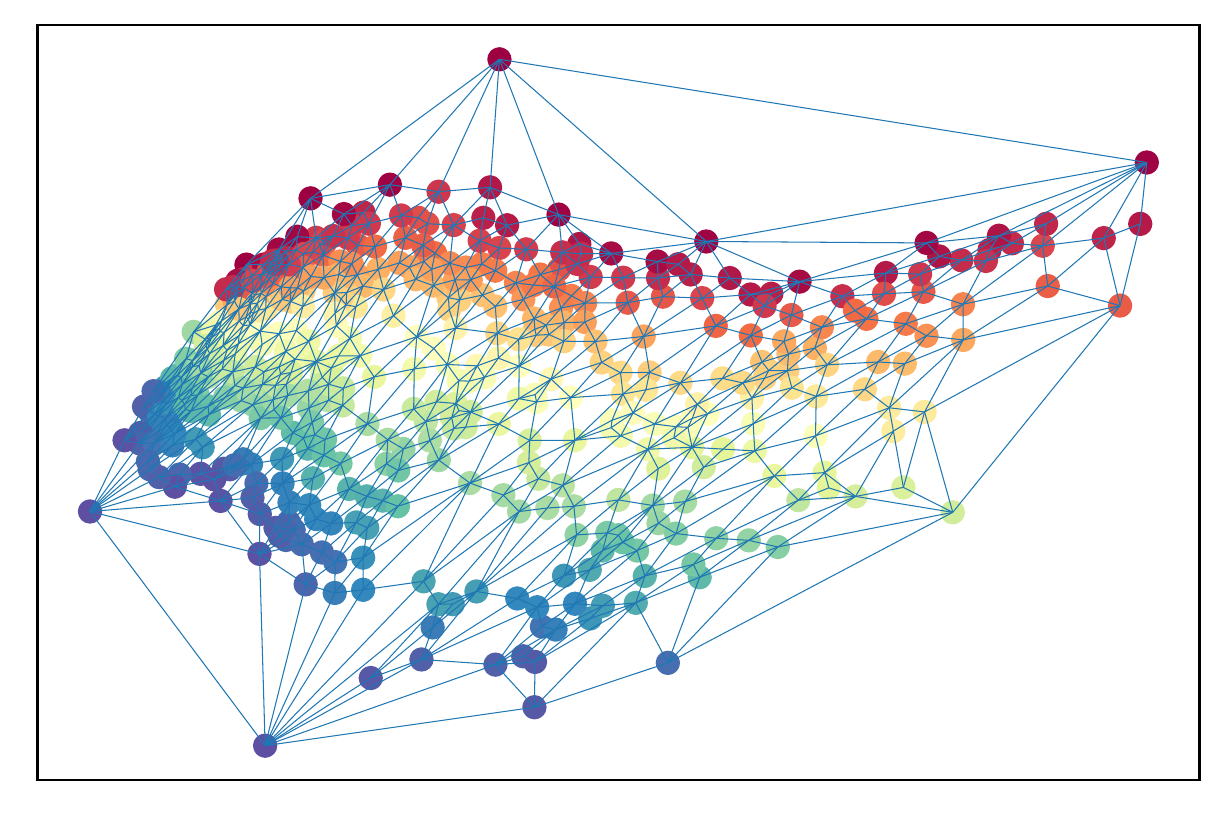}}\;\;\;
\caption{FPLM on Swiss roll:(a) Manifold scatters,(b) Triangulation on manifold, (c) Boundary detection (d) First round FPLM, (e) Boundary detection for the first round FPLM, (f) Final result.}  
\label{FPLM_swiss_roll}
\end{figure}

\begin{figure}[H]
     \centering
     \subfloat[AE]{\includegraphics[width = 0.11\textwidth, height = 0.13\textwidth]{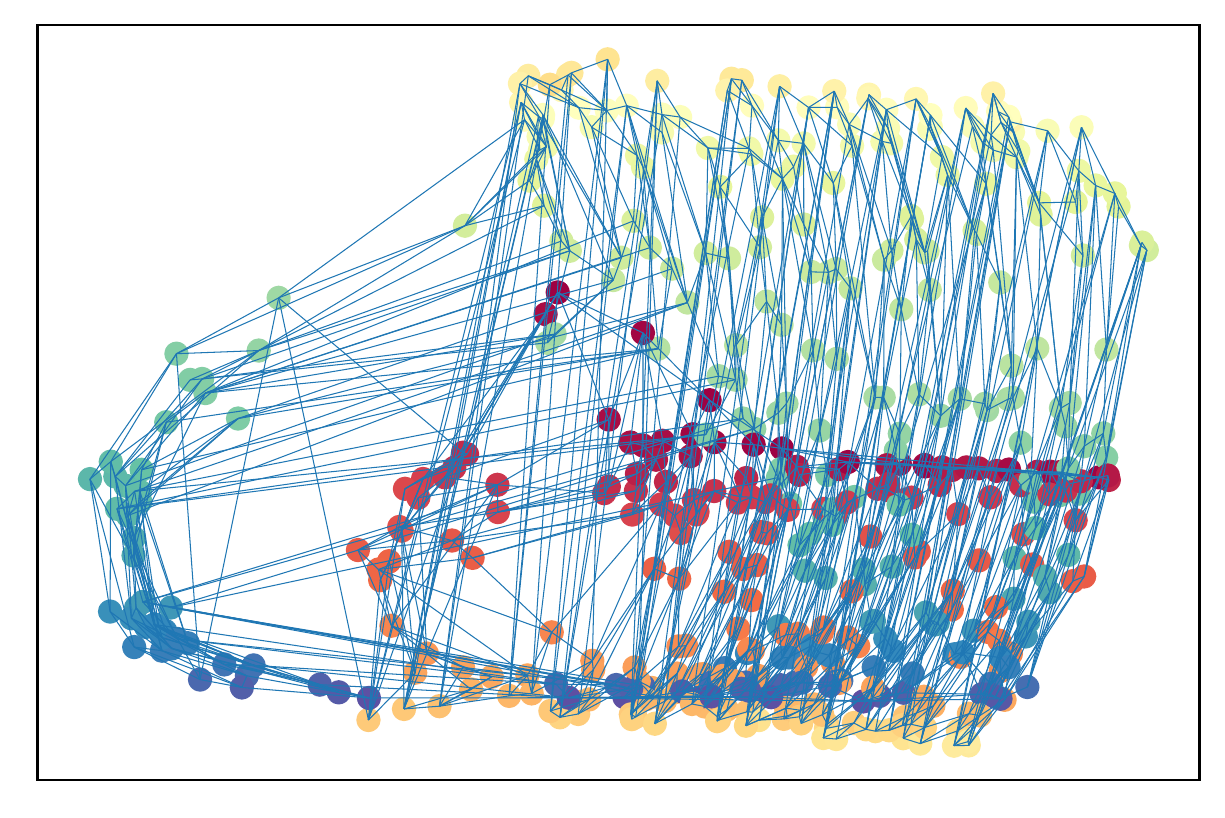}} \;\;\;
     \subfloat[Isomap]{\includegraphics[width = 0.11\textwidth, height = 0.13\textwidth]{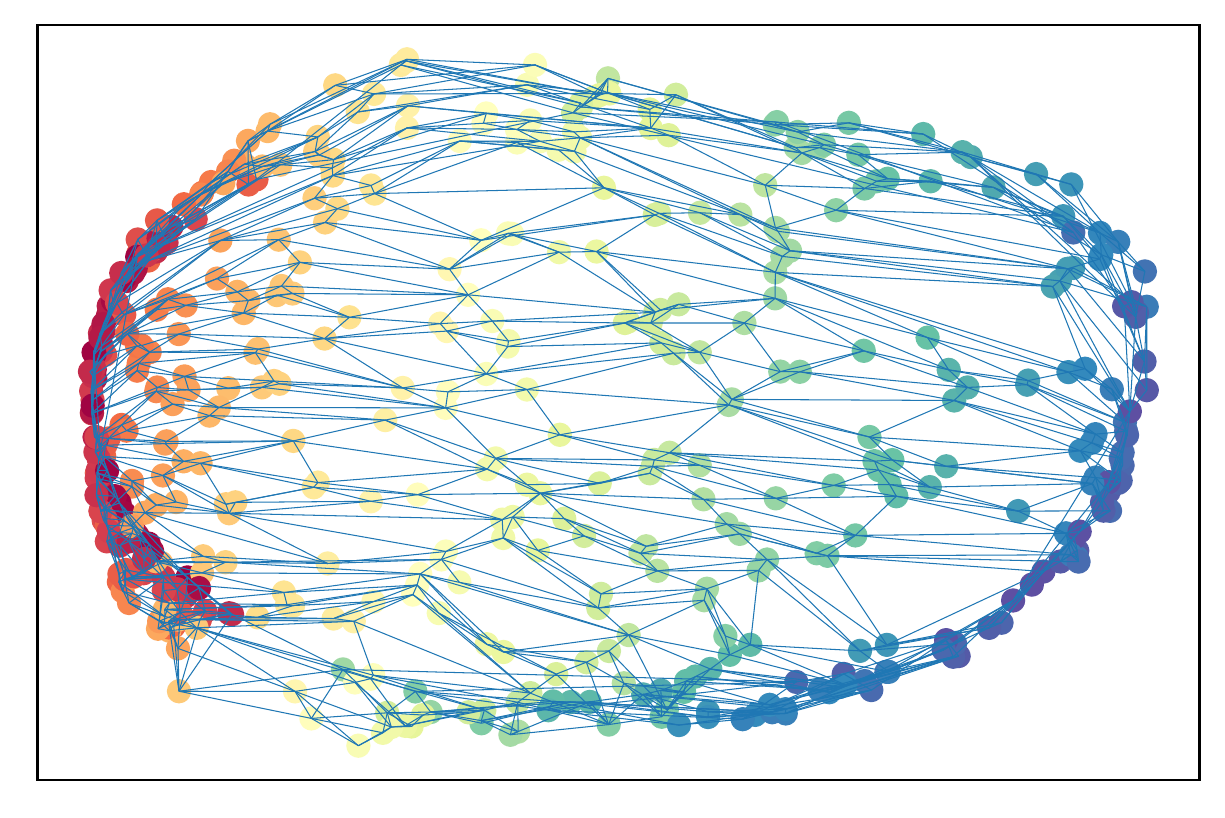}}\;\;\;
     \subfloat[LE]{\includegraphics[width = 0.11\textwidth, height = 0.13\textwidth]{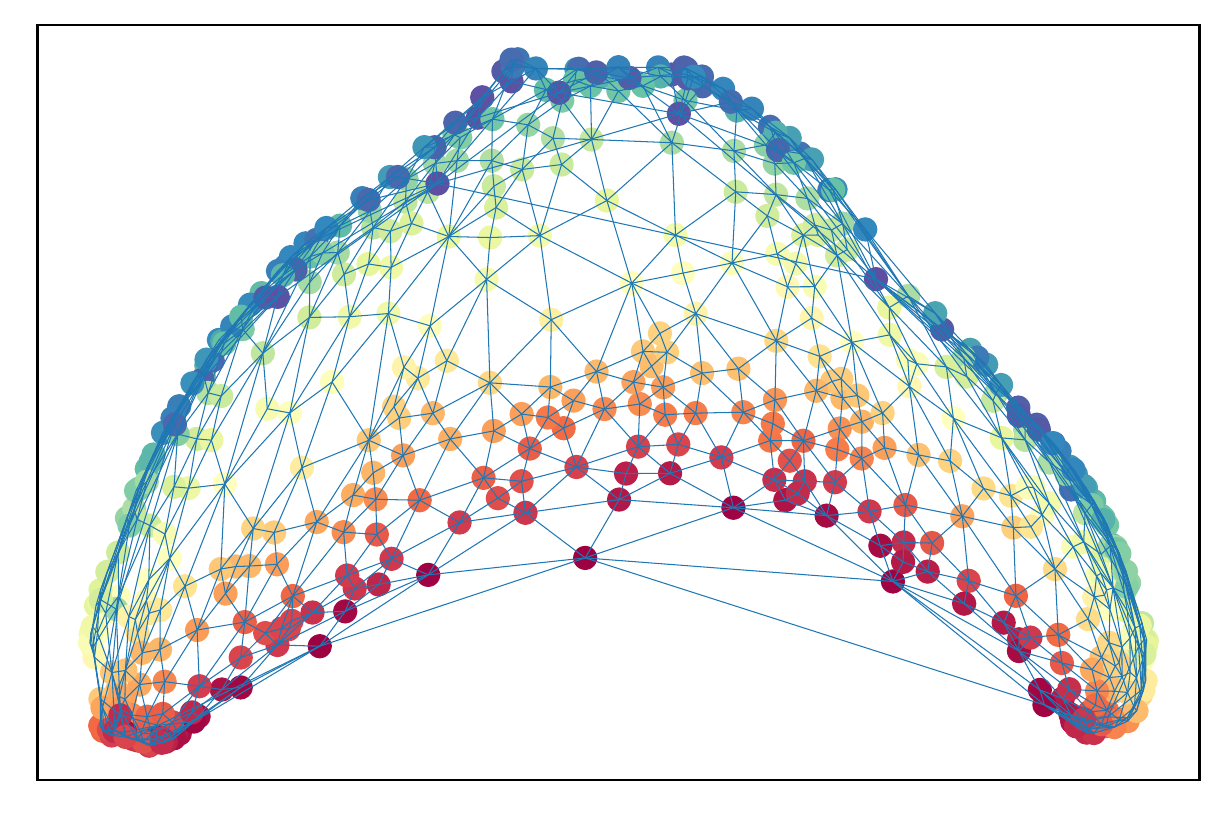}}\;\;\;
     \subfloat[LLE]{\includegraphics[width = 0.11\textwidth, height = 0.13\textwidth]{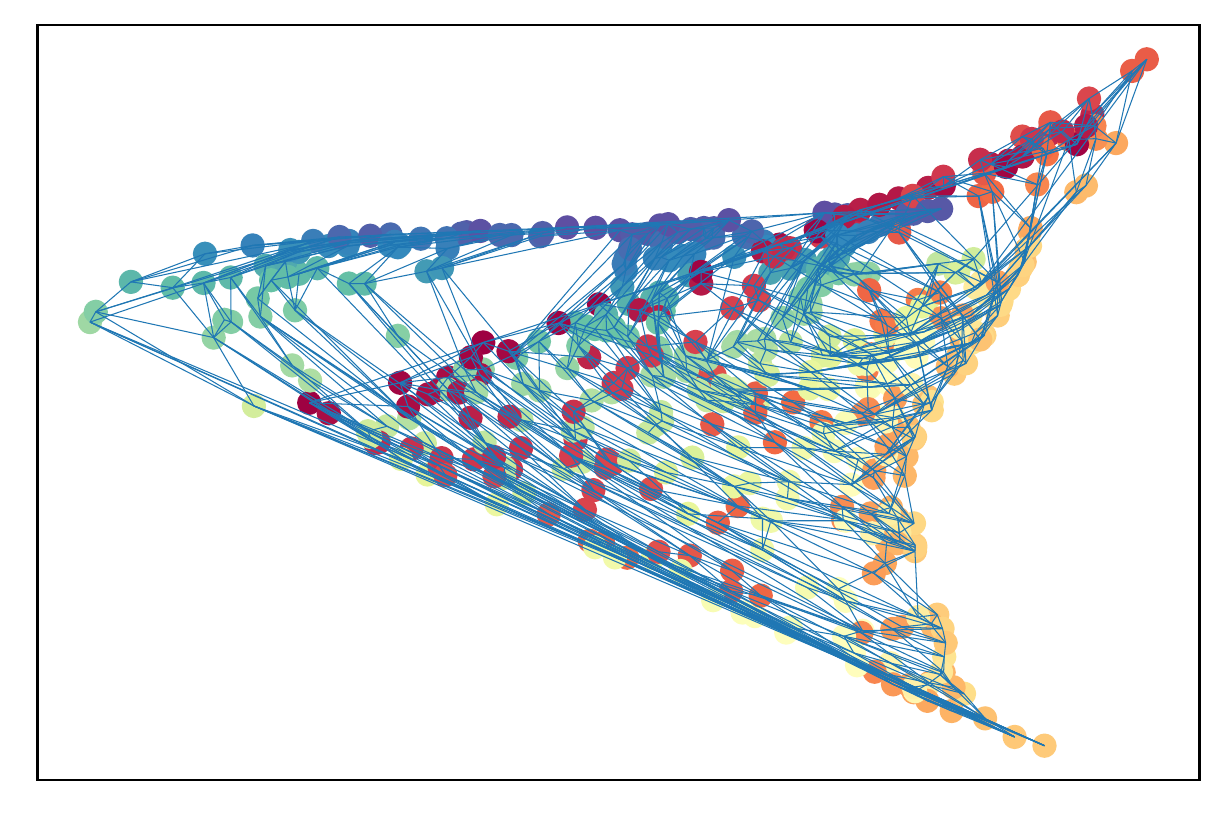}}\;\;\;
     \subfloat[LTSA]{\includegraphics[width = 0.11\textwidth, height = 0.13\textwidth]{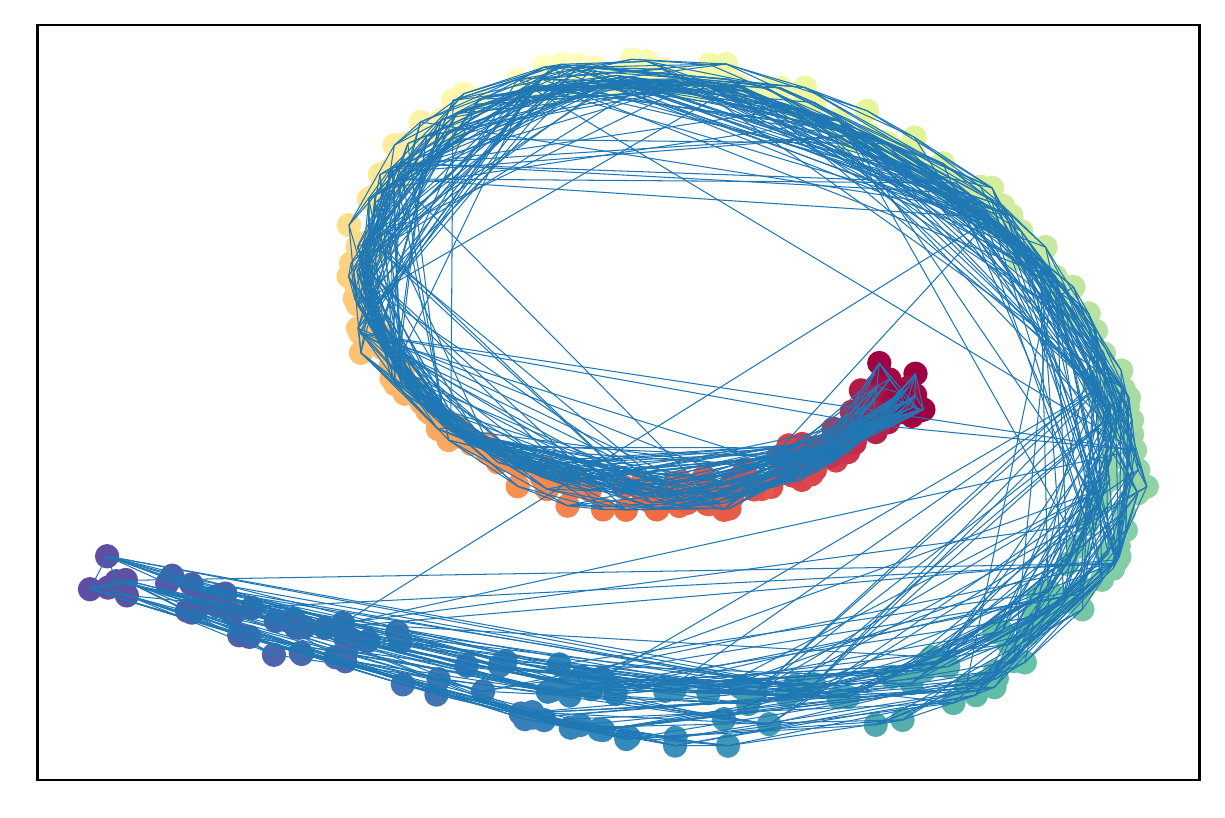}}\;\;\;
     \subfloat[MDS]{\includegraphics[width = 0.11\textwidth, height = 0.13\textwidth]{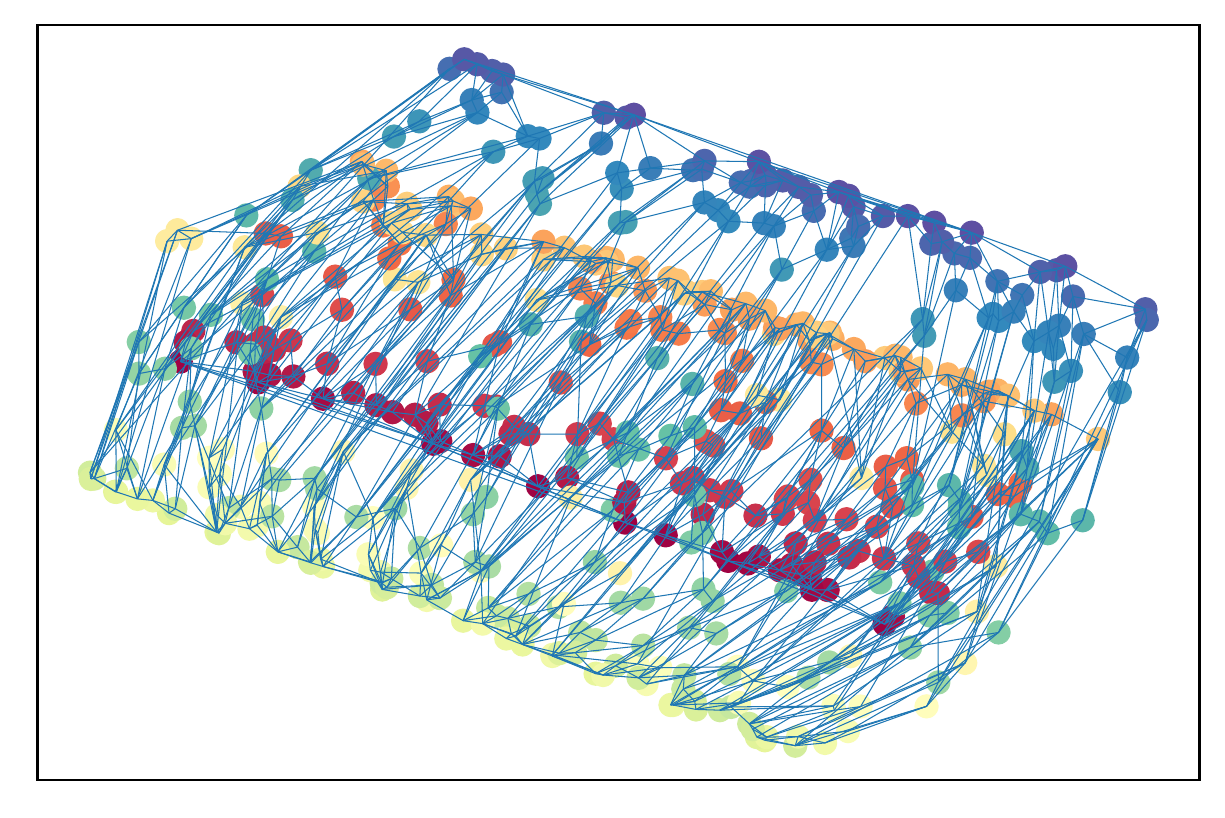}}\;\;\;
     \subfloat[t-SNE]{\includegraphics[width = 0.11\textwidth, height = 0.13\textwidth]{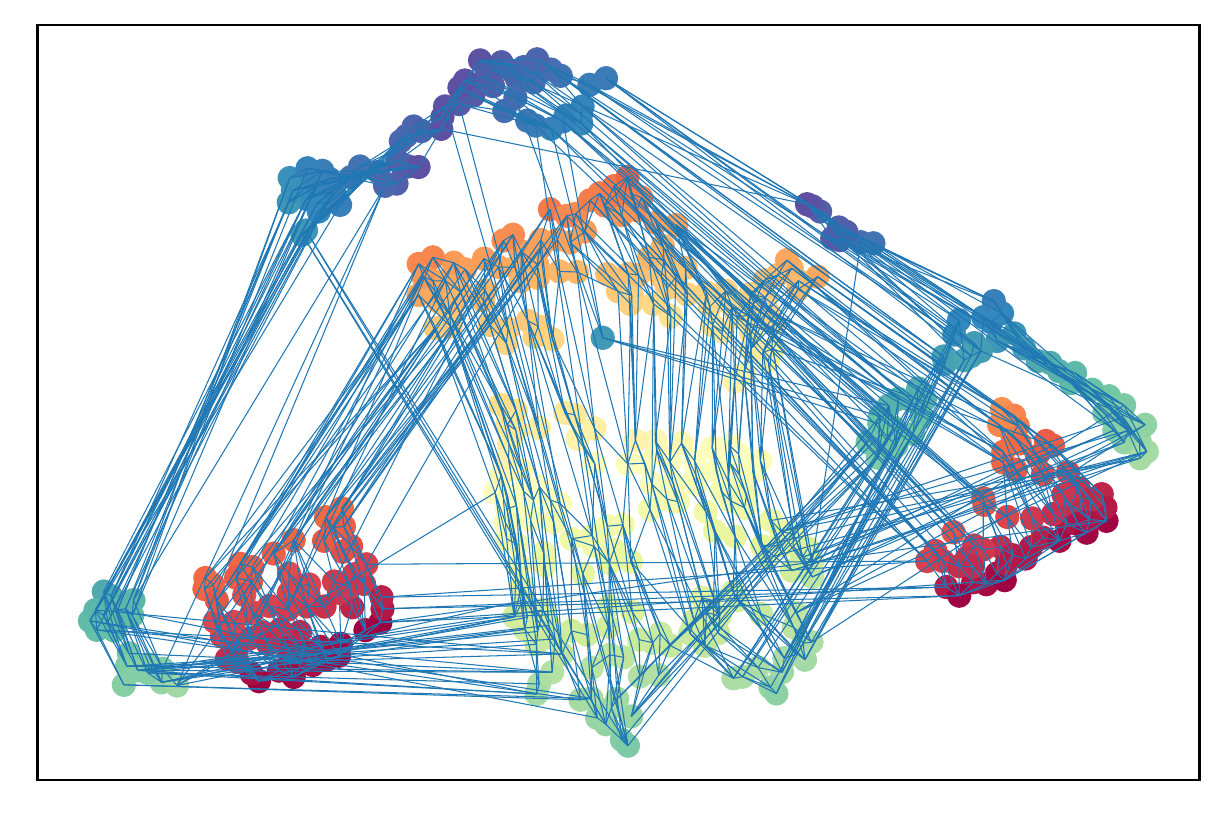}}\;\;\;
\caption{Other methods on Swiss roll. (a) 4585 crosses,(b) 1942 crosses, (c) 937 crosses, (d) 3623 cross, (e) 36773 crosses, (f) 3804 crosses, (g) 10088 crosses}   
\label{other_methods_swiss_roll}
\end{figure}
As we can see, all results in Figure~\ref{other_methods_swiss_roll} are with line crosses, indicating that the connectivity between triangles is not preserved, thus the mapping induced from these methods is not one-to-one. We now show the result for the manifold without boundary e.g. 2-sphere. Note that we only need one round of FPLM to finish the entire process. This is because we assumed that the sample $\mathbf{X}$ is a subset of the manifold, hence the triangulation conducted on $\mathbf{X}$ is always with a boundary. Thus, it is reasonable for us to only use one round of FPLM  given any single triangle can be served as the boundary.

\begin{figure}[H]
     \centering
     \subfloat[]{\includegraphics[width = 0.2\textwidth, height = 0.2\textwidth]{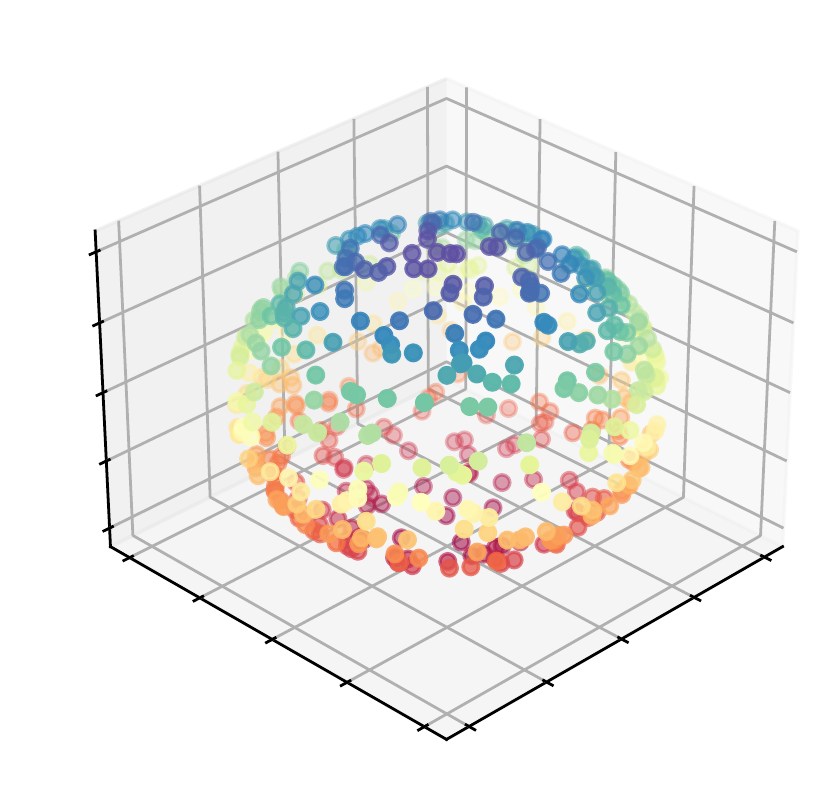}} \;\;\;
     \subfloat[]{\includegraphics[width =0.2\textwidth, height = 0.2\textwidth]{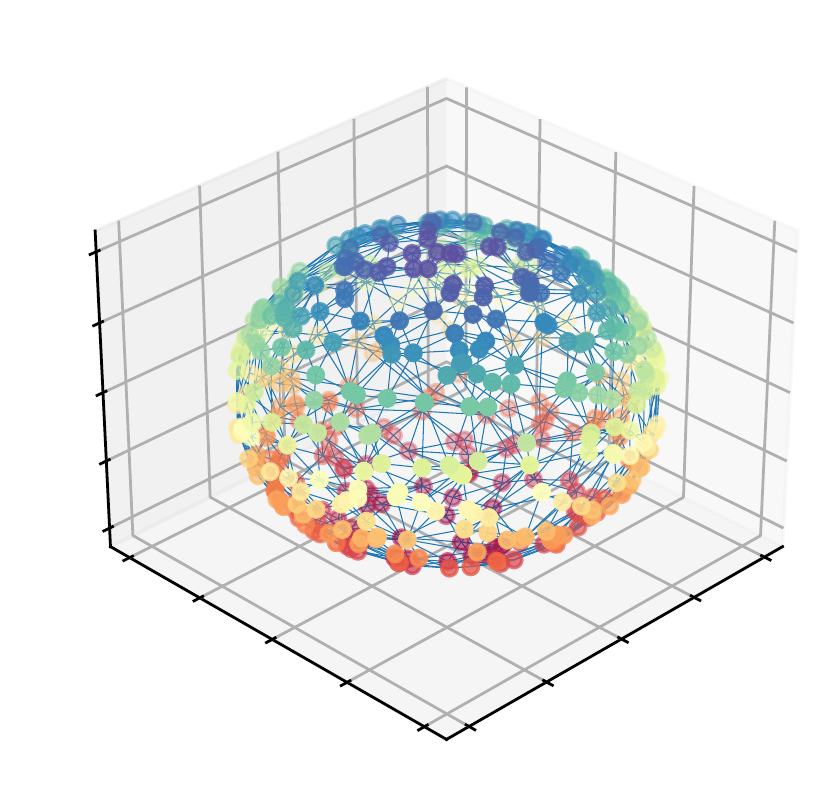}}\;\;\; 
    \subfloat[]{\includegraphics[width =0.2\textwidth, height = 0.2\textwidth]{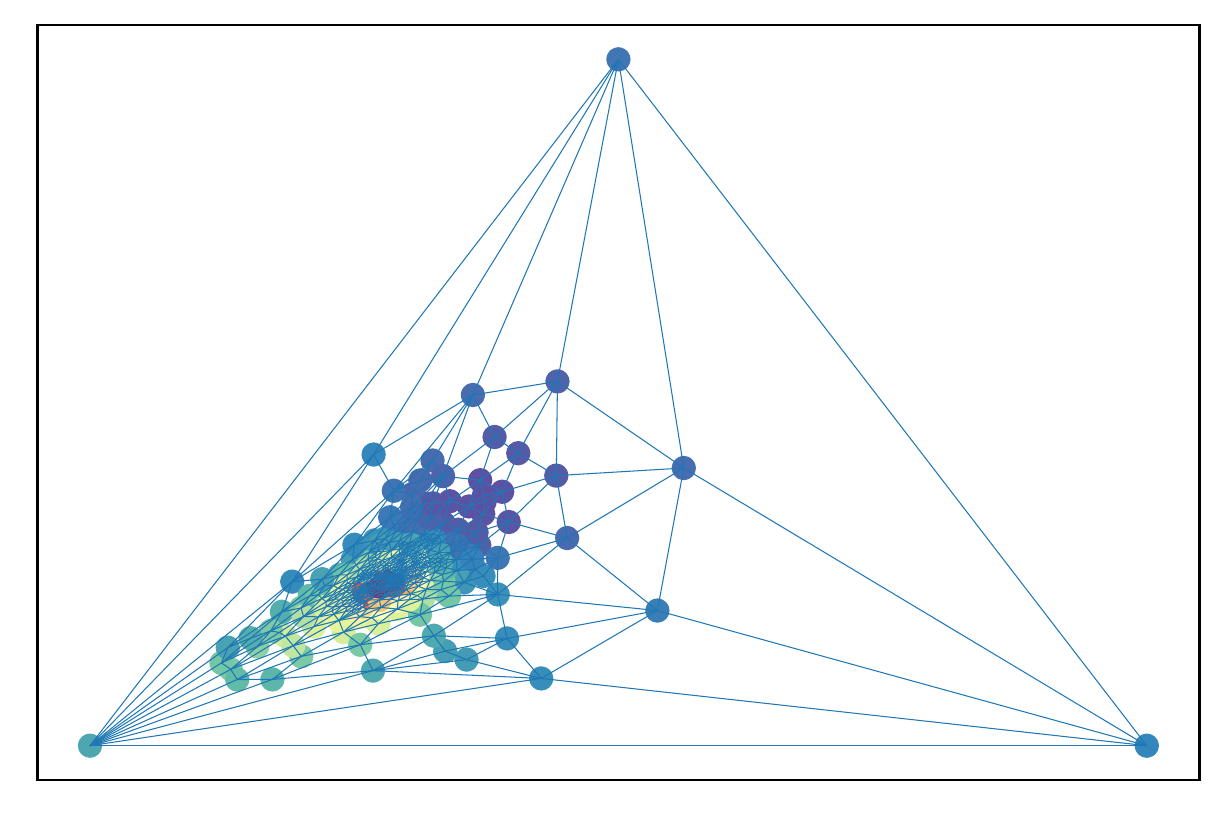}}\;\;\;
\caption{FPLM on Sphere} 
\end{figure}

\begin{figure}[H]
     \centering
     \subfloat[AE]{\includegraphics[width = 0.11\textwidth, height = 0.13\textwidth]{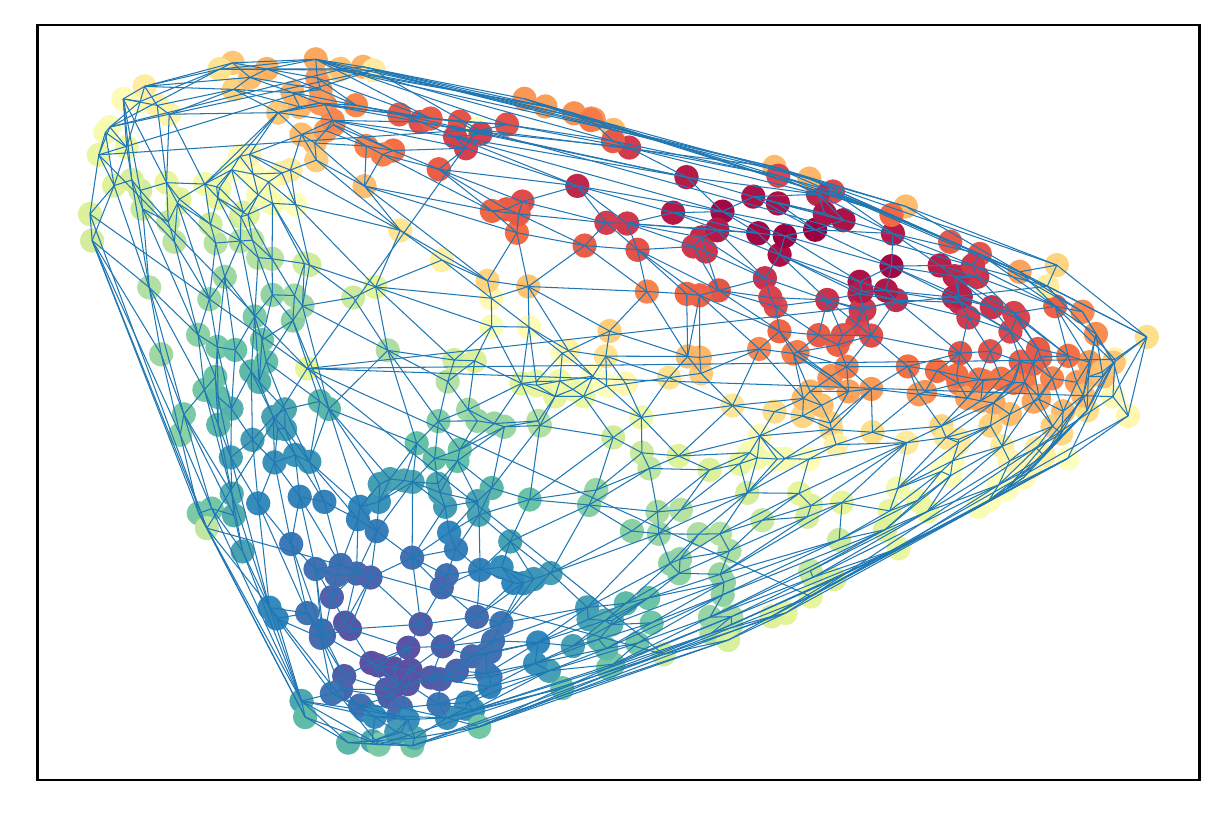}} \;\;\;
     \subfloat[Isomap]{\includegraphics[width = 0.11\textwidth, height = 0.13\textwidth]{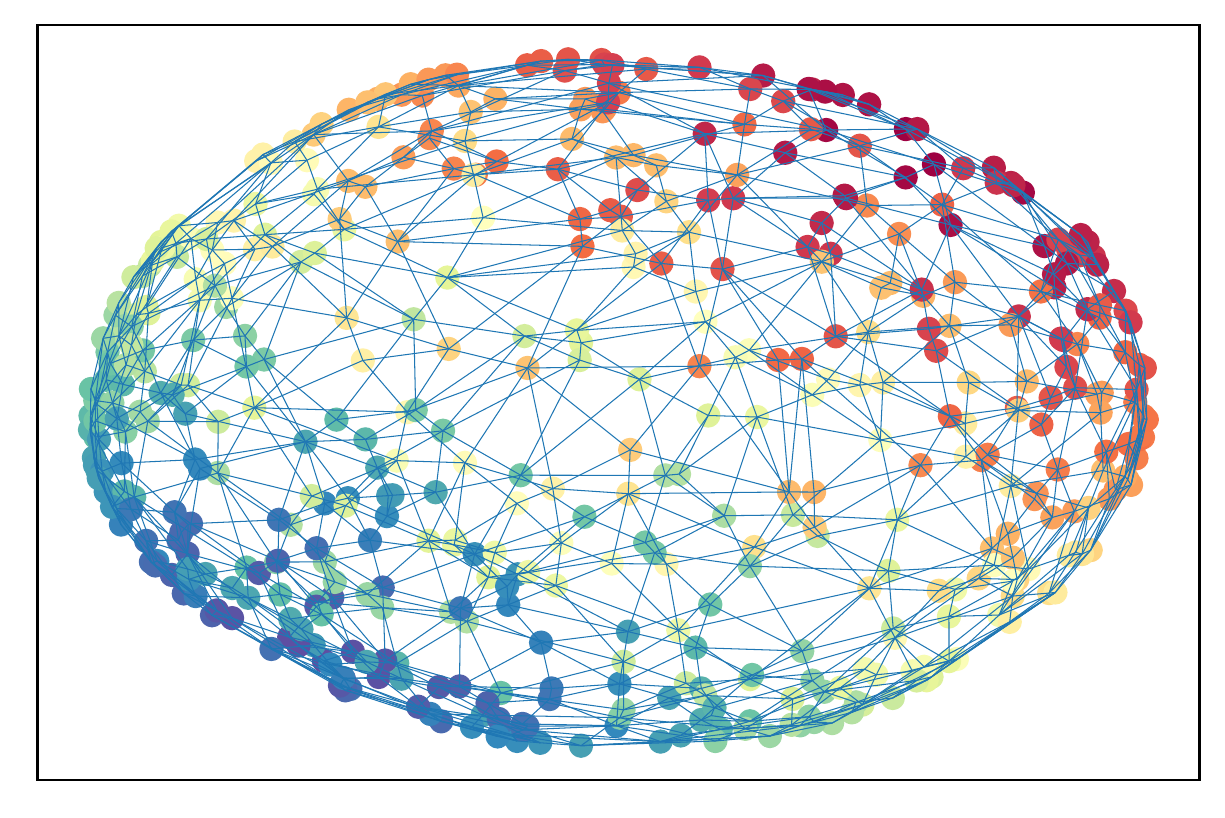}}\;\;\;
     \subfloat[LE]{\includegraphics[width = 0.11\textwidth, height = 0.13\textwidth]{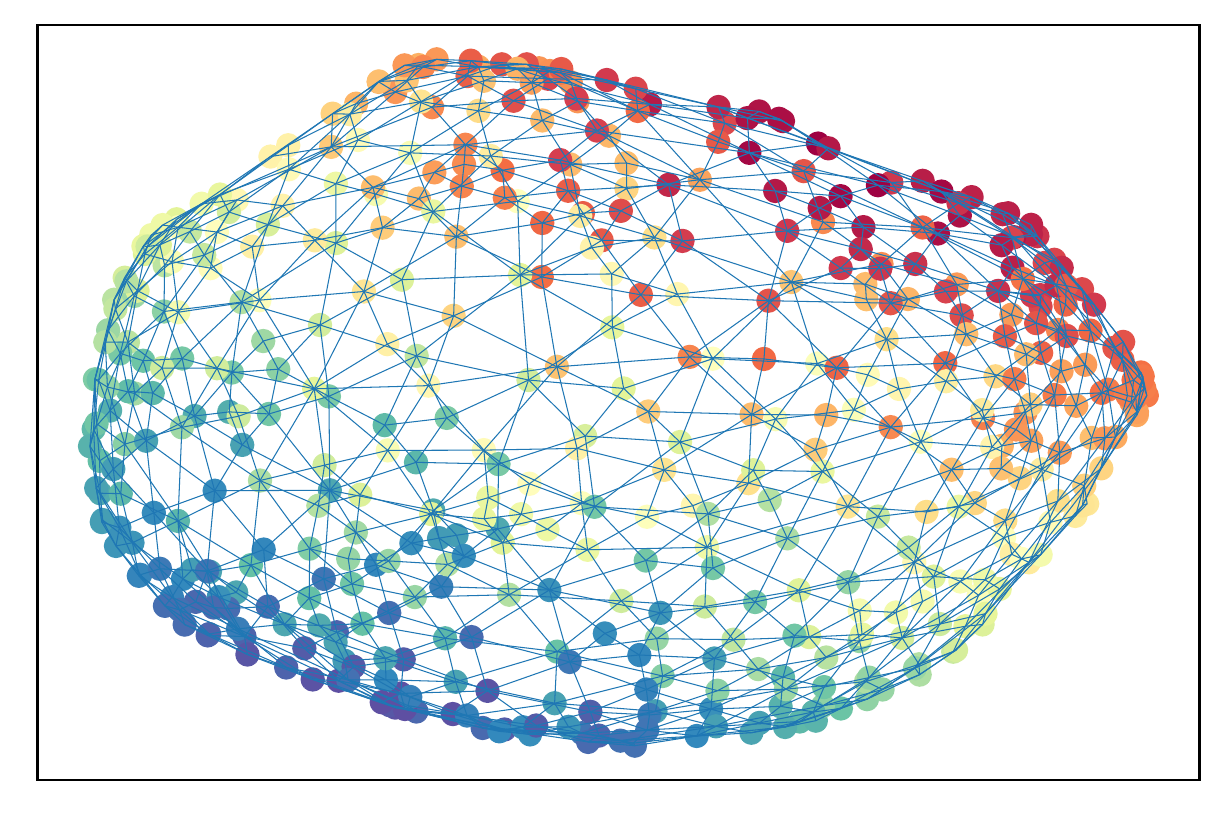}}\;\;\;
     \subfloat[LLE]{\includegraphics[width = 0.11\textwidth, height = 0.13\textwidth]{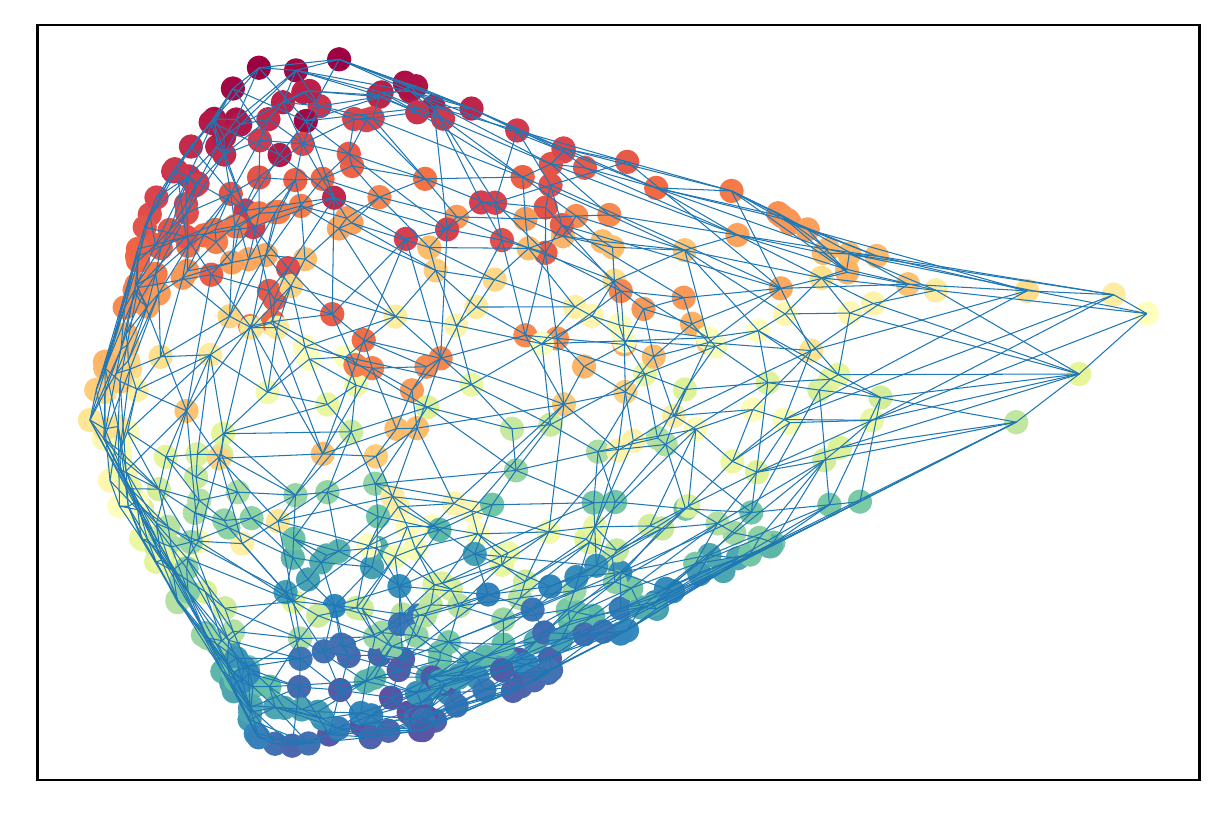}}\;\;\;
     \subfloat[LTSA]{\includegraphics[width = 0.11\textwidth, height = 0.13\textwidth]{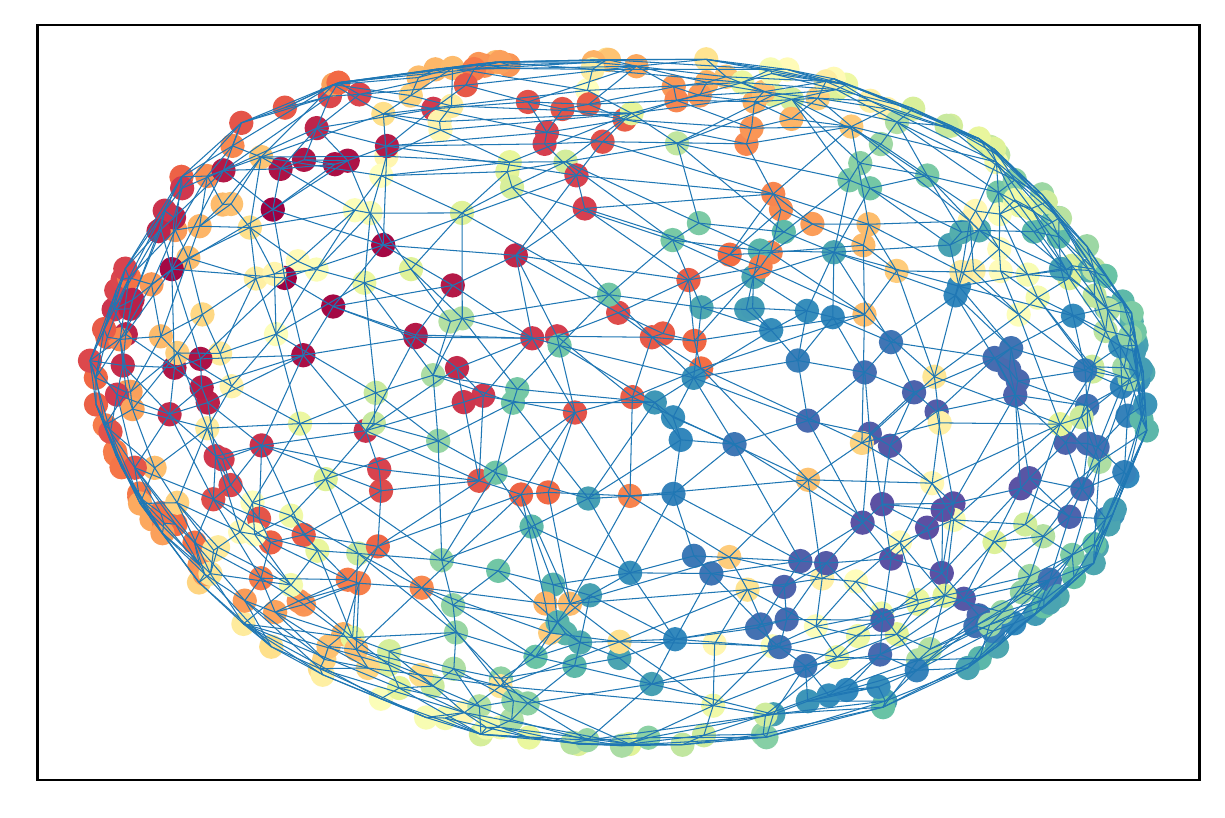}}\;\;\;
     \subfloat[MDS]{\includegraphics[width = 0.11\textwidth, height = 0.13\textwidth]{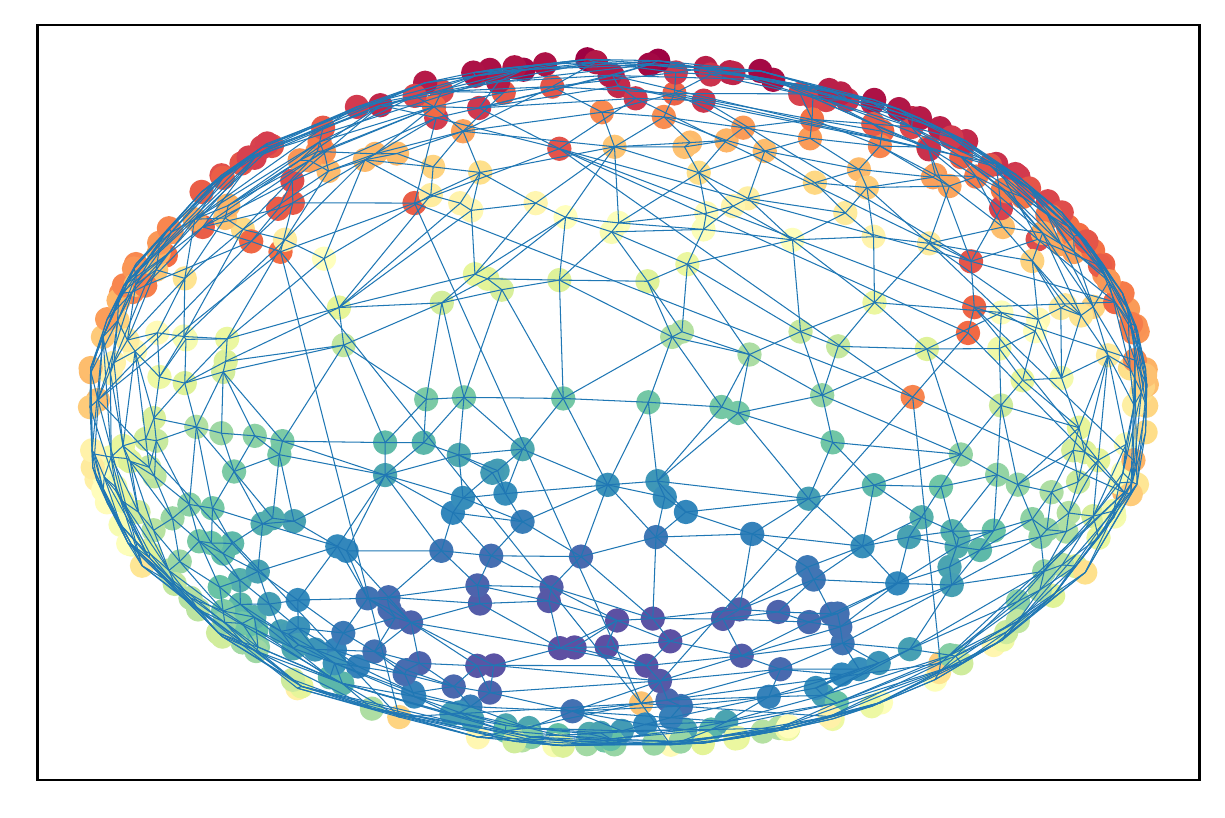}}\;\;\;
     \subfloat[t-SNE]{\includegraphics[width = 0.11\textwidth, height = 0.13\textwidth]{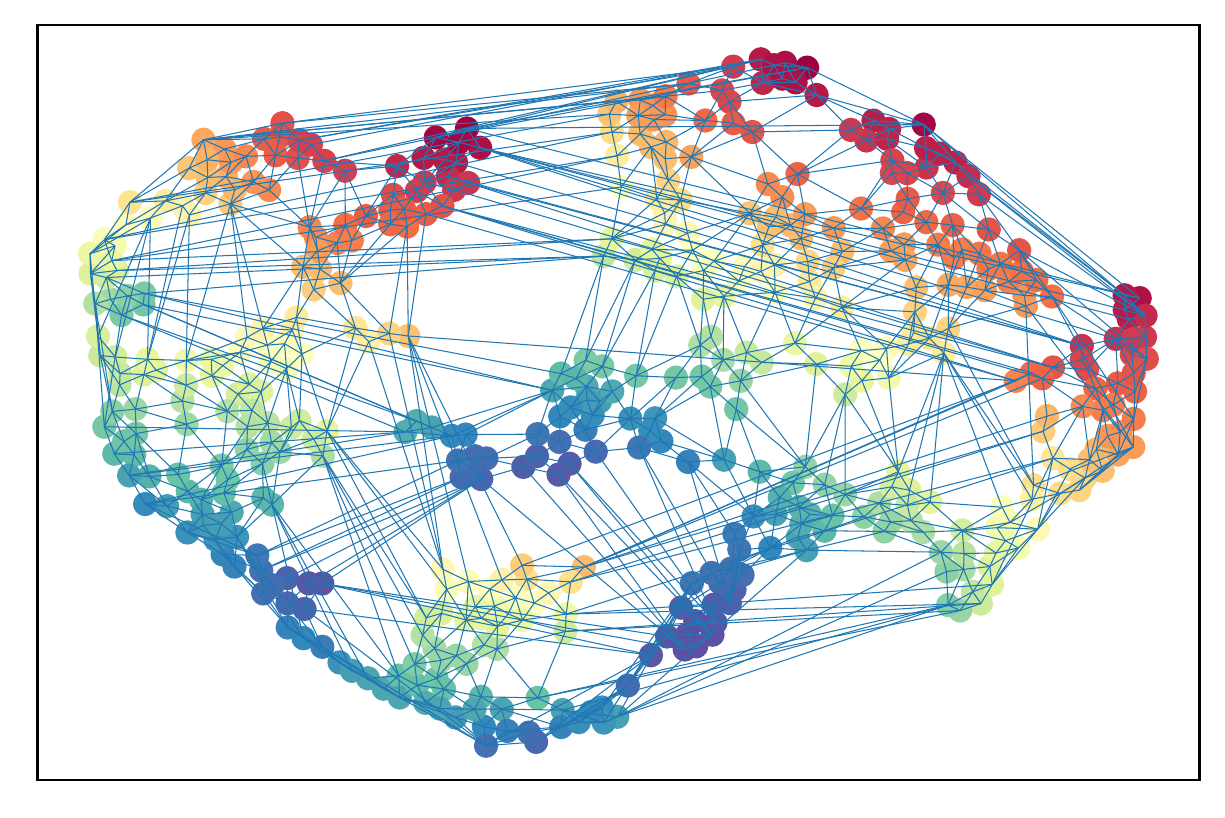}}\;\;\;
\caption{Other methods on Sphere. (a) 770 crosses,(b) 1786 crosses, (c) 1683 crosses, (d) 1883 cross, (e) 1795 crosses, (f) 1667 crosses, (g) 2329 crosses}   
\label{other_methods_sphere}
\end{figure}


\subsection{FPLM on 3-manifolds}
To show the learning performance of FPLM on any given tetrahedron mesh, 
we use both Delaunay tetrahedralization algorithm described in Tetgen \cite{si2015tetgen} and TC to create tetrahedral meshes in $\mathbb R^3$. Given that all the $d$-manifolds we considered in this paper can at least be embedded into $\mathbb R^{d+1}$; hence the points that we simulated in the manifold latent space can always be embedded in at least $\mathbb R^4$ without self-intersection. Note that the boundary of tetrahedral mesh can be detected as the $2$-simplex that is not shared in  tetrahedral mesh. Due to the variety of embedding functions from $\mathbb R^3$ to $\mathbb R^4$, the boundary detect from $\mathbb R^3$ will be, in general, different from the boundary of the manifold in $\mathbb R^4 $ or higher. Moreover, for the first round of FPLM, the fixed points will be the vertices of randomly selected tetrahedral; for the second round of FPLM, the fixed points will be the vertices of the polytope directly detected from tetrahedralization mesh. The following figure shows FPLM results on tetrahedral mesh of 3-ball.

\begin{figure}[H]
     \centering
     \subfloat[]{\includegraphics[width = 0.2\textwidth, height = 0.15\textwidth]{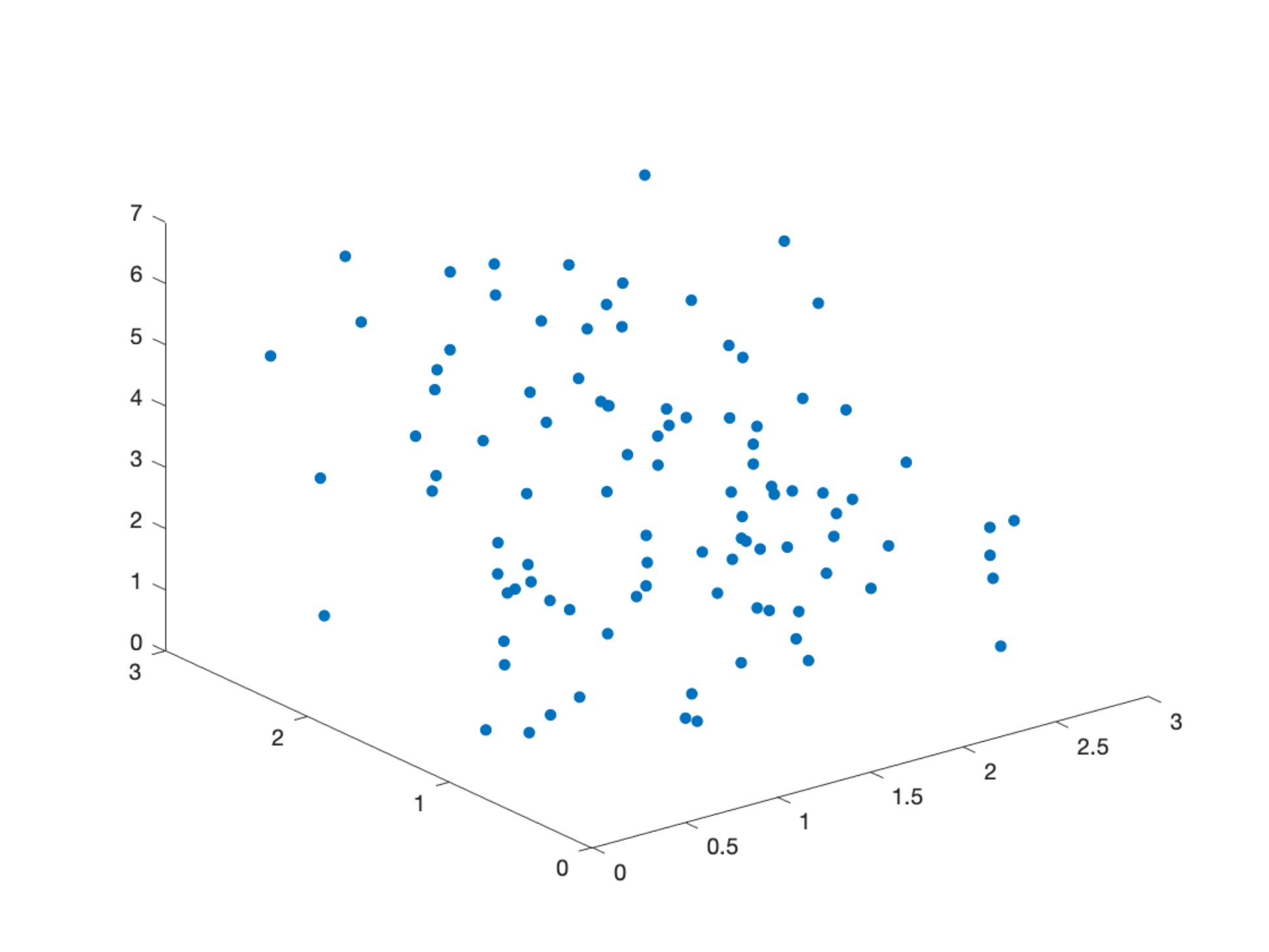}} \;\;\;
     \subfloat[]{\includegraphics[width =0.2\textwidth, height = 0.15\textwidth]{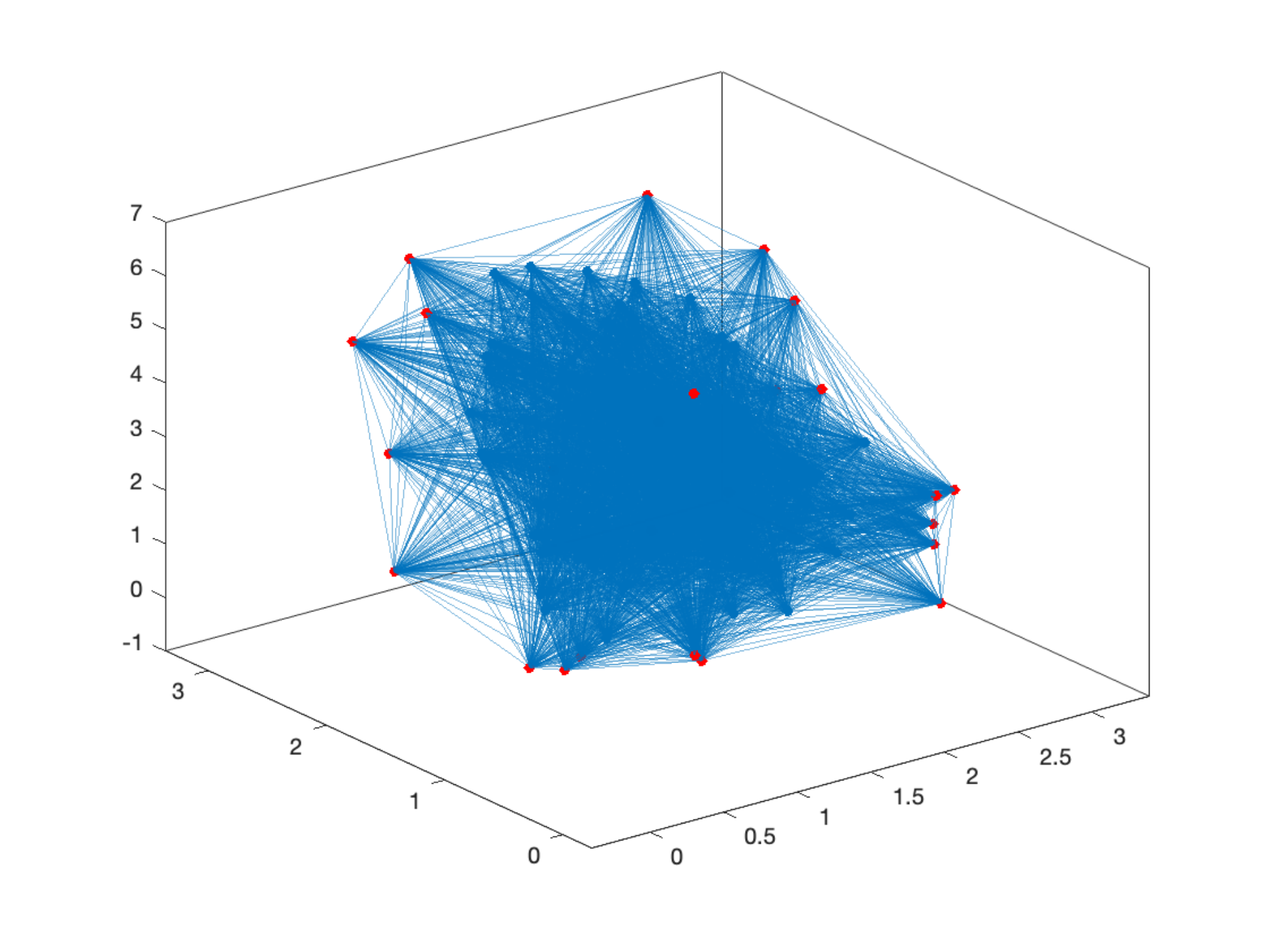}}\;\;\;
      \subfloat[]{\includegraphics[width =0.2\textwidth, height = 0.15\textwidth]{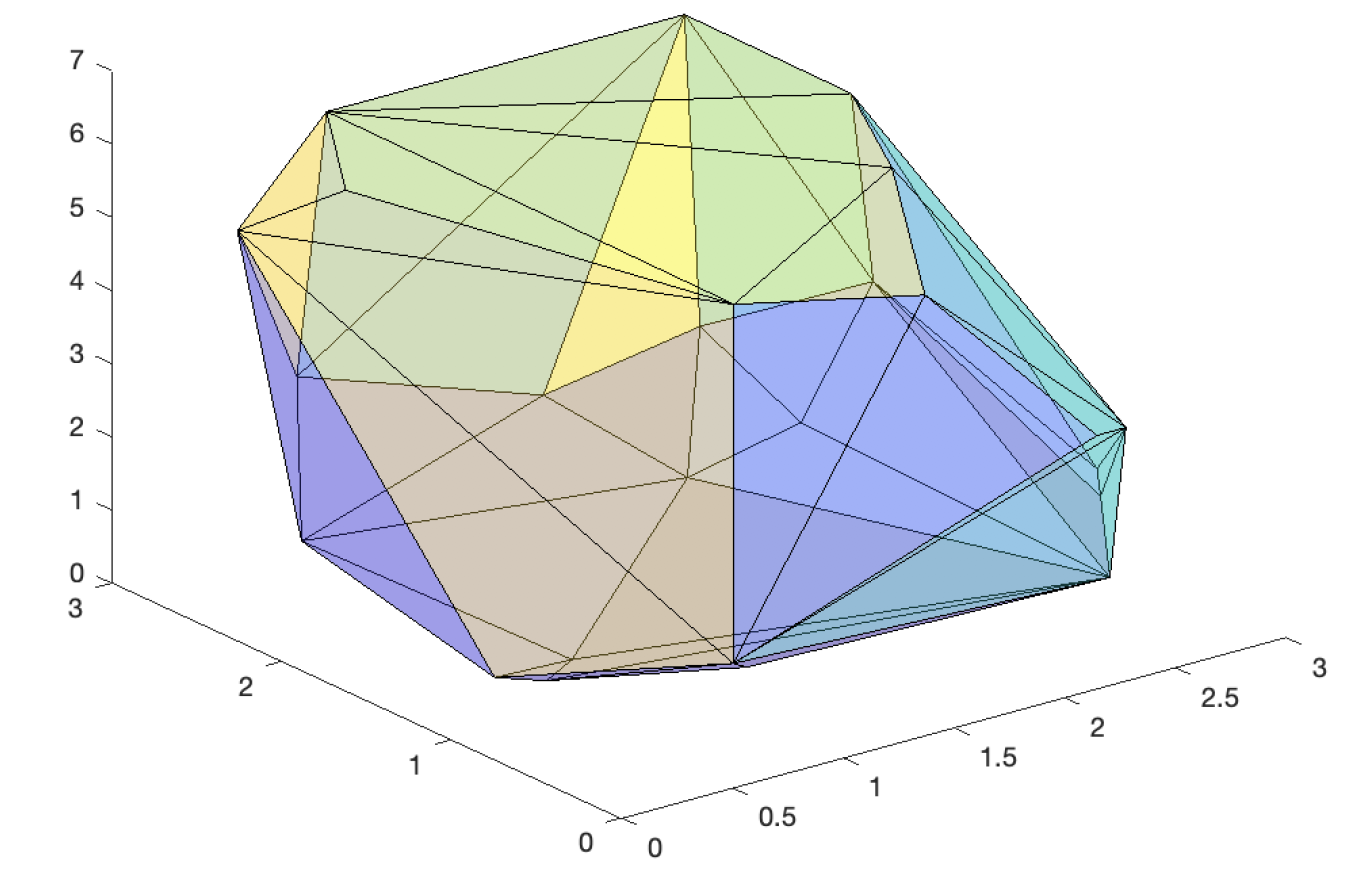}}\;\;\;
     \subfloat[]{\includegraphics[width =0.2\textwidth, height = 0.15\textwidth]{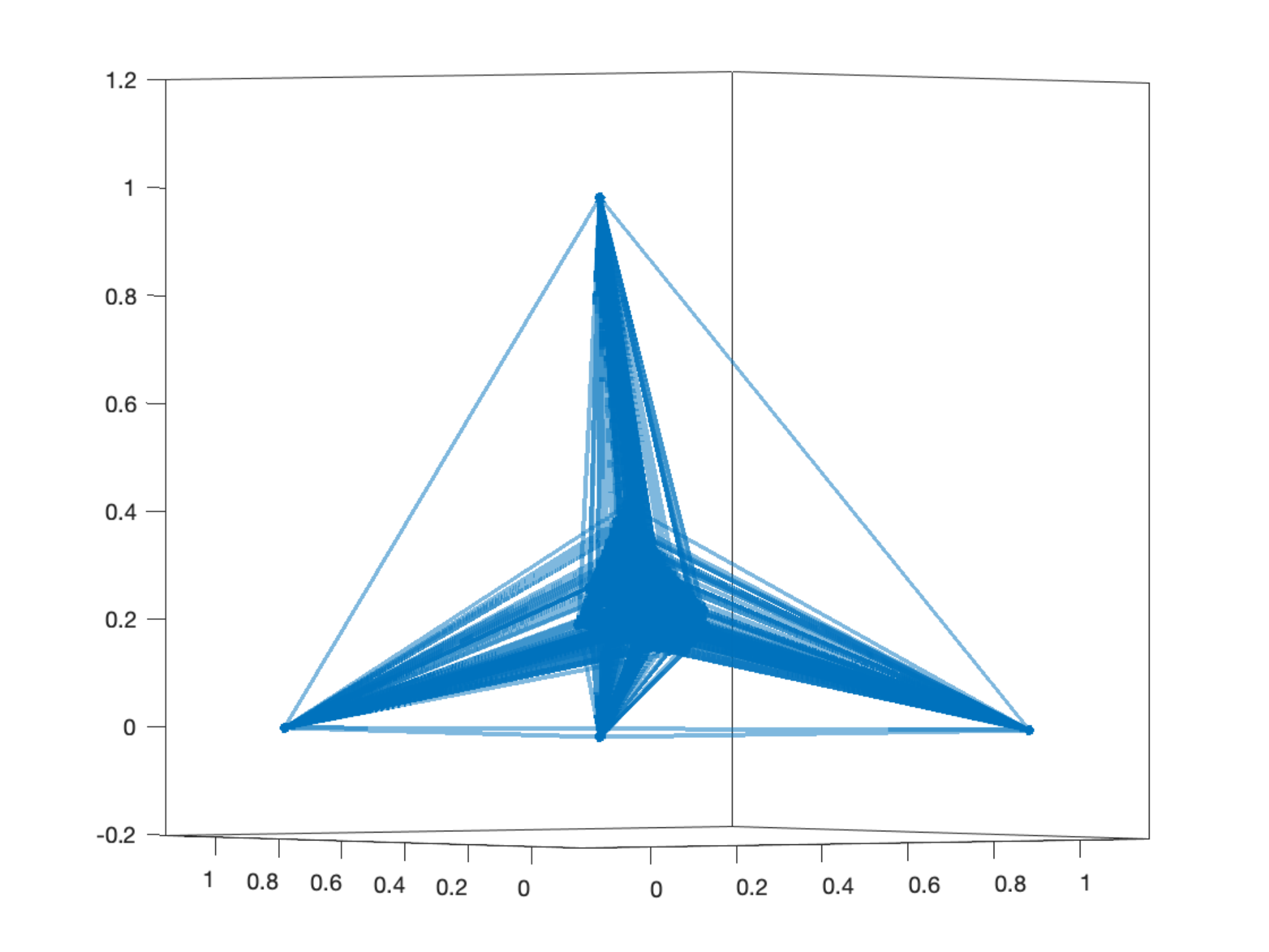}}\;\;\;
\caption{FPLM on 3-ball ground truth latent variables ($\psi, \phi$ and $\theta$), where  $\phi$ and $\theta$ are from $[0,\pi]$ and $\psi$ is from $[0,2\pi]$. Figure (a) Scatter plot on 3-ball (will be 3-sphere in $\mathbb R^4$) in $\mathbb R^3$. (b):  Tetrahedralization  (c): Boundary face detection (d): FPLM on 3-ball.}   
\label{FPLM_three_ball}
\end{figure}

By counting the number of intersections between the planes formed by the faces (triangles) of tetrahedrons (Figure~\ref{FPLM_three_ball}(d)), we found that the result generated from FPLM perfectly preserved the structure of the manifold, since all planes are only intersect with either a common edge, or point. For 3-manifold with boundary, we use the famous ``Delaunay Example'' tetrahedral mesh provided in python vista \cite{sullivan2019pyvista} to check the performance of FPLM. In addition, to show a better visualization result, we plot a subset of the tetrahedralization result by visualizing the tetrahedron below the (x,y) plane. The FPLM process, however, will still conduct using the entire dataset. By direct observation, we can see FPLM preserves the structure of tetrahedralization result. 

\begin{figure}[H]
     \centering
     \subfloat[]{\includegraphics[width = 0.15\textwidth, height = 0.15\textwidth]{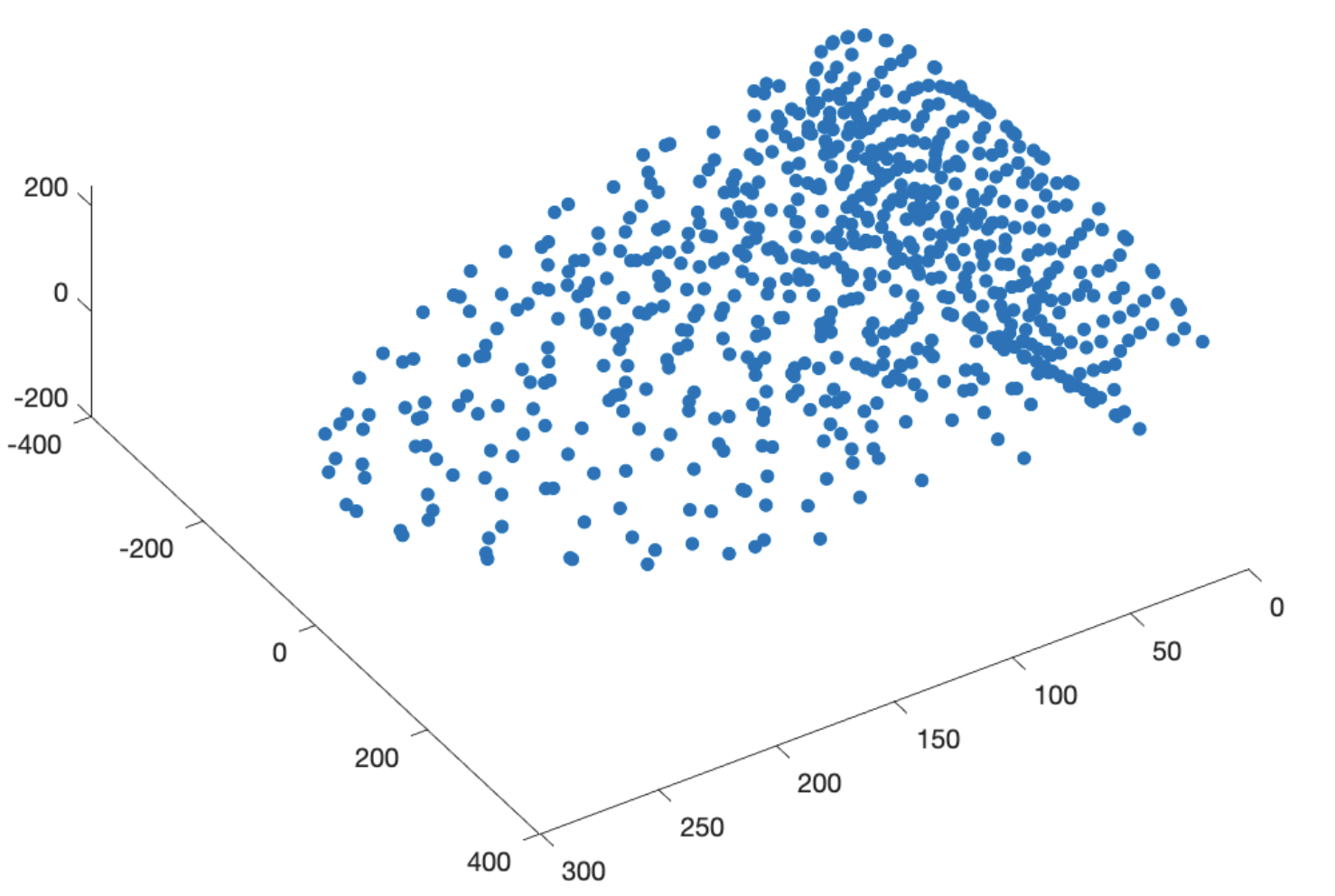}} \;\;\;
     \subfloat[]{\includegraphics[width =0.15\textwidth, height = 0.15\textwidth]{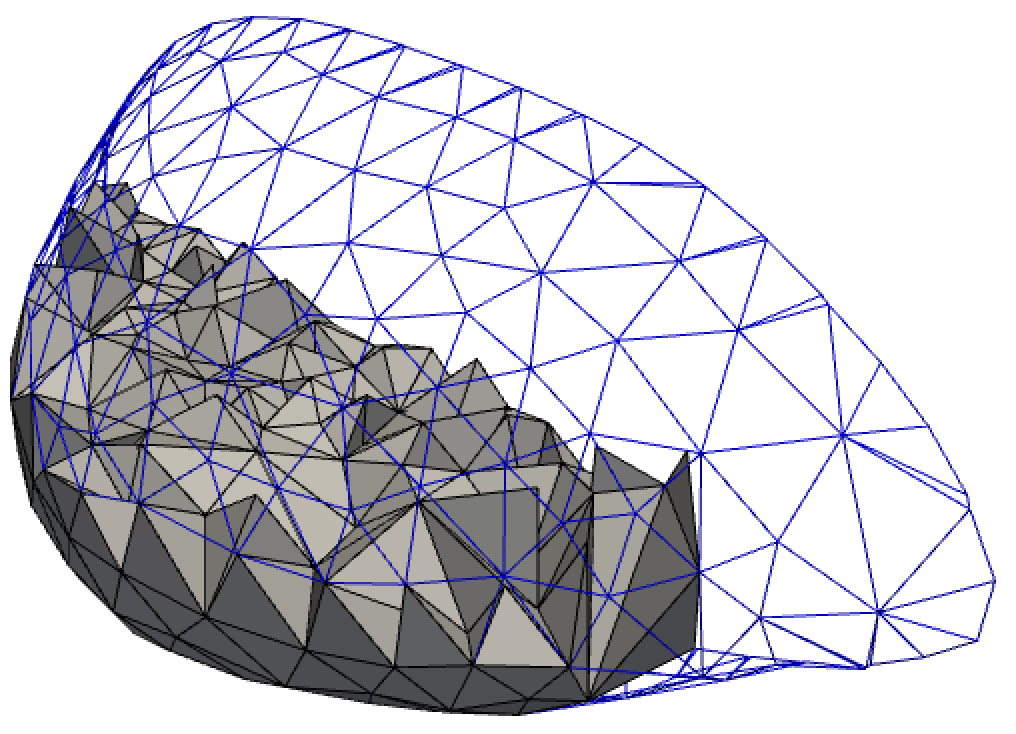}}\;\;\;
      \subfloat[]{\includegraphics[width =0.15\textwidth, height = 0.15\textwidth]{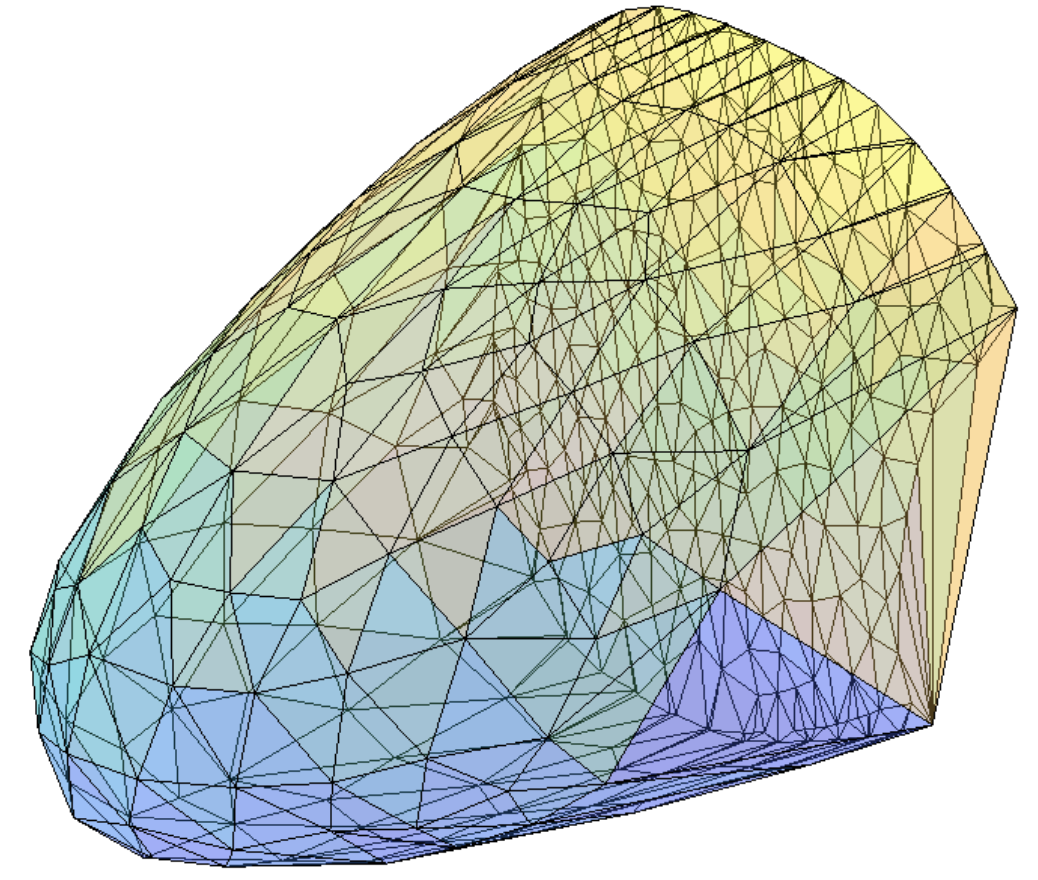}}\;\;\;
     \subfloat[]{\includegraphics[width =0.15\textwidth, height = 0.15\textwidth]{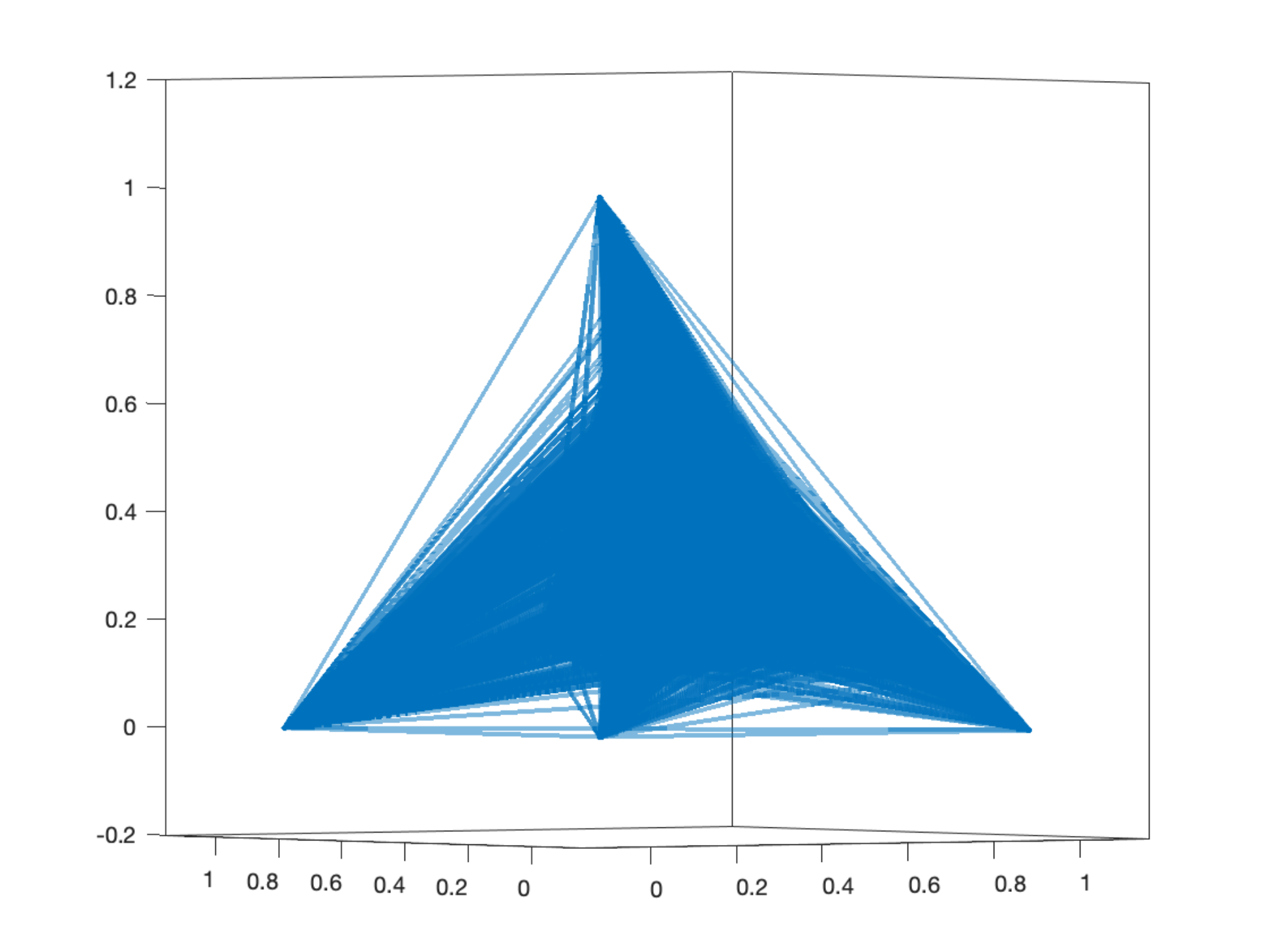}}\;\;\;
     \subfloat[]{\includegraphics[width =0.15\textwidth, height = 0.15\textwidth]{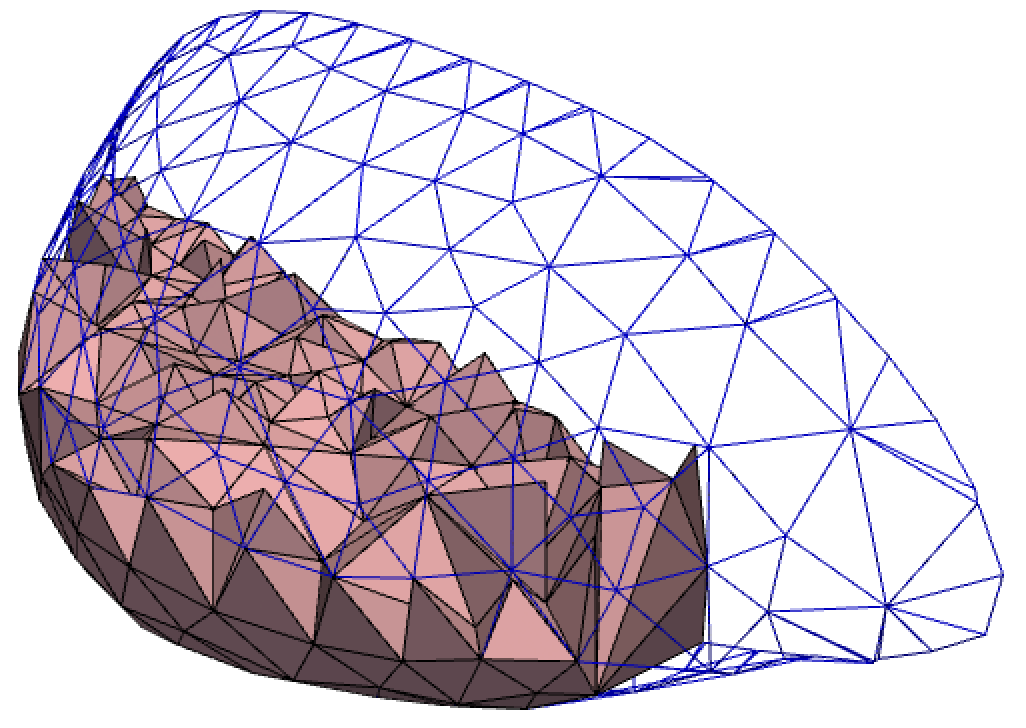}}\;\;\;
\caption{FPLM on tetrahedron mesh example: (a). Point scatter (b). tetrahedralization (c). Boundary detection  (d). First round FPLM (e). Second round FPLM}  
\label{FPLM_delaunay_example}
\end{figure}

\section{Discussion}\label{Discussion}

\paragraph{Boundary Advantage of FPLM}
To get the manifold structure, we apply a simplex decomposition algorithm (if exist) on $\mathcal M$. By our definition of the simplex decomposition, the boundary of the decomposition result will be a $(d-1)$-dimenional closed polytope formed by those $(d-1)$-simplices that are not shared. However, in practice, when $d$ is large, constructing a bijective mapping between boundaries can be as challenging as the original manifold learning problem. Fortunately, when $\mathcal{G}_{\mathcal{S}}$ is strongly connected, we evade this problem by two rounds FPLM in Algorithm 1 because the the boundary of the decomposition result is automatically determined. This is also the reason for the first round.

\paragraph{Limitation in Higher Dimensional Manifold Mesh}
It has been reported that convex combination may not be one-to-one if  $d \geq 3$.  A counter-example has been reported in \cite{floater2006convex}. However, in that particular counter-example, they created one point within a facet of a tetrahedron conflicting with our assumption on the discrete sample on manifold, i.e. all points are assumed in general position. Hence, this counter-example does not apply. We further point that orientation preserving (OP) is necessary for an algorithm with its induced mapping to be bijective for any  $d$-simplex decomposition \cite{lipman2014bijective}. Based on the conclusion from \cite{floater2003one}, we easily derive that FPLM is both local/global OP for connected orientable 2-manifolds due to its proven bijectivity over triangulation. However, when $d \geq 3$, the proof of OP in FPLM is still wanted. 

Nevertheless, we hypothesize that FPLM will always be bijective in high dimensional manifolds under some conditions. One can understand the process of FPLM as to draw a $d$-simplex decomposition result in $\mathbb R^d$ at the same time minimizing the sum of distances. The minimization process in FPLM is equivalent to minimizing the Dirichlet energy of the piece-wise linear mapping $\phi$. It is well-known that Delaunay simplex decomposition minimizes Dirichlet energy of the piece-wise linear function \cite{rippa1990minimal}, suggesting that FPLM could possibly map to Delaunay simplex decomposition. A solid mathematical proof will be sought in future work. 

\section{Conclusion Remarks}\label{conclusion}

\paragraph{Summary of the Paper}\label{summary_of_the_paper}
This paper explores the learning performances of the most widely used state-of-the-art dimensionality reduction algorithms by assessing whether these methods can generate a valid latent space representation aligning with the basic definition of manifold, which is the bijectivity of its chart map. We show that the mapping induced by all examined DR/ML are not one-to-one. Hence they are not learning the manifold in the mathematical sense. We develop a method, two-round FPLM, with the geometric guarantee of bijectivity in its induced map. 
From the experimental results, we found that two-round FPLM can perfectly deal with many  2-manifolds and some 3-manifolds. A future study is to investigate if this procedure has injectivity/bijectivity for manifold with arbitrary dimension.

\bibliographystyle{plain}
\bibliography{ref}

\begin{thebibliography}{10}

\bibitem{balasubramanian2002isomap}
Mukund Balasubramanian, Eric~L Schwartz, Joshua~B Tenenbaum, Vin de~Silva, and
  John~C Langford.
\newblock The isomap algorithm and topological stability.
\newblock {\em Science}, 295(5552):7--7, 2002.

\bibitem{belkin2003laplacian}
Mikhail Belkin and Partha Niyogi.
\newblock Laplacian eigenmaps for dimensionality reduction and data
  representation.
\newblock {\em Neural computation}, 15(6):1373--1396, 2003.

\bibitem{boissonnat2018delaunay}
Jean-Daniel Boissonnat, Ramsay Dyer, and Arijit Ghosh.
\newblock Delaunay triangulation of manifolds.
\newblock {\em Foundations of Computational Mathematics}, 18(2):399--431, 2018.

\bibitem{boissonnat2014manifold}
Jean-Daniel Boissonnat and Arijit Ghosh.
\newblock Manifold reconstruction using tangential delaunay complexes.
\newblock {\em Discrete \& Computational Geometry}, 51(1):221--267, 2014.

\bibitem{boyd2004convex}
Stephen Boyd, Stephen~P Boyd, and Lieven Vandenberghe.
\newblock {\em Convex optimization}.
\newblock Cambridge university press, 2004.

\bibitem{cazals2006delaunay}
Fr{\'e}d{\'e}ric Cazals and Joachim Giesen.
\newblock Delaunay triangulation based surface reconstruction.
\newblock In {\em Effective computational geometry for curves and surfaces},
  pages 231--276. Springer, 2006.

\bibitem{chen2011locally}
Jing Chen and Yang Liu.
\newblock Locally linear embedding: a survey.
\newblock {\em Artificial Intelligence Review}, 36(1):29--48, 2011.

\bibitem{courant2005dirichlet}
Richard Courant.
\newblock {\em Dirichlet's principle, conformal mapping, and minimal surfaces}.
\newblock Courier Corporation, 2005.

\bibitem{cox2008multidimensional}
Michael~AA Cox and Trevor~F Cox.
\newblock Multidimensional scaling.
\newblock In {\em Handbook of data visualization}, pages 315--347. Springer,
  2008.

\bibitem{donoho2003hessian}
David~L Donoho and Carrie Grimes.
\newblock Hessian eigenmaps: Locally linear embedding techniques for
  high-dimensional data.
\newblock {\em Proceedings of the National Academy of Sciences},
  100(10):5591--5596, 2003.

\bibitem{floater2003one}
Michael Floater.
\newblock One-to-one piecewise linear mappings over triangulations.
\newblock {\em Mathematics of Computation}, 72(242):685--696, 2003.

\bibitem{floater2006convex}
Michael~S Floater and Val{\'e}rie Pham-Trong.
\newblock Convex combination maps over triangulations, tilings, and tetrahedral
  meshes.
\newblock {\em Advances in Computational Mathematics}, 25(4):347--356, 2006.

\bibitem{flototto2003coordinate}
Julia Fl{\"o}totto.
\newblock {\em A coordinate system associated to a point cloud issued from a
  manifold: definition, properties and applications}.
\newblock PhD thesis, Universit{\'e} Nice Sophia Antipolis, 2003.

\bibitem{freedman2002efficient}
Daniel Freedman.
\newblock Efficient simplicial reconstructions of manifolds from their samples.
\newblock {\em IEEE transactions on pattern analysis and machine intelligence},
  24(10):1349--1357, 2002.

\bibitem{george1991automatic}
Paul~Louis George, Fr{\'e}d{\'e}ric Hecht, and {\'E}ric Saltel.
\newblock Automatic mesh generator with specified boundary.
\newblock {\em Computer methods in applied mechanics and engineering},
  92(3):269--288, 1991.

\bibitem{guillemin2010differential}
Victor Guillemin and Alan Pollack.
\newblock {\em Differential topology}, volume 370.
\newblock American Mathematical Soc., 2010.

\bibitem{kneser1926losung}
Hellmuth Kneser.
\newblock Losung der aufgabe 41.
\newblock {\em Jahresber. Deutsche Meth.}, pages 123--124, 1926.

\bibitem{kuratowski1930probleme}
Casimir Kuratowski.
\newblock Sur le probleme des courbes gauches en topologie.
\newblock {\em Fundamenta mathematicae}, 15(1):271--283, 1930.

\bibitem{lipman2014bijective}
Yaron Lipman.
\newblock Bijective mappings of meshes with boundary and the degree in mesh
  processing.
\newblock {\em SIAM Journal on Imaging Sciences}, 7(2):1263--1283, 2014.

\bibitem{maceachren1987sampling}
Alan~M MacEachren and John~V Davidson.
\newblock Sampling and isometric mapping of continuous geographic surfaces.
\newblock {\em The American Cartographer}, 14(4):299--320, 1987.

\bibitem{maglo2012progressive}
Adrien Maglo, Cl{\'e}ment Courbet, Pierre Alliez, and C{\'e}line Hudelot.
\newblock Progressive compression of manifold polygon meshes.
\newblock {\em Computers \& Graphics}, 36(5):349--359, 2012.

\bibitem{munkres2018elements}
James~R Munkres.
\newblock {\em Elements of algebraic topology}.
\newblock CRC press, 2018.

\bibitem{mwangi2014review}
Benson Mwangi, Tian~Siva Tian, and Jair~C Soares.
\newblock A review of feature reduction techniques in neuroimaging.
\newblock {\em Neuroinformatics}, 12(2):229--244, 2014.

\bibitem{rippa1990minimal}
Samuel Rippa.
\newblock Minimal roughness property of the delaunay triangulation.
\newblock {\em Computer Aided Geometric Design}, 7(6):489--497, 1990.

\bibitem{roweis2000nonlinear}
Sam~T Roweis and Lawrence~K Saul.
\newblock Nonlinear dimensionality reduction by locally linear embedding.
\newblock {\em science}, 290(5500):2323--2326, 2000.

\bibitem{rui2016dimensionality}
Liu Rui, Hossein Nejati, and Ngai-Man Cheung.
\newblock Dimensionality reduction of brain imaging data using graph signal
  processing.
\newblock In {\em 2016 IEEE International Conference on Image Processing
  (ICIP)}, pages 1329--1333. IEEE, 2016.

\bibitem{ruppert1995delaunay}
Jim Ruppert.
\newblock A delaunay refinement algorithm for quality 2-dimensional mesh
  generation.
\newblock {\em Journal of algorithms}, 18(3):548--585, 1995.

\bibitem{si2015tetgen}
Hang Si.
\newblock Tetgen, a delaunay-based quality tetrahedral mesh generator.
\newblock {\em ACM Transactions on Mathematical Software (TOMS)}, 41(2):1--36,
  2015.

\bibitem{sullivan2019pyvista}
C~Bane Sullivan and Alexander~A Kaszynski.
\newblock Pyvista: 3d plotting and mesh analysis through a streamlined
  interface for the visualization toolkit (vtk).
\newblock {\em Journal of Open Source Software}, 4(37):1450, 2019.

\bibitem{tutte1963draw}
William~Thomas Tutte.
\newblock How to draw a graph.
\newblock {\em Proceedings of the London Mathematical Society}, 3(1):743--767,
  1963.

\bibitem{van2009dimensionality}
Laurens Van Der~Maaten, Eric Postma, and Jaap Van~den Herik.
\newblock Dimensionality reduction: a comparative.
\newblock {\em J Mach Learn Res}, 10(66-71):13, 2009.

\bibitem{whitney1944singularities}
Hassler Whitney.
\newblock The singularities of a smooth n-manifold in (2n-1)-space.
\newblock {\em Annals of Mathematics}, pages 247--293, 1944.

\bibitem{wold1987principal}
Svante Wold, Kim Esbensen, and Paul Geladi.
\newblock Principal component analysis.
\newblock {\em Chemometrics and intelligent laboratory systems}, 2(1-3):37--52,
  1987.

\bibitem{zhang2004principal}
Zhenyue Zhang and Hongyuan Zha.
\newblock Principal manifolds and nonlinear dimensionality reduction via
  tangent space alignment.
\newblock {\em SIAM journal on scientific computing}, 26(1):313--338, 2004.

\bibitem{zou2016novel}
Quan Zou, Jiancang Zeng, Liujuan Cao, and Rongrong Ji.
\newblock A novel features ranking metric with application to scalable visual
  and bioinformatics data classification.
\newblock {\em Neurocomputing}, 173:346--354, 2016.

\end{thebibliography}

\newpage

\appendix

\section{Appendix one: Analysis of FPLM and Geometric guarantees} \label{Append1}

In this Appendix one, we first provide the algebraic solution of FPLM; then prove that the graph from the triangulation on 2-manifold is planar; thirdly, we show that the constraints that we assign to each round of FPLM can ensure that FPLM is one-to-one over entire triangulation. We also show that FPLM can be used for any edge-to-edge tessellation of polygons on $2$-manifolds. 
For arbitrary dimensional manifold, we discuss sufficient conditions for bijectivity. Unfortunately, computational methods on discrete data sampled from manifold with orientation preserving is still an open question in mathematics.



\subsection{Algebraic solution of FPLM} \label{Algebraic solution of FPLM}
We show the optimization process of both two rounds of FPLM here. First, consider the optimization problem of FPLM in Section \ref{FPLM_algorithm}. We  set some elements in $\mathbf Y \in \mathbb{R}^{N \times d}$ to be fixed points $\mathbf C \in \mathbb{R}^{p \times d}$ ($p \geq d+1$) and optimize the rest. Therefore, we rearrange $\mathbf Y = [ \tilde{\mathbf Y}; \mathbf C ] $ with $\tilde{\mathbf Y} \in \mathbb{R}^{(N-p)\times d}$ being the unknowns and
\[
\mathbf L = \begin{bmatrix} \mathbf L_y & \mathbf L_{yc} \\[3pt] \mathbf L_{yc}^T & \mathbf L_c  \end{bmatrix}
\]
where $\mathbf L_{yc} \in \mathbb{R}^{(N-p)\times p}$. Therefore we can reformulate the problem as 
\begin{align*}
    \min_{ \tilde{\mathbf Y} \in \mathbb{R}^{(N-p) \times d}} \text{tr} ( \mathbf C^T \mathbf L_c \mathbf C + \tilde{\mathbf Y}^T \mathbf L_{yc} \mathbf C + \mathbf C^T \mathbf L_{yc}^T \tilde{\mathbf Y} + \tilde{\mathbf Y}^T \mathbf L_y \tilde{\mathbf Y}) \label{barycenter_mapping}
\end{align*}

This quadratic optimization problem is convex as $\mathbf L_y$ is positive definite. Then by first order condition, a global minimizer $\tilde{\mathbf Y}^*$ exists such that
\begin{equation}
    \mathbf L_{y} \tilde{\mathbf Y}^* + \mathbf L_{yc} \mathbf C = \boldsymbol 0, \text{ with solution } \tilde{\mathbf Y}^* = - \mathbf L_y^{-1} \mathbf L_{yc} \mathbf C, \label{foc_FPLM} 
\end{equation}
where $\mathbf L_y^{-1}$ is the  inverse of $\mathbf L_y$. Since Laplacian matrix acts as a difference operator on features, a geometric interpretation of \eqref{foc_FPLM} is that $\tilde{\mathbf Y}^*$ should have its sum of the weighted difference of its neighbours equal $\boldsymbol 0$, regardless of whether they are connected to the fixed points. This is obvious after rewriting \eqref{foc_FPLM} by components. That is, for any $\tilde{\mathbf y}_i^*$, $i = 1,...,n-p$,
\begin{equation}
    \mathbf D_{ii} \tilde{\mathbf y}_i^* - \sum_{j\in [1, n-p]} \mathbf A_{ij} \tilde{\mathbf y}_j^* - \sum_{l \in [n-p+1, n]} \mathbf A_{il} \mathbf c_l = \boldsymbol 0, \label{convex_combination_function}
\end{equation} 
where $\mathbf A_{ij}$ is the weight between sample $i, j$ and $\mathbf D_{ii}$ is the degree of $\mathbf{x}_i$, including the fixed points. We may further simplify \eqref{convex_combination_function} by considering $\mathbf Y^* = [ \tilde{\mathbf Y}^*; \mathbf C ]$. That is,
\begin{equation}
    {\mathbf y}_i^* = \sum_{j =1}^{n} \frac{\mathbf A_{ij}}{\mathbf D_{ii}} {\mathbf y}_j^*  = \sum_{j =1}^{n} \lambda_{ij} {\mathbf y}_j^*, \quad \forall{i = 1, ..., n-p}, \label{barycenter_mapping}
\end{equation} 
By definition of the degree matrix, we have $\sum_{j=1}^n \lambda_{ij} = 1$, $\forall i$. This shows that every optimal non-fixed point is a convex combination of points in its neighbourhood.

 \subsection{Planarity}\label{Planarity}


Given a triangulation on $\mathcal M$, we denote $G(V, E)$ as the graph containing the adjacency information of vertices and edges from $\mathcal T$. For any vertex $v$ in $G$, we denote $N^{\mathcal T}_v$ as the vertices set that contains $v$'s neighboring vertices directly connected to $v$ in $\mathcal T$. We also denote $E^{\mathcal T}_v$ as the set contains all the edges of $v$ as starting/ending point. Based on the second feature of simplex decomposition in definition 1, we  denote the boundary of manifold $\partial \mathcal M$ (from triangulation) as those edges that are only contained in one triangle. 
In this section, we will demonstrate that the graph induced from triangulation on $\mathcal M$ is planar that can be reduced into a subset of plane so that edges will only intersect at their endpoints. We now define planar graph by stating Kuratowski's theorem\cite{kuratowski1930probleme}:


\begin{thm}[Planar Graph,Kuratowski]\label{thm:planargraph}
A finite graph is planar if and only if it does not contain a sub-graph of the complete graph $K_5$ or the complete bipartite graph $K_{3,3}$ (utility graph).
\end{thm} 
We say a graph is complete if the graph is a simple undirected graph in which a unique edge connects every pair of distinct vertices. Figure~\ref{figure_Kuratowski}(a) shows a complete graph of five vertices $K_5$. We say a graph is a complete bipartite graph if there are two sets of vertices $U$ and $V$ and every vertex of the first set is connected to every vertex of the second set. Figure~\ref{figure_Kuratowski}(b) shows a complete bipartite graph ($K_{3,3})$ in which each vertex set contains 3 vertices. 

\begin{figure}[H]
     \centering
     \subfloat[$K_5$]{\includegraphics[width = 0.25\textwidth, height = 0.2\textwidth]{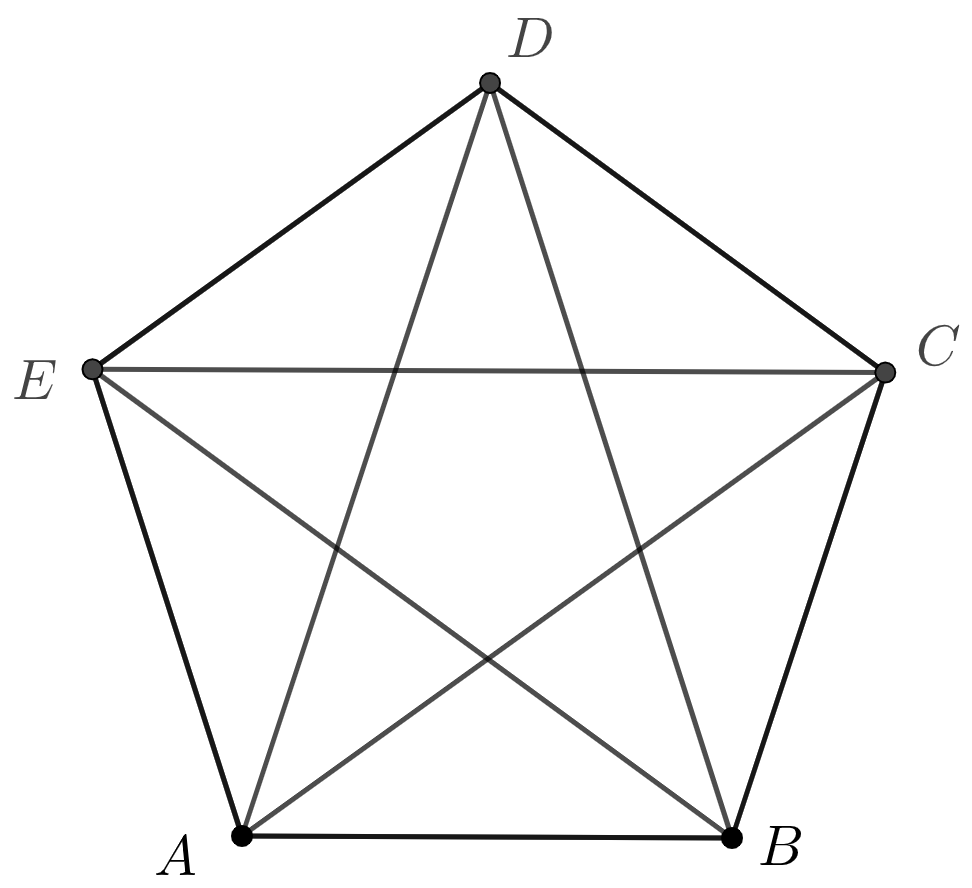}} 
     \hspace{0.8in}
     \subfloat[$K_{3,3}$]{\includegraphics[width =0.25\textwidth, height = 0.2\textwidth]{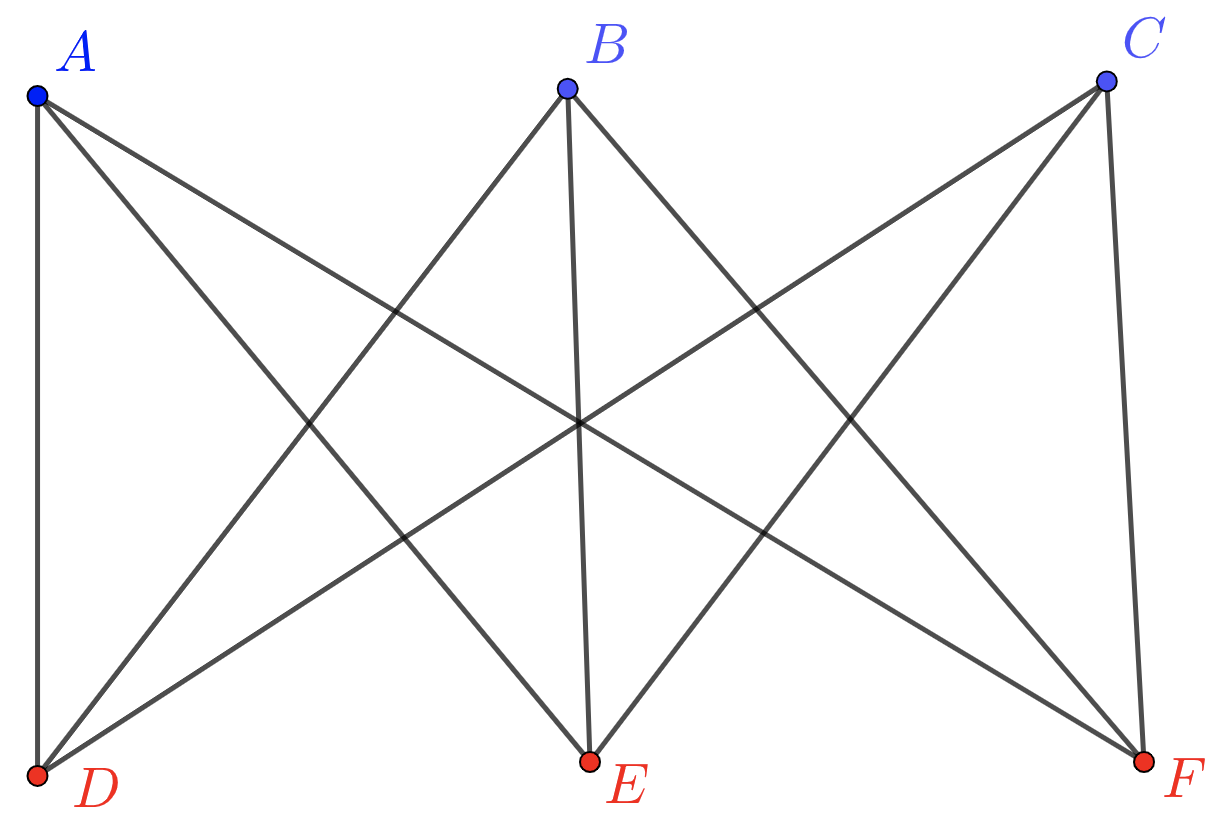}}
\caption{Kuratowski subgraph $K_5$ and $K_{3,3}$}   
\label{figure_Kuratowski}
\end{figure}

\subsection{Planarity of the graph induced from triangulation}

\begin{prop}\label{prop:nok5}
For a  triangulation on a 2-manifold in $\mathbb R^l$, there is no Kuratowski sub-graph $K_5$. 
\end{prop}

\begin{proof}
The proof is by contradiction. Assume there is $K_5$. By the definition of triangulation on manifold, the intersection of any pair of triangles is either empty, a common vertex, or a common edge. However, a $K_5$ in $\mathbb R^2$ has two triangles intersecting with thire edges of them. For example, in Figure~\ref{figure_Kuratowski} (a) The intersection of triangle T[A,B,D] and triangle T[A,B,C] is line segment [A,B] while their edges [A,C] and [B,D], intersects, contradicting the condition of triangulation on 2-manifold.

Furthermore, as the triangulation on the manifold is conducted on $\mathbb R^l$ (for example $l=3$), it is possible to have the situation that one of $K_5$'s vertices is lifted up in another dimension so that the entire $K_5$ in $\mathbb R^3$ becomes a pyramid shown in Figure~\ref{high_dim_k5} below. However, from the definition of triangulation, all edges can only be shared by at most once. From the figure below, it is clear to see that edge [C,D] is shared by triangle T[B,C,D], T[A,C,D],T[E,C,D], and that contradicts to the definition of triangulation. 



\end{proof}
\begin{figure}[H]
\centering
\includegraphics[width =0.3\textwidth, height = 0.3\textwidth]{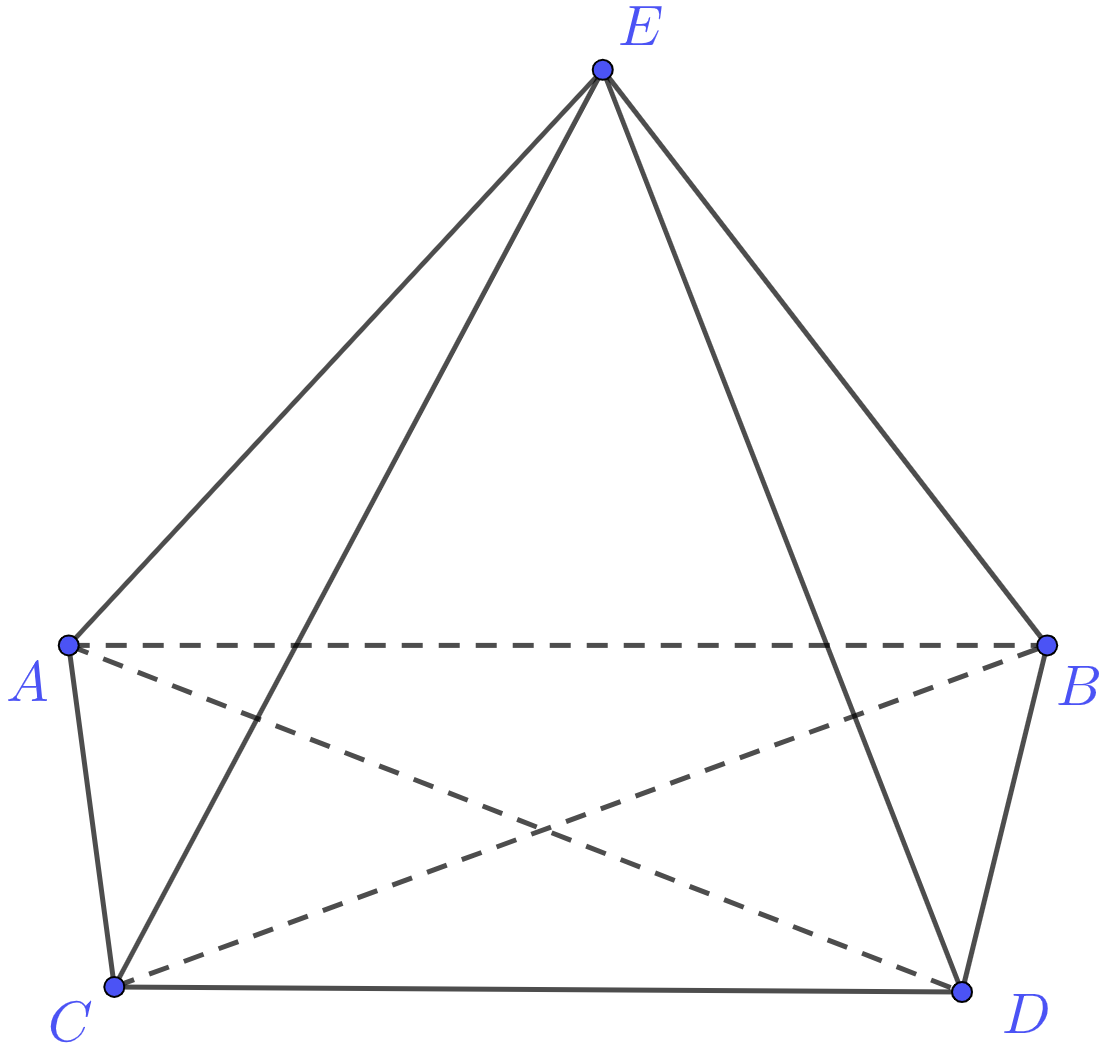}
\caption{Sub-graph of $K_5$ in $\mathbb R^3$}\label{high_dim_k5}
\end{figure}

\begin{prop}\label{prop:nok33}
For a triangulation on 2-manifold in $\mathbb R^l$, there is no Kuratowski sub-graph $K_{3,3}$. 
\end{prop}

\begin{proof}
The proof is also done by contradiction. Assume there is $K_{3,3}$. If $l=2$, the result is trivial as shown in Figur~\ref{figure_Kuratowski}(b):
$K_{3,3}$ in $\mathbb R^2$ is always with line-segment cross, and that is contradict to our definition of triangulation. 

If $l>2$, we have the situation shown in  Figur~\ref{high_dim_k33}, in which all vertices are in $\mathbb R^l$. Observe that there is no line-cross (edge intersection) contained in this high-dimensional $K_{3,3}$. 
However, the plane define by the vertices $[F,A,B]$ intersects the plane defined by vertices $[A,B,E]$ at line-segment $[A,B]$, indicating self-intersection of the manifold due to the fact that the simplex decomposition is homeomorphic to $\mathcal M$. Thus the manifold that contains such feature can only be immersed to $\mathbb R^l$ but not embedded. That leads to a contradiction to our basic assumption to 2-manifold being proper embedding. 

\end{proof}

\begin{figure}[H]
\centering
\includegraphics[width =0.3\textwidth, height = 0.3\textwidth]{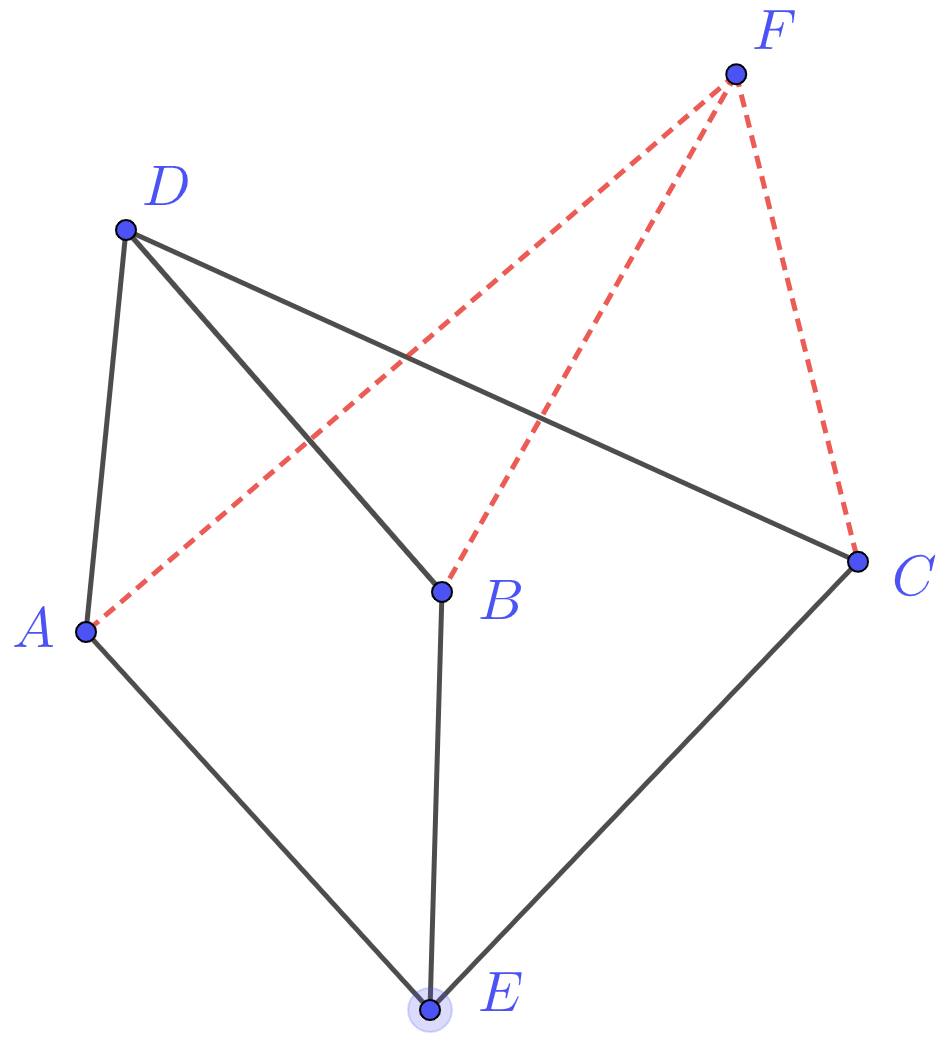}
\caption{Sub-graph of $K_{3,3}$ in $\mathbb R^3$}\label{high_dim_k33}
\end{figure}

Applying propositions \ref{prop:nok5} and \ref{prop:nok33} with theorem \ref{thm:planargraph} leads the claim in Theorem \ref{thm:2dmfdtriangulationplanar}. 

From  Kuratowski theorem, the triangulation $\mathcal T$ on $\mathcal M$ with the above features induces a \textit{planar straight line graph}.  Note that the triangulation we discuss can  be the results from any triangulation outputs such as surface triangulation and tangential complex.

\begin{rem}[Manifold without boundary] \label{manifold_without_boundary}
For some manifolds without boundary, e.g. two sphere $S^2$, it is well known that we can not map the entire manifold on the plane. One can only discretely sample from the underlying manifold and hence leave with many ``oles'' in the manifold. In other words, what we observed is not the entire sphere but a measure-less subset of it. Thus, it is reasonable to randomly select a triangle from this triangulation as the boundary of a ``new'' manifold without it, which is almost the same as $S^2$ but with the missing triangle. Figure~\ref{Sphere}below shows a triangle selected whose boundary serves as the boundary of this new manifold for the sampled points. There is no information loss as we know what we have taken away. 
\end{rem}
\begin{figure}[H]
\centering
\includegraphics[width =0.6\textwidth, height = 0.4\textwidth]{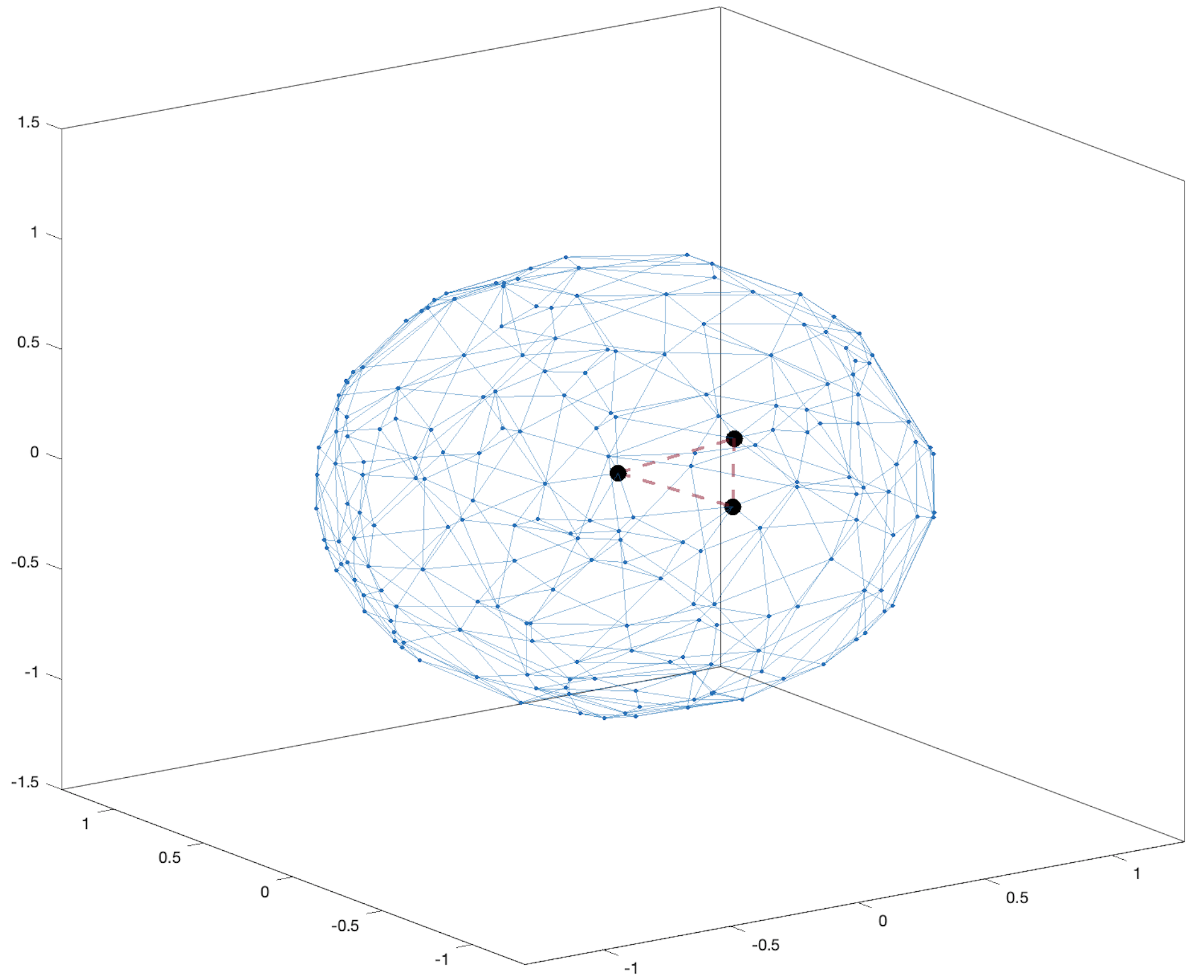}
\caption{Boundary triangle for a triangulation of the subset of 2D-sphere}\label{Sphere}
\end{figure}

\subsection{Piece-wise linear mapping of FPLM}\label{Piece-wise linear mapping of FPLM}

As we mentioned earlier, the connectivity between simplices formed on manifold should remain the same in both image and pre-image of a bijective function over the entire simplex decomposition. We let the image set of the chart map as $\mathbf Z =\{\mathbf z_1, \mathbf z_2,.. \mathbf z_N \} \in \mathbb R^d$. The exact position of every point in $ \mathbf Z$ is unknown; however, from the bijectivity of the chart map, we know that $ \mathbf Z$  perfectly preserves the connectivity between simplices. This means if we repeat the graph in $\mathbf Z$ using the adjacency information from the simplex decomposition ($\mathcal S$) on $\mathcal M$, the integrity, connectivity, and neighboring relations between simplices will remain unchanged. We write the simplex decomposition in $\mathbb R^d$ as $\mathcal S'$ although it is the same as $\mathcal S$. Then for the points in both  $\mathcal S$ and $\mathcal S'$, they are linearly related, e.g. the weights on the edges, which is apparent in equation (\ref{barycenter_mapping}) equal to $\mathbf{\frac{A_{ij}}{D_{ii}}}$. 

We will use $N^{\mathcal S'}_{\mathbf z_i}$ to denote the set of neighbors of $\mathbf z_i$ in  $\mathcal S'$. From equation \eqref{barycenter_mapping}, the solution of FPLM is obtained by solving a system of linear equations. 
From equation \eqref{barycenter_mapping} we know that each interior vertex $\mathbf y_i^*$ is a convex combination of its neighbors, well aligned with the convex combination function defined as follows:



\begin{defn}[Convex Combination function] 
For every interior vertex $\mathbf{z}_i$ of a simplex decomposition {$\mathcal S'$} in $\mathbb R^d$ and $\lambda_{ij}\ge0$, for $N^{\mathcal S'}_{\mathbf z_i}$, if a piece-wise linear function $f: D_{\mathcal S'} \rightarrow \mathbb R$ satisfies: 
\begin{equation}
    \sum_{\mathbf z_j \in N^{\mathcal S'}_{\mathbf z_i}} \lambda_{ij} =1 \label{summation of weights}
\end{equation}
and

\begin{equation}
    f(\mathbf z_i) = \sum_{\mathbf z_{j} \in N^{\mathcal S'}_{\mathbf z_i}}\lambda_{ij}f(\mathbf z_{j})
\end{equation}
Then we call $f$ a convex combination function
\end{defn}

Similarly, we will have \textit{piece-wise linear mapping } $\phi : D_{\mathcal S'} \rightarrow \mathbb R^d$ to be any mapping that $\phi = (f_1,\ldots,f_d)$ in which $f_i $'s are the  piece-wise linear function act on each coordinate component of a given vertex $\mathbf z_i$. We call $\phi$ a \textit{convex combination mapping}  given a set of fixed non-negative weights $\lambda_{ij}$ for the neighbours $N^{\mathcal S'}_{\mathbf z_i}$ of each interior vertices $\mathbf z_i \in \mathbf Z$.
We have:
\begin{equation} \label{convex_combination_final}
     \mathbf{y}^*_i=\phi(\mathbf z_i) = \sum_{\mathbf z_{j} \in N^{\mathcal T'}_{\mathbf z_i}}\lambda_{ij}\phi(\mathbf z_{j})
\end{equation}

The convex combination mapping linearly adjusts the coordinates of each interior vertex in $\mathcal S'$ so that for each vertex, the mapping result  $\phi(\mathbf z)$ lies in the convex hull formed by its neighbors. It is clear that $\mathbf{Y^*}$, the optimizer of FPLM, satisfies equation \eqref{convex_combination_final}. For the rest of the paper, we write $\phi_1$ for the convex combination mapping for the first round of FPLM, similarly, $\phi_2$ for the second round of FPLM.

\begin{rem}\label{Function_of_1stround_FPLM}
Together with the chart map $\psi$, we now summarize the whole process of Algorithm \ref{2sFPLM_algorithm}. 
If $\mathcal M$ is a manifold without boundary as required, then the whole process of FPLM (one round) will be: $ \phi_1 \circ \psi(\mathbf{X})$; Otherwise, the two rounds of FPLM is: $ \phi_2 \circ \phi_1 \circ \psi(\mathbf{X})$.

\end{rem}

\subsection{one-to-one mapping induced from FPLM on triangulation} \label{one_to_one_analysis}

\begin{prop}
\label{inside_fp_prop}
FPLM maps all non-fixed points inside the convex hull formed by the fixed points ($\mathbf{P}(\mathbf{C})$).
\end{prop}
\begin{proof}
From previous discussion, it is clear that FPLM is a convex combination mapping, meaning every non-fixed point must be a convex combination of its neighbors. Assuming on contrary, there is one point outside the convex hull of $\mathbf{P}(\mathbf{C})$, then there must be more points outside too due to \eqref{barycenter_mapping1}. For those outside points, find the one on the edge of the convex hull (this point always exists due to finiteness), then it must have more points surrounding it too. Continue this process until all non-fixed points are exhausted. The out-most one will not have a convex hull formed by its neighbors according to supporting hyperplane theorem \cite{boyd2004convex}. This is against the fact that every non-fixed point has to be convex combination of its neighbors. Therefore the assumption is incorrect. 

\end{proof}

We now restrict to 2-manifolds and prove that the mappings induced by both two rounds of FPLM are one-to-one over triangulation, followed by the work from \cite{floater2003one}, we firstly state the Rad\'o-Kneser-Choquet theorem (RKC):

\begin{thm}[Rad\'o-Kneser-Choquet] \label{RKC}
Suppose $\mathcal T$ is a strongly connected triangulation and that $\phi: D_{\mathcal T} \rightarrow \mathbb R^2$ is a convex combination mapping which maps $\partial D_{\mathcal T}$ homeomorphically into the boundary $\partial \Omega$ of some (closed) convex region $\Omega \subset \mathbb R^2$. Then $\phi$ is one-to-one.
\end{thm}

By generalizing the RKC theorem, Floater's \cite{floater2003one} work provided a necessary and sufficient one-to-one condition of $\phi$ for any triangulation: 

\begin{thm}[Floater, 2003]
Suppose $\mathcal{T}$  is any triangulation and let :$\phi: D_{\mathcal T} \rightarrow \mathbb R^2$ is a convex combination mapping which maps $\partial D_{\mathcal T}$ homeomorphically into the boundary $\partial \Omega$ of some (closed) convex region $\Omega \subset \mathbb R^2$. Then $\phi$ is one-to-one if and only if no dividing edge $[v, w]$ of $\mathcal{T}$ is mapped by $\phi$ into $\partial \Omega$ \label{no_dividing_edge_thm}.
\end{thm} 

Followed by the above claims, we now explore the features of $\phi_1$ which is the induced mapping from the first round of FPLM. 

\begin{prop}\label{prop:innerconvex}
If $\mathcal T$ is strongly connected, then $\mathbf{P}(\mathbf{C_2})$ must be a convex polygon formed by the boundary vertices of $\mathcal T$.
\end{prop}
\begin{proof}
We first identify that all points in $\mathbf C_2$ are boundary points of the 2-manifold $\mathcal M$. This is due to Proposition \ref{inside_fp_prop} and the distance minimisation in FPLM. Assume that $\mathbf{P}(\mathbf{C_2})$ is concave.  By \eqref{convex_combination_final} there must be a dividing edge that connects the boundary vertex with the reflex interior angle to another boundary vertex. However, by the definition of strongly connected triangulation, there is no dividing edge between boundary vertices. Hence, $\mathbf{P}(\mathbf{C_2})$ must be convex.  
\end{proof}

\begin{figure}[H]
     \centering
     \subfloat[]{\includegraphics[width = 0.3\textwidth, height = 0.3\textwidth]{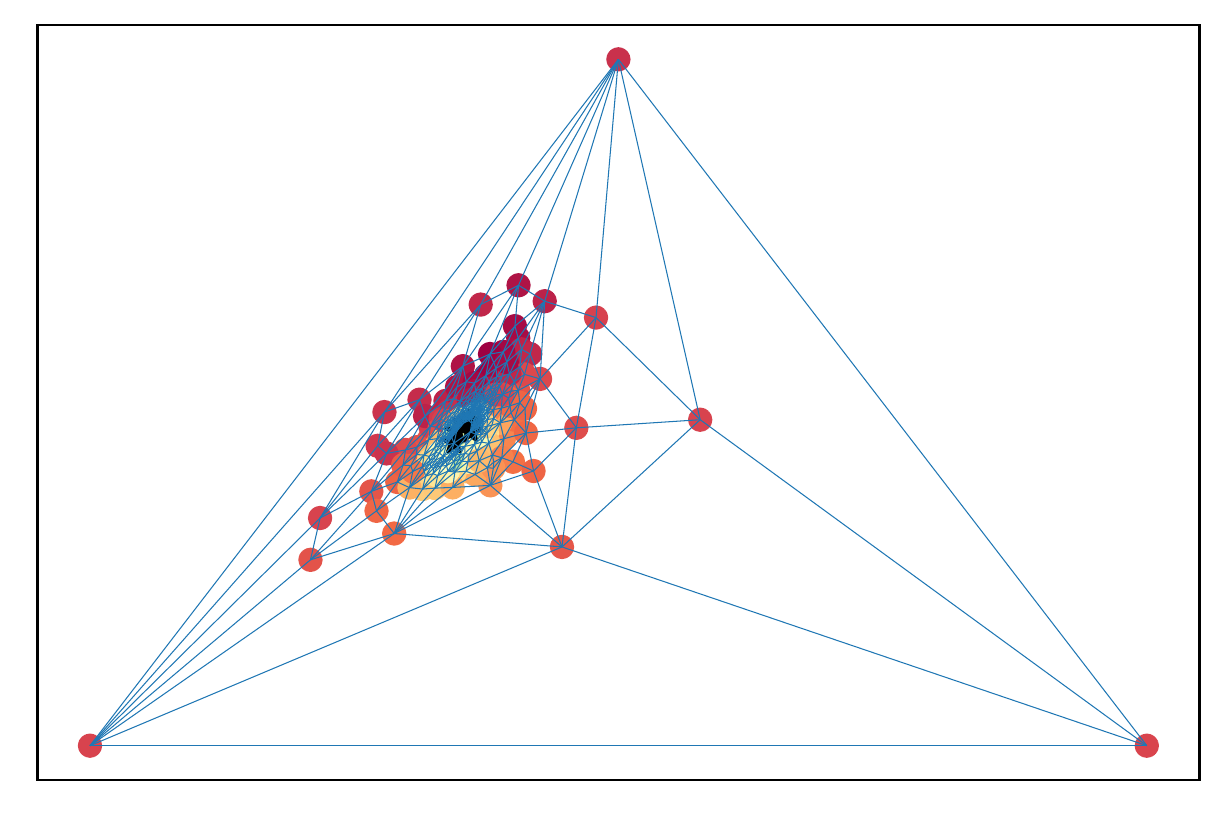}} \;\;\;
     \subfloat[]{\includegraphics[width =0.3\textwidth, height = 0.3\textwidth]{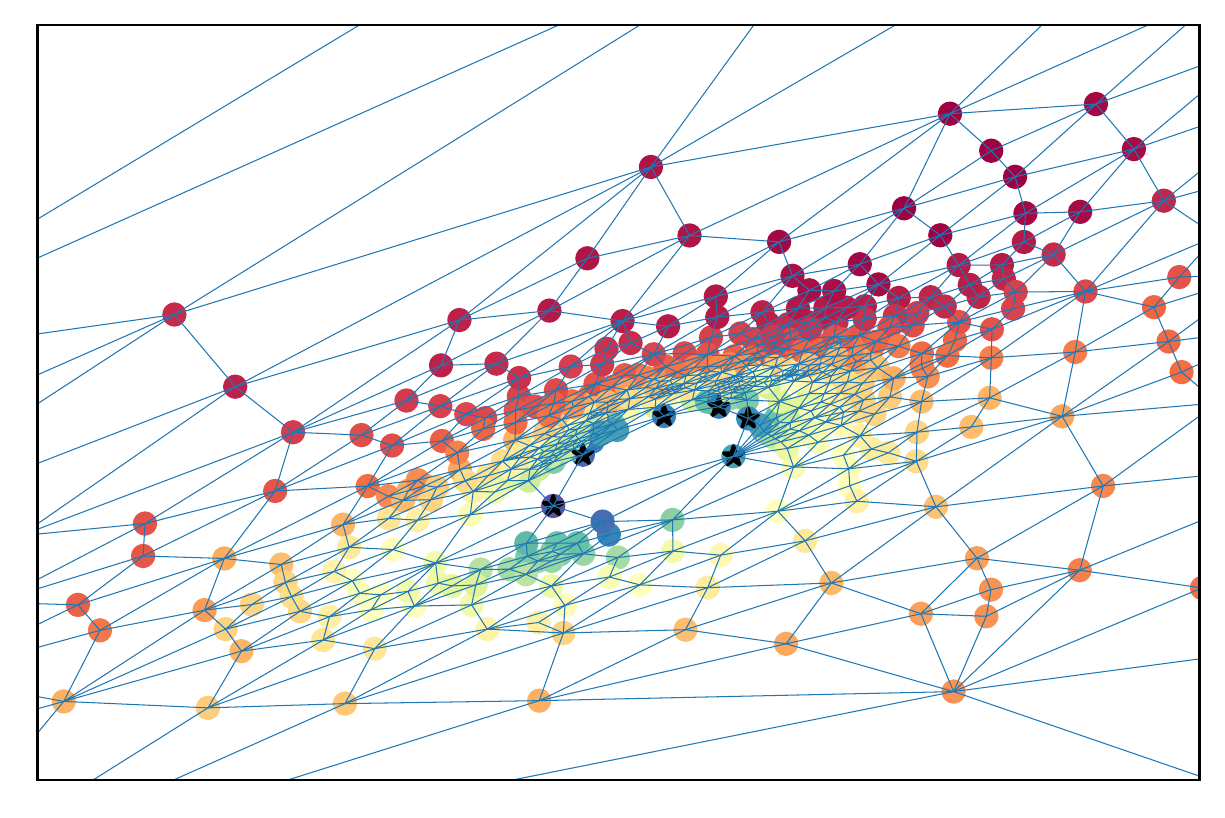}}\;\;\; 
\caption{The convex polygon (Vertices shown as black stars) formed by the boundary points from strongly connected triangulation. } \label{convexity_of_polygon}
\end{figure}

\begin{prop}
If $\mathcal T$ is strongly connected. The mapping $\phi_1$ induced from the  first round of FPLM is one-to-one.  \label{lemma2}
\end{prop}

\begin{proof}
The result is directly from Theorem \ref{no_dividing_edge_thm} as a triangle is a convex hull in $\mathbb R^2$ and without any dividing edge across the triangle.
\end{proof}


We now explore the property of the second round of FPLM. Recall that the second round of FPLM takes the simple polygon formed by joining the boundary vertices inside the first round FPLM result. 

\begin{lem}
Given a strongly connected $\mathcal T'$, the mapping of second round of FPLM bounded by the convex polygon $\mathbf {P(C_2)}$ is one-to-one. \label{lemma3}
\end{lem}
\begin{proof}
Given $\mathbf {P(C_2)}$ is a convex polygon without dividing edge due to Proposition \ref{prop:innerconvex}, the one-to-one result is obvious combinging Theorem \ref{no_dividing_edge_thm}. 
\end{proof}


\begin{rem}[Special case of the first round of FPLM]
When the triangulation is not strongly connected, the first round of FPLM is no longer injective because the dividing edge forms a closed boundary of a subset of the manifold causing the part of the manifold that does not contain the selected triangle to collapse into the dividing edge. Therefore, we have to directly detect the boundary from $\mathcal{G}_{\mathcal{S}}$ and generate a $p$-side convex polygon in $\mathbb R^2$ so that all dividing edges remain inside the boundary and none of the boundary vertices are colinear. Theorem \ref{no_dividing_edge_thm} leads to Theorem \ref{thm:withdividingedge} directly and justifies our choice. 
\end{rem}

From the lemma above, we know that given a strongly connected triangulation $\mathcal T$ on the manifold, $\mathcal T$ can be mapped into a closed convex subset in $\mathbb R^2$ by using either one or two rounds of FPLM. 

\begin{rem}[Manifold with genus]
For surface manifolds, the genus of them is an integer ($g$) representing the maximum number of cuttings along non-intersecting closed simple curves without rendering the resultant manifold disconnected \cite{munkres2018elements}. It will interfere with the boundary detection on the manifold compromising boundary identifiability.  
Hence FPLM is not functional to the surface manifold with non-zero genus such as torus. 
\end{rem}

\begin{rem}[Bijectivity and Homeomorphism]
The sample $\mathbf X =\{\mathbf x_1, \mathbf x_2,..\mathbf x_N\}$ we observed is a subset of a manifold. Since the FPLM maps the triangulation $\mathcal T$ on the 2-manifold to a convex closed area in $\mathbb R^2$, and every $2,1,0$-simplex in $\mathcal T$ is mapped to exactly one specific $2,1,0$-simplex in $\mathbb R^2$, together with the chart map, the mapping induced from FPLM (i.e. $\phi_1 \circ \psi$ for one round, $\phi_2 \circ \phi_1 \circ \psi$ for two rounds) is at least continuous over $\mathcal T$ and one-to-one. Hence, the mapping generated by Algorithm \ref{2sFPLM_algorithm} process ($\phi \circ \psi)$ restricted to a closed area in $\mathbb R^2$ is bijective. 
\end{rem}

Based on the property of FPLM on 2-manifolds, we now explore the feature of FPLM when the manifold structure is obtained by edge-to-edge tessellation of polygons (triangle is a three-sided polygon). 


\begin{defn}[Edge-to-Edge Tessellation of polygons on 2-manifolds]
Given 2-manifold, if the manifold can be decomposed with a list of polygons with the number of side  ($n$) larger than equal to 3, and the intersection between each polygon can only either be empty, a common point, or a common edge, we then say this manifold is tessellated by these polygons, and that is an edge-to-edge tessellation on the manifold. 
\end{defn}

Let $\mathcal{TL}$ be the tessellation described above. If we further triangulate $\mathcal{TL}$, for example we add edges which partition each face of $\mathcal{TL}$ into triangles, then we can use convex combination mapping $\phi': D_{\mathcal{TL}} \rightarrow \mathbb R^2$ and the $\phi'$ is linear over each triangle in $D_{\mathcal{TL}}$ and continuous. Clearly, if $\mathcal{TL}$ is strongly connected, based what we discussed earlier, Algorithm \ref{2sFPLM_algorithm} is one-to-one with the requirement that the selected polygon in the first round is convex. If $\mathcal{TL}$ is not strongly connected, again, boundary detection is necessary to form a $p$ side polygon manually.  

We now focus on the property of FPLM in higher dimenional simplex decomposition. Taking 3-simplex decomposition (tetrahedralization) as an example, it has been reported that the convex combination mapping may not be one-to-one over tetrahedral meshes, and a counter-example has been reported in \cite{floater2006convex}. However, the counter-example mentioned on that paper has four points positioned in one face of a tetrahedron, this conflicts our assumption that all points should in general position. Also as FPLM starts from sum of squared distances, which corresponding to a special type of convex mapping, different from the one in \cite{floater2006convex}. Hence, FPLM still works for this counter example. 

In regards to $d$-manifold for $d>2$, the situation is more complicated. The bijectivity of piece-wise linear mapping relates to orientation preserving and some boundary conditions. We restate the key theorem here, which is in \cite{lipman2014bijective}.

\begin{thm}[Sufficient conditions for bijectivity]
Given a $d$ dimensional connected orientable manifold $\mathcal M$ and its $d$-simplex decomposition constructed on a discrete sample, then a piece-wise linear mapping $\phi$ from $\mathcal M$ to $\mathbb R^d$ is bijective if it satisfies the following conditions : 
\begin{enumerate}
    \item The mapping $\phi$ is orientation preserving over entire decomposition.
    \item The boundary of simplex decomposition is mapped to a polytope in $\mathbb R^d$ bijectively.
\end{enumerate}
\end{thm}
The first round FPLM with a selected simplex maps the boundary of the manifold in the centre of the simplex as a convex polytope as shown in \ref{prop:innerconvex}. This step guarantees the second condition mentioned above. However, orientation preserving property is not yet clarified, although we conjure that it may be there. The experiments of 3-manifolds support this conjecture. Therefore rigorous proof is still wanted.

\newpage
\section{Appendix two: More experimental results}
In this section, we add more experimental results and briefly introduce the d-simplex decomposition methods such as Tangential Complex and Tetgen. 

\textbf{Tangential Complex}\\
Followed by \cite{boissonnat2014manifold}, we use the Tangential Complex (TC) algorithm to construct triangulation on the manifold. One requirement for conducting TC is to require that each point's tangent space on the manifold be estimated by using PCA. The tangential complex is obtained by gluing the local (Delaunay) triangulations around each sample point. The output of TC is a sub-complex of the $l$-dimensional Delaunay simplices of the sample points, but it can be computed using mostly operations in the $d$-dimensional tangent spaces.\cite{boissonnat2014manifold}. It can be proved that the output of the reconstructed manifold from the TC algorithm can be isotopic to the original manifold. However, due to the appearance of so-called \textit{inconsistencies}, TC does not always generate the triangulation result that we defined in \ref{Triangulation and connectivity}. Even though this situation has been reported \cite{freedman2002efficient}, there is no universal solution except for the case of curves ($d=1$) \cite{flototto2003coordinate}. Hence, one way to deal with this problem is to give each point that contained inconsistent simplex a small perturbation of their weights so that the position of\textit{ medial axis} of the points can be adjusted accordingly. Unfortunately, there is no guarantee that this perturbation method can always reduce the number of inconsistencies to zero. Hence, if the TC result has inconsistency even after perturbation, we will use Delaunay or surface triangulation.

\textbf{Tetgen}\\
One of the most widely applied tetrahedral mesh generation methods: Tetgen, is comprehensively described in \cite{si2015tetgen}. It is a mixture of a few classic constrain methods described in \cite{george1991automatic} and the classic Delaunay refinement algorithm \cite{ruppert1995delaunay}. Given a set of points from an underlying manifold in $\mathbb R^l$, with an intrinsic dimension equal to three,Tetgen can generate a 3D piece-wise linear complex, collectively named as cells. The property of such cells includes 1.  the boundary of each cell in the complex is a union of cells in the complex; 2. The intersection (if it exists) of two cells is the simplicial complex with a lower dimension, at least less than one compared to the two intersected cells. If all the cells in these underlying 3-manifolds are tetrahedral, we would call the piece-wise linear complex formed by tetrahedral mesh. More generally, the piece-wise linear meshes generated from Tetgen offer a facet-to-facet tessellation of manifold in $\mathbb R^l$.

\subsection{Additional results on 2-manifolds}

We add some additional experimental results for 2-manifolds here: 

\begin{figure}[H]
     \centering
     \subfloat[]{\includegraphics[width = 0.13\textwidth, height = 0.13\textwidth]{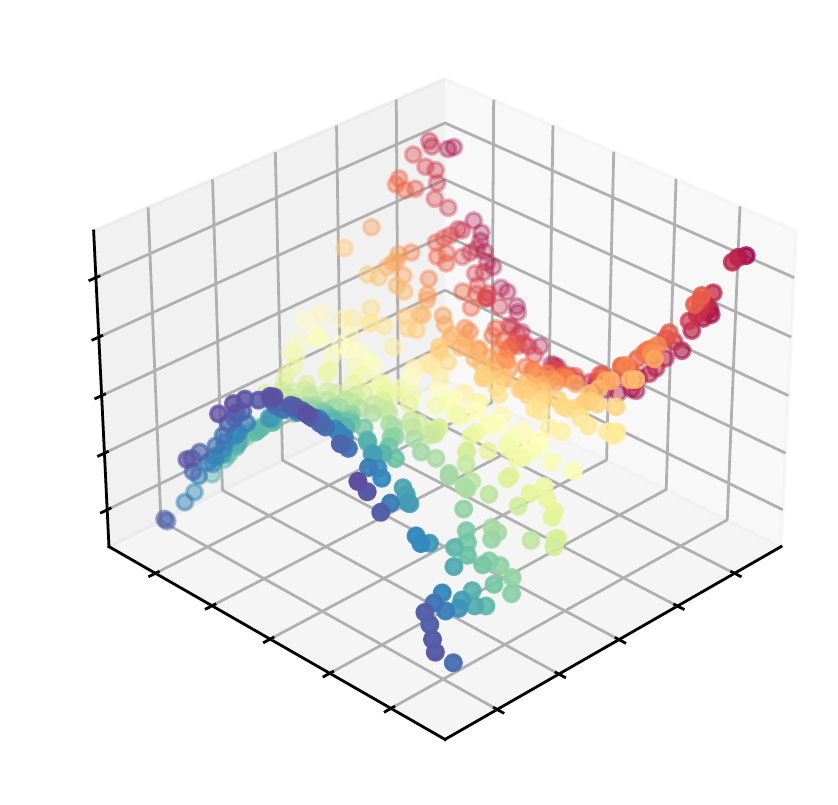}} \;\;\;
     \subfloat[]{\includegraphics[width =0.13\textwidth, height = 0.13\textwidth]{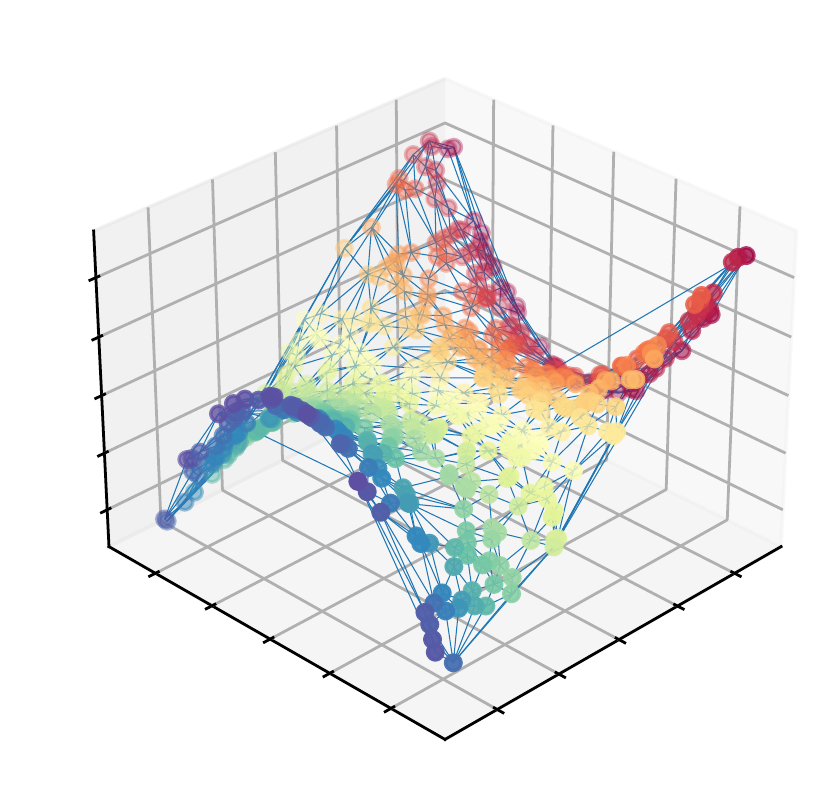}}\;\;\;
     \subfloat[]{\includegraphics[width =0.13\textwidth, height = 0.13\textwidth]{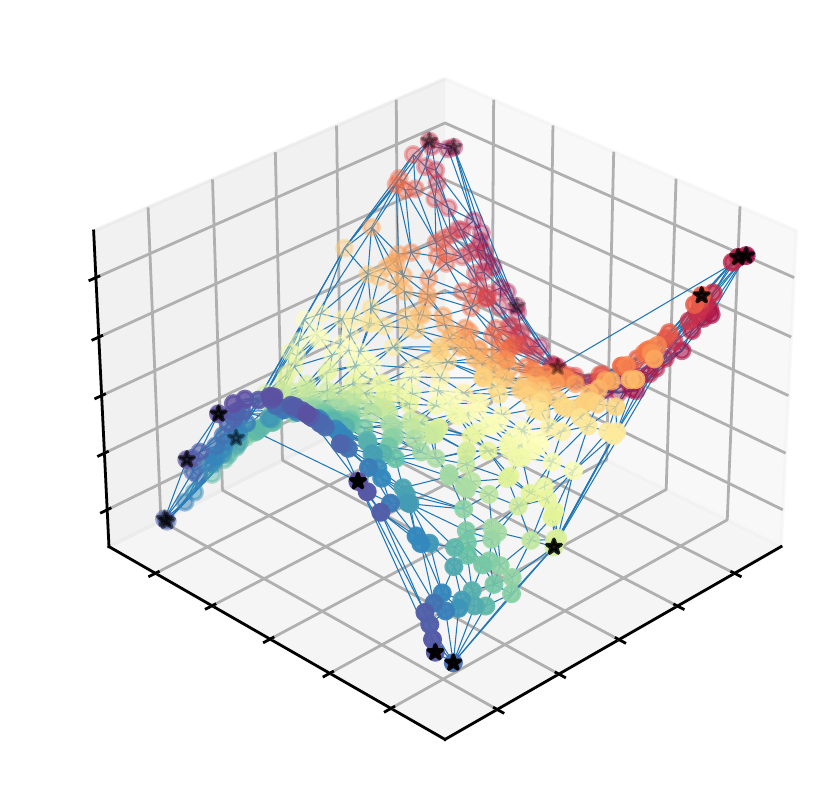}}\;\;\;
     \subfloat[]{\includegraphics[width =0.13\textwidth, height = 0.13\textwidth]{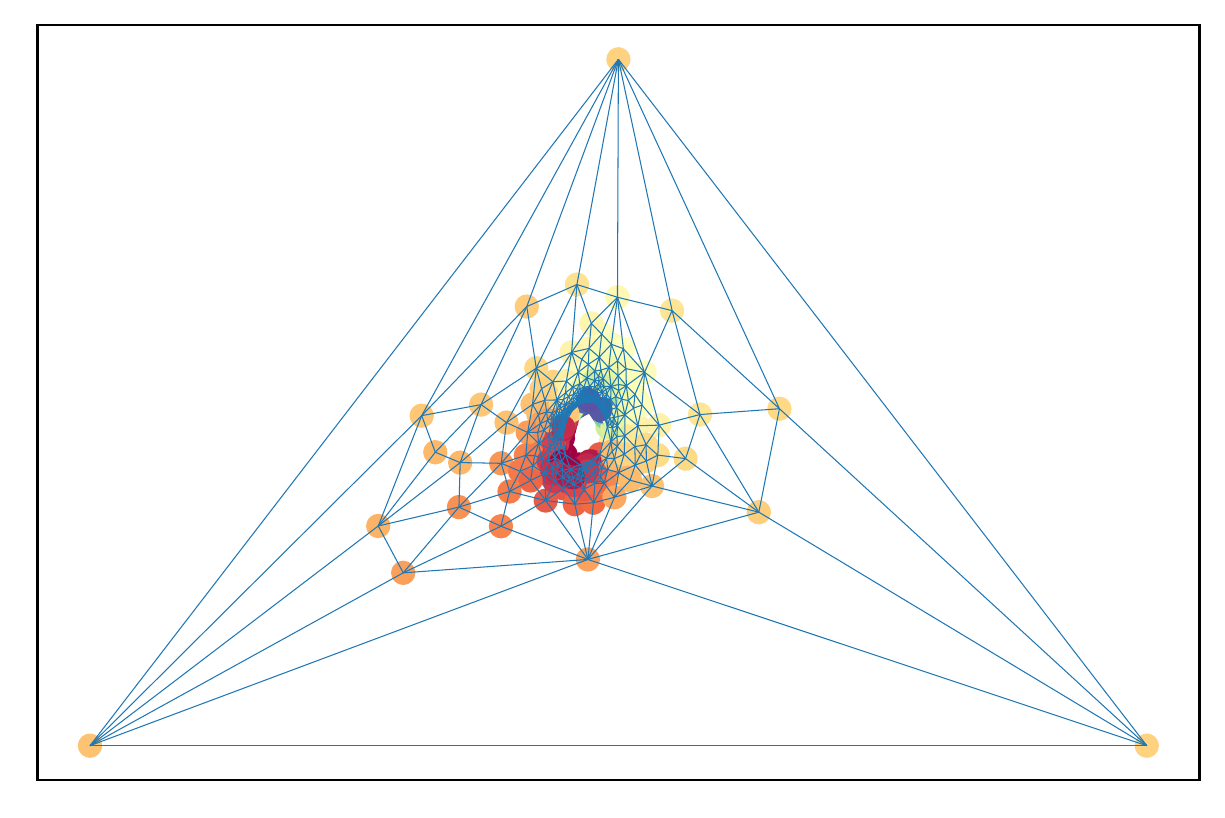}}\;\;\;
     \subfloat[]{\includegraphics[width =0.13\textwidth, height = 0.13\textwidth]{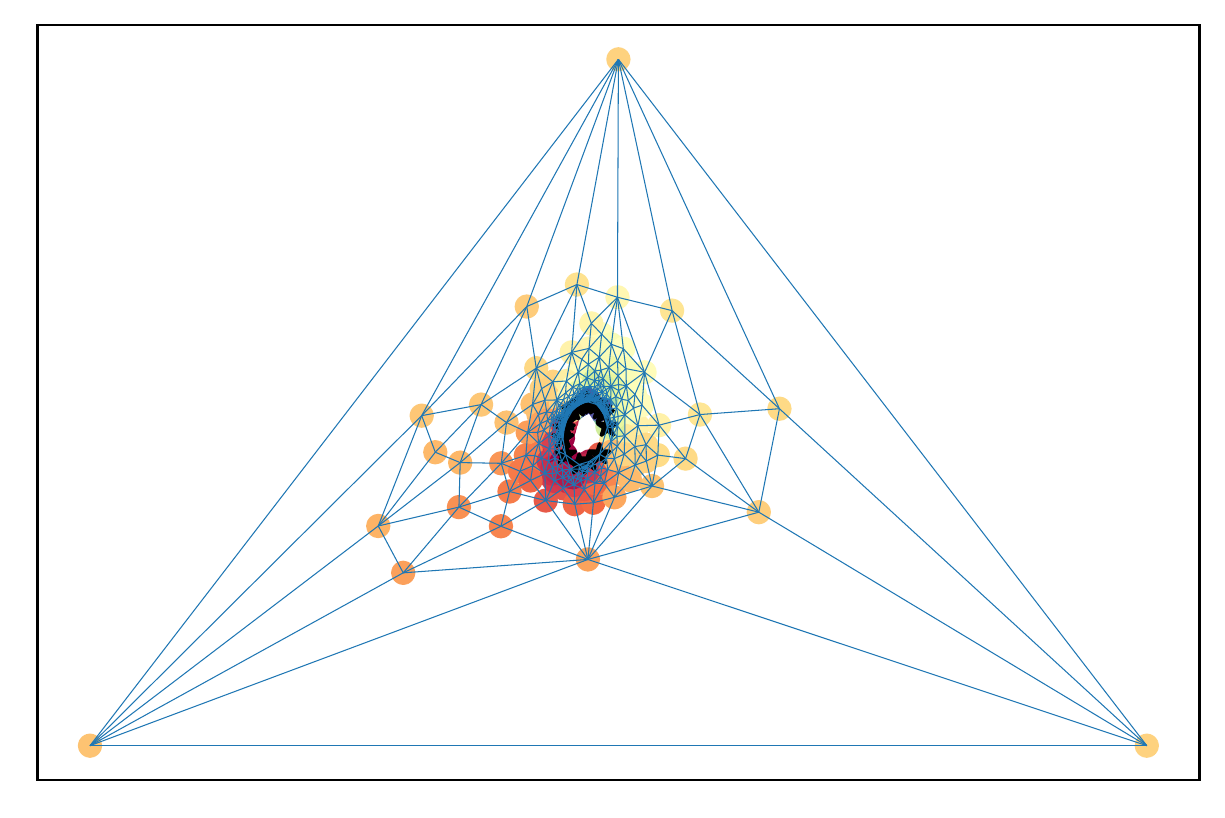}}\;\;\;
     \subfloat[]{\includegraphics[width =0.13\textwidth, height = 0.13\textwidth]{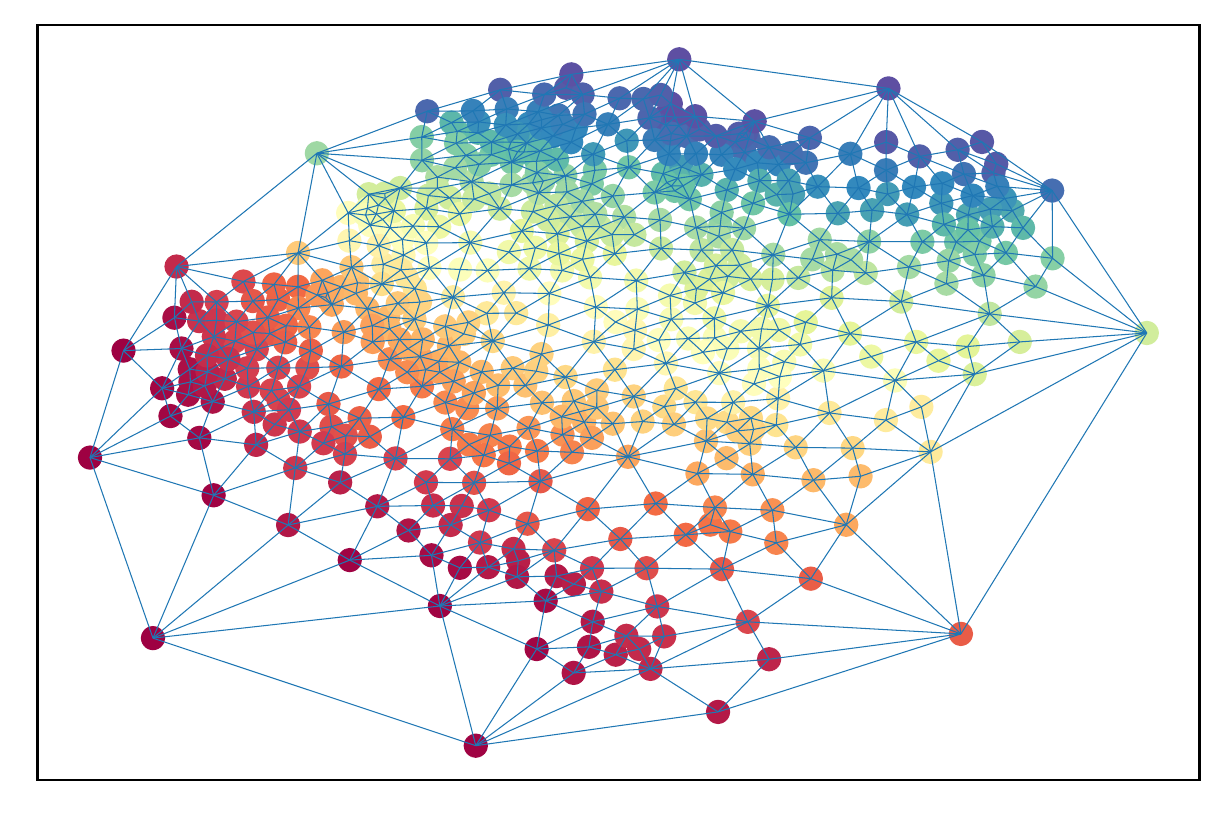}}\;\;\;
\caption{FPLM on Monkey Saddle, (a) Manifold scatters,(b) Triangulation on manifold, (c) Boundary detection from triangulation result, (d) First round FPLM result, (e) Boundary detection of the first round FPLM, (f) Final result. }  
\label{FPLM_monkey_saddle}
\end{figure}

\begin{figure}[H]
     \centering
     \subfloat[AE]{\includegraphics[width = 0.11\textwidth, height = 0.13\textwidth]{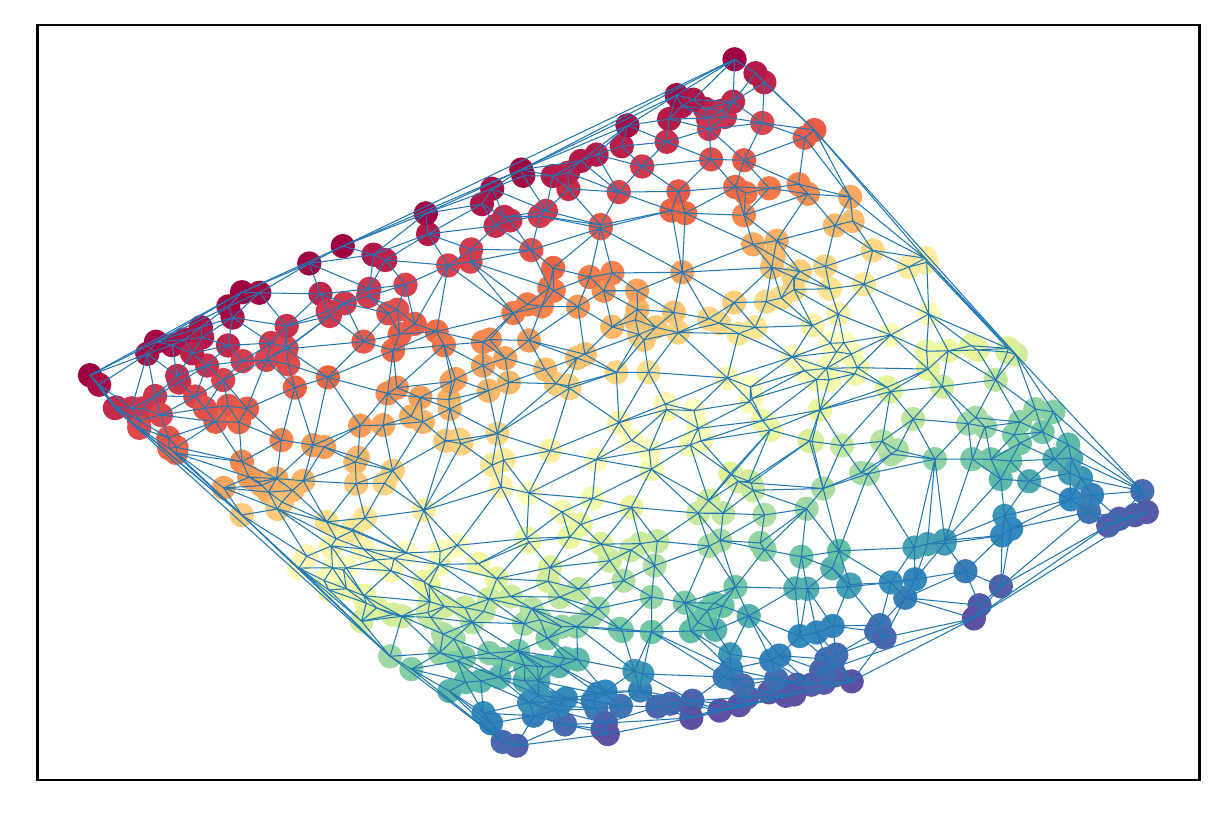}} \;\;\;
     \subfloat[Isomap]{\includegraphics[width =0.11\textwidth, height = 0.13\textwidth]{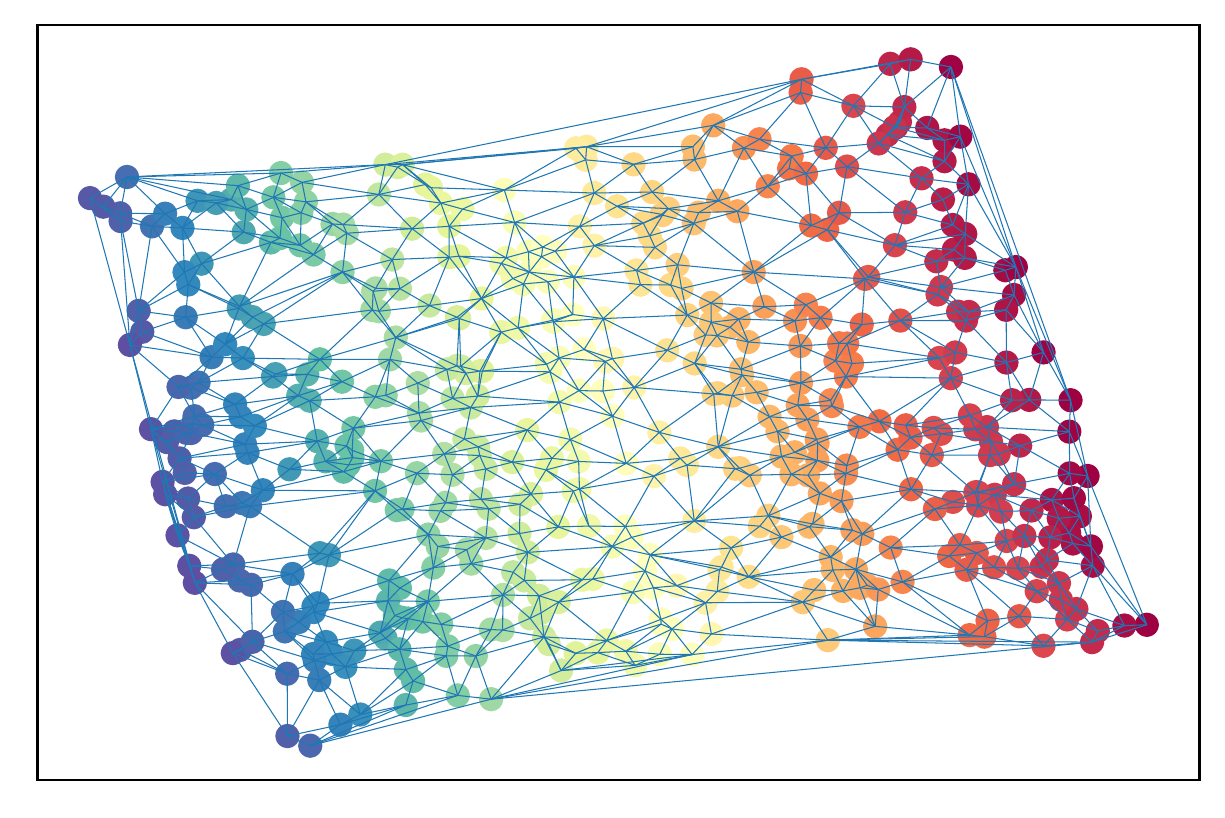}}\;\;\;
     \subfloat[LE]{\includegraphics[width =0.11\textwidth, height = 0.13\textwidth]{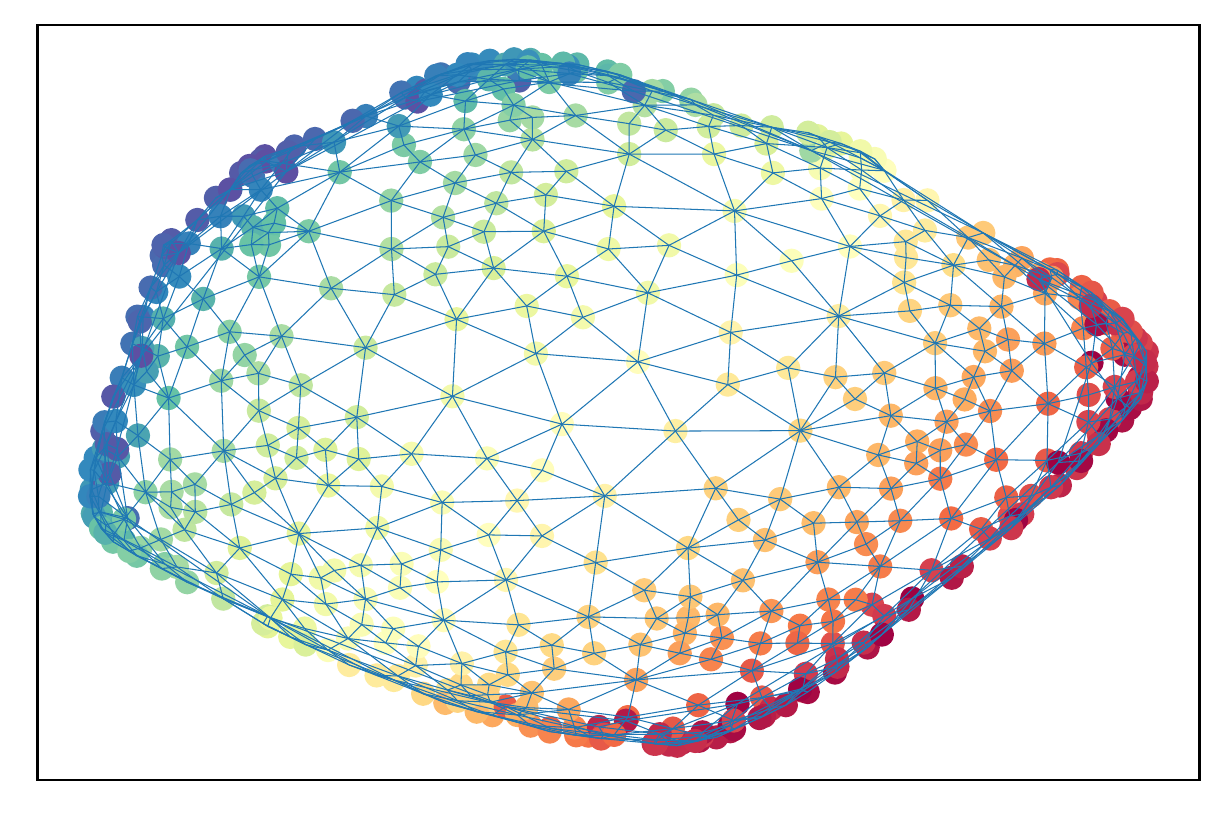}}\;\;\;
     \subfloat[LLE]{\includegraphics[width =0.11\textwidth, height = 0.13\textwidth]{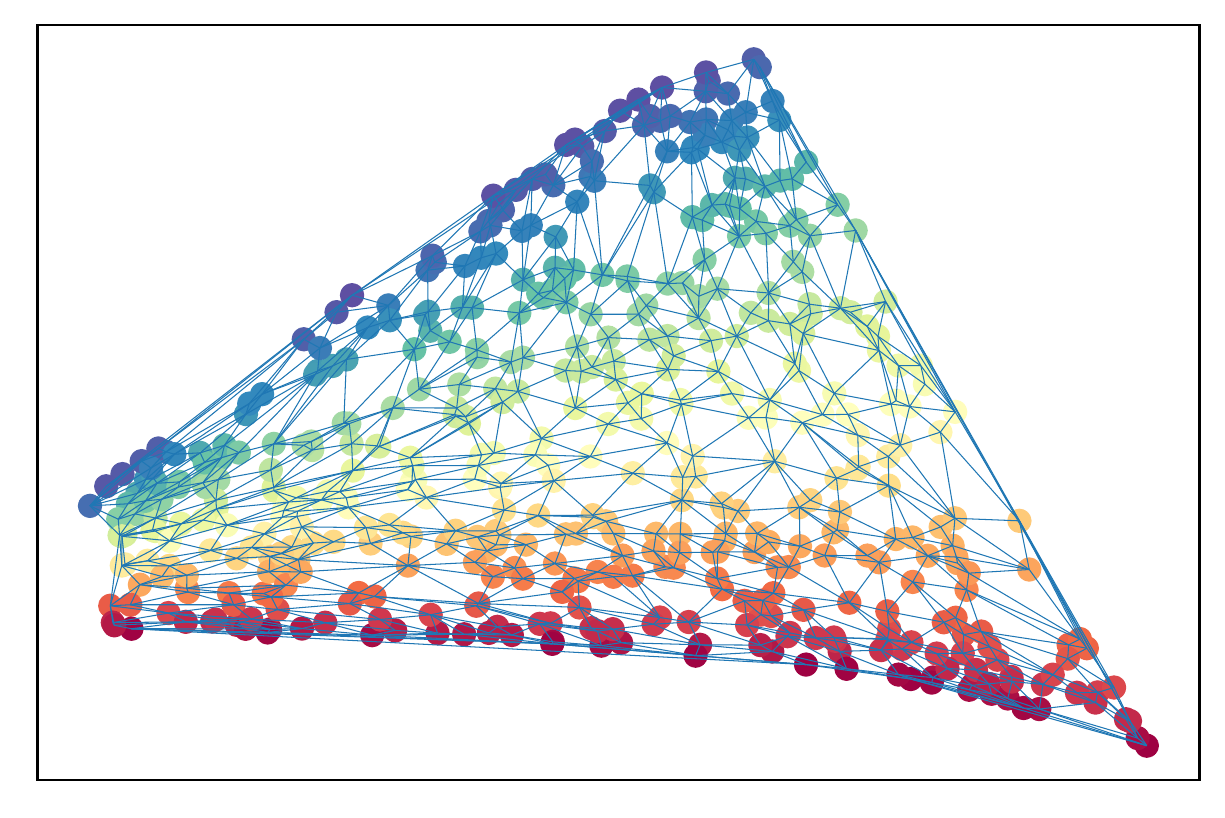}}\;\;\;
     \subfloat[LSTA]{\includegraphics[width =0.11\textwidth, height = 0.13\textwidth]{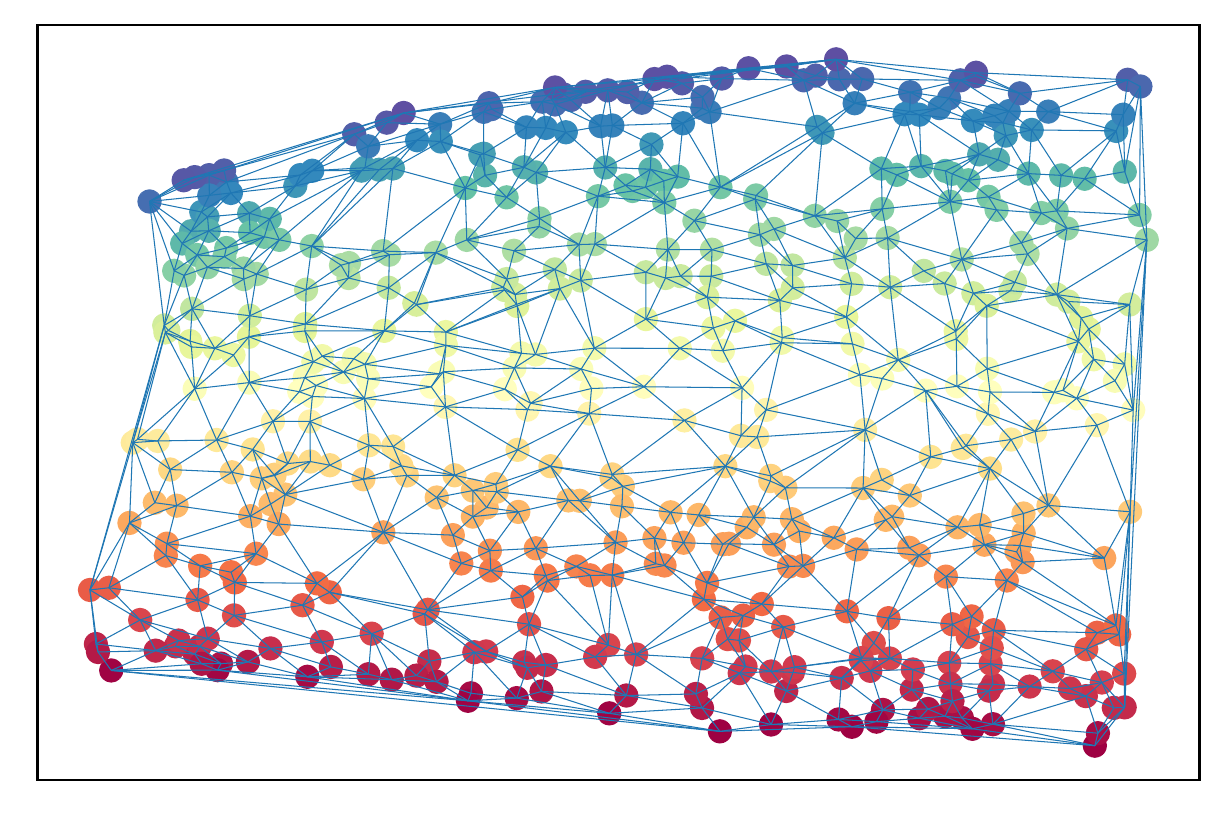}}\;\;\;
     \subfloat[MDS]{\includegraphics[width =0.11\textwidth, height = 0.13\textwidth]{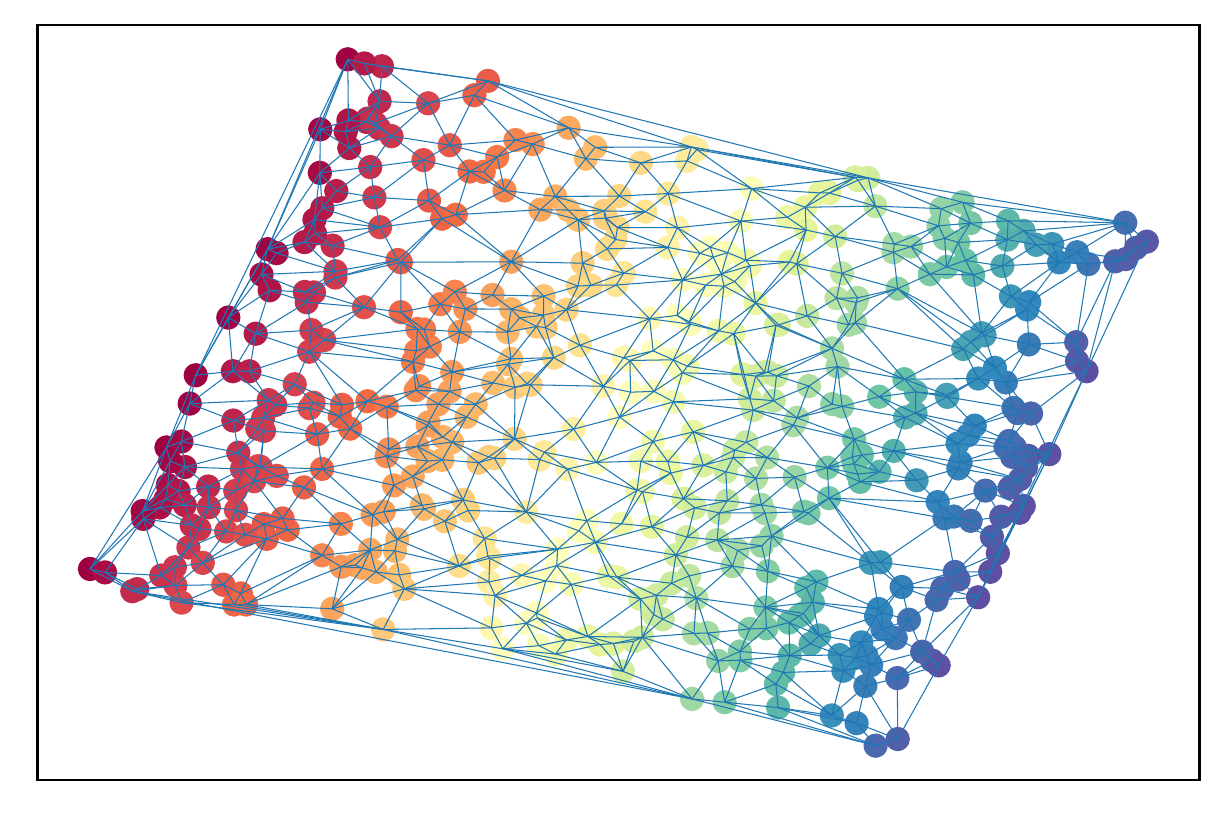}}\;\;\;
     \subfloat[t-SNE]{\includegraphics[width =0.11\textwidth, height = 0.13\textwidth]{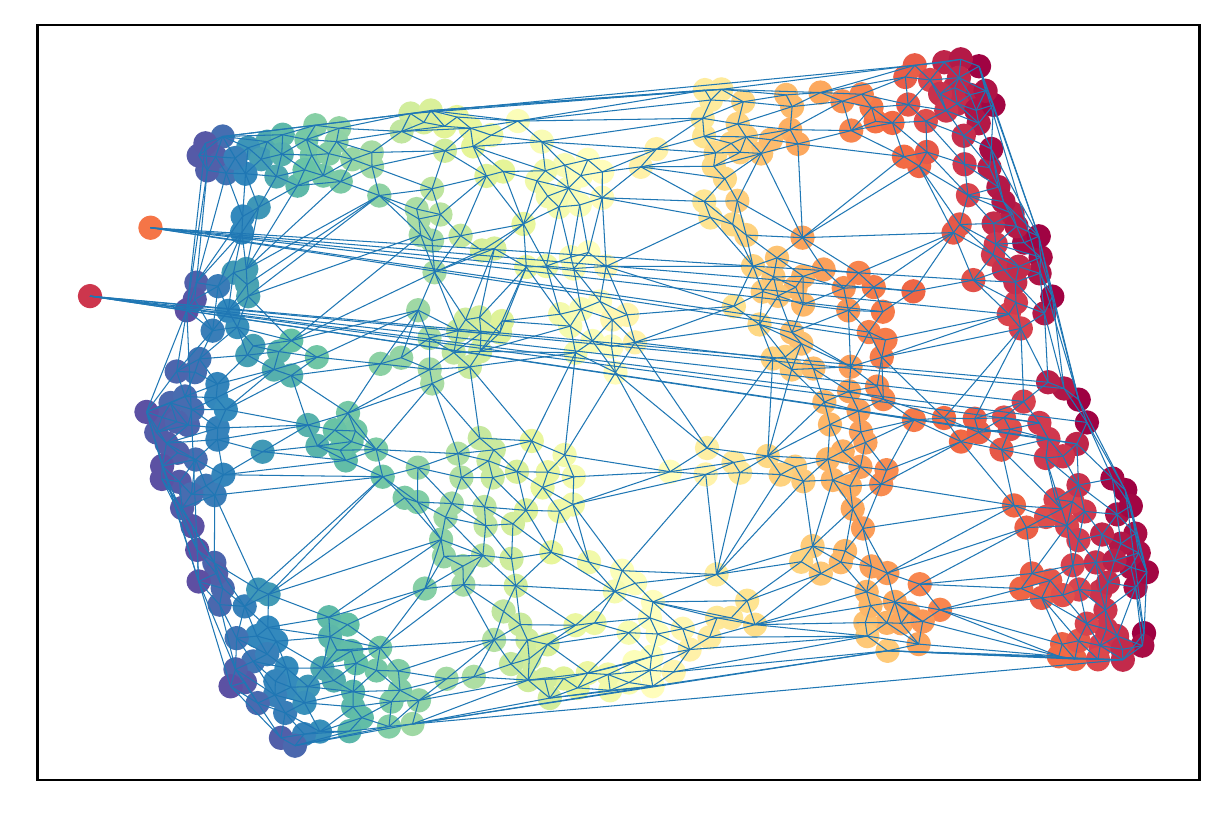}}\;\;\;
\caption{Other methods on monkey saddle. (a) 21 crosses,(b) 46 crosses, (c) 815 crosses, (d) 57 cross, (e) 33 crosses, (f) 41 crosses, (g) 735 crosses}   \label{other_methods_monkey_saddle}
\end{figure}

Manifold: Paraboloid
\begin{figure}[H]
     \centering
     \subfloat[]{\includegraphics[width = 0.13\textwidth, height = 0.13\textwidth]{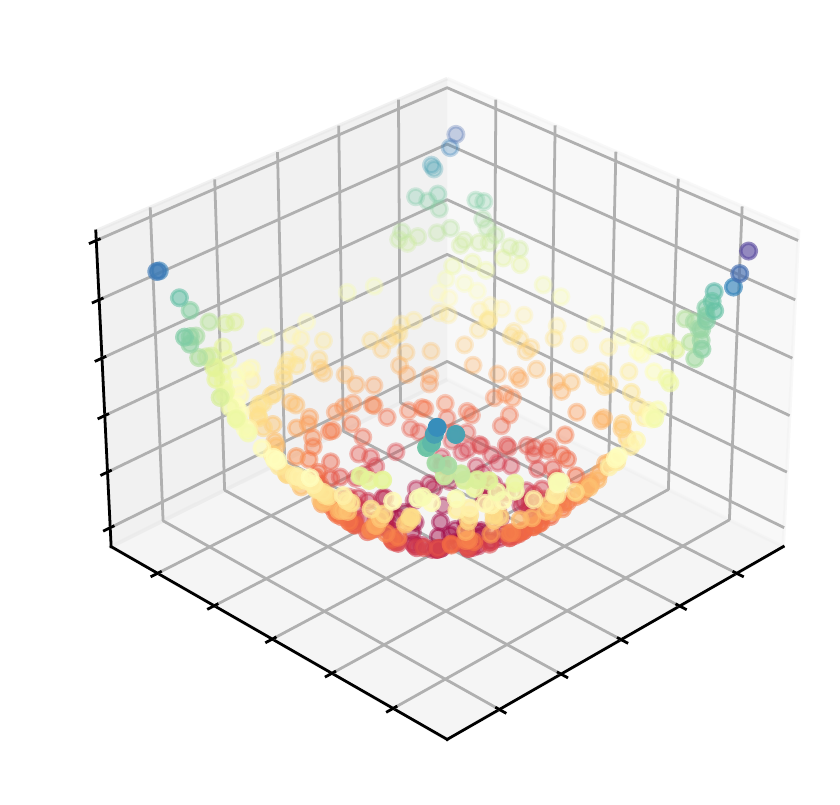}} \;\;\;
     \subfloat[]{\includegraphics[width =0.13\textwidth, height = 0.13\textwidth]{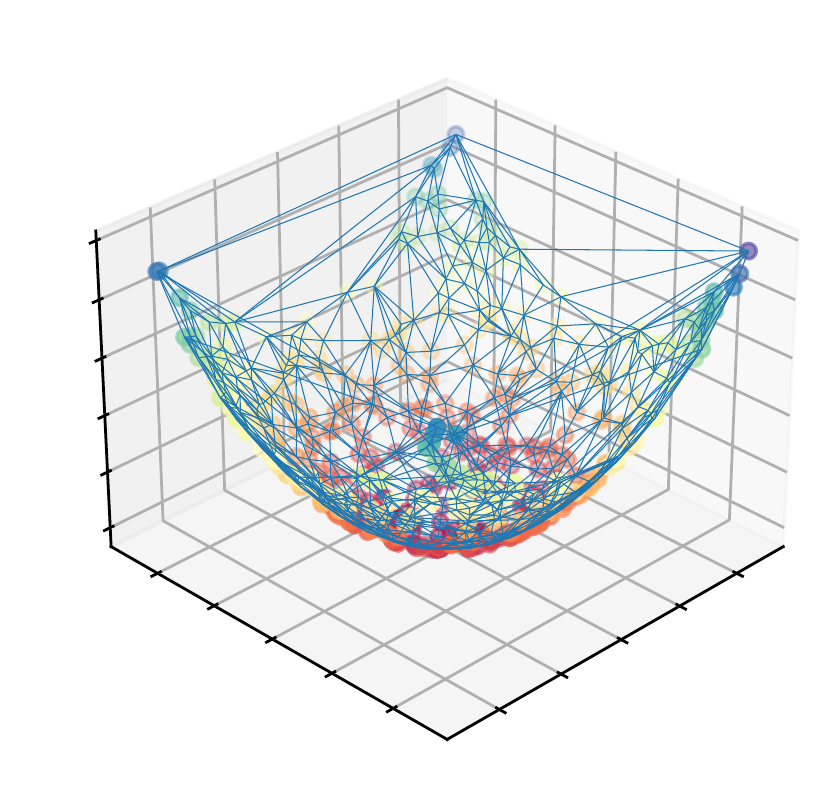}}\;\;\;
     \subfloat[]{\includegraphics[width =0.13\textwidth, height = 0.13\textwidth]{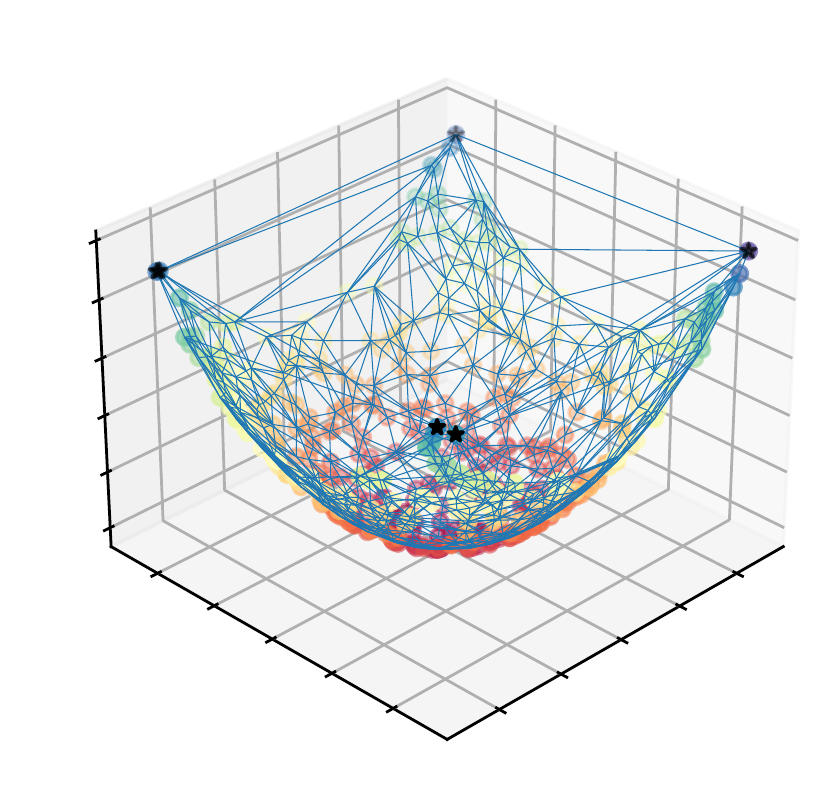}}\;\;\;
     \subfloat[]{\includegraphics[width =0.13\textwidth, height = 0.13\textwidth]{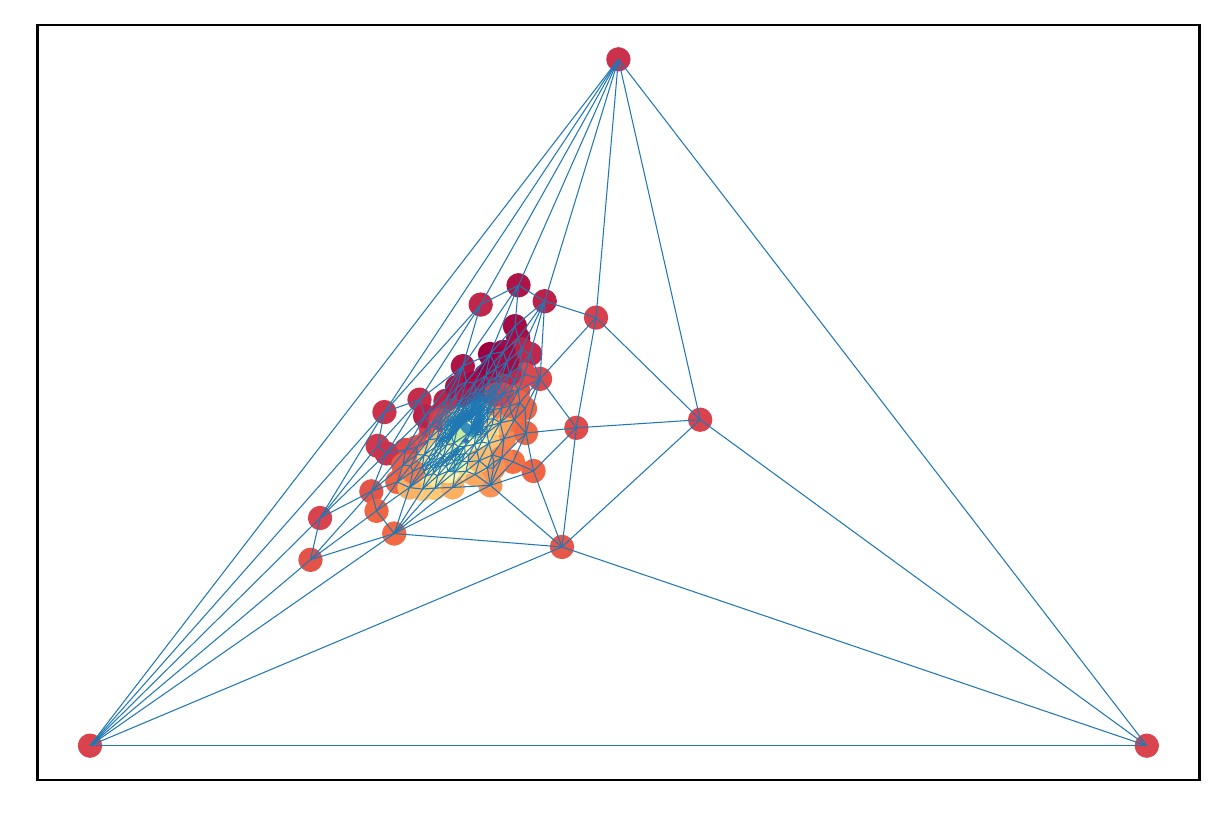}}\;\;\;
     \subfloat[]{\includegraphics[width =0.13\textwidth, height = 0.13\textwidth]{parabola_results/parabola_FLE_1step_bd.pdf}}\;\;\;
     \subfloat[]{\includegraphics[width =0.13\textwidth, height = 0.13\textwidth]{parabola_results/parabola_FLE_2step.pdf}}\;\;\;
\caption{FPLM on Paraboloid, (a) Manifold scatters,(b) Triangulation on manifold, (c) Boundary detection from triangulation result, (d) First round FPLM result, (e) Boundary detection of the first round FPLM, (f) Final result.}   
\label{FPLM_Parabola}
\end{figure}

\begin{figure}[H]
     \centering
     \subfloat[AE ]{\includegraphics[width = 0.11\textwidth, height = 0.13\textwidth]{ 
     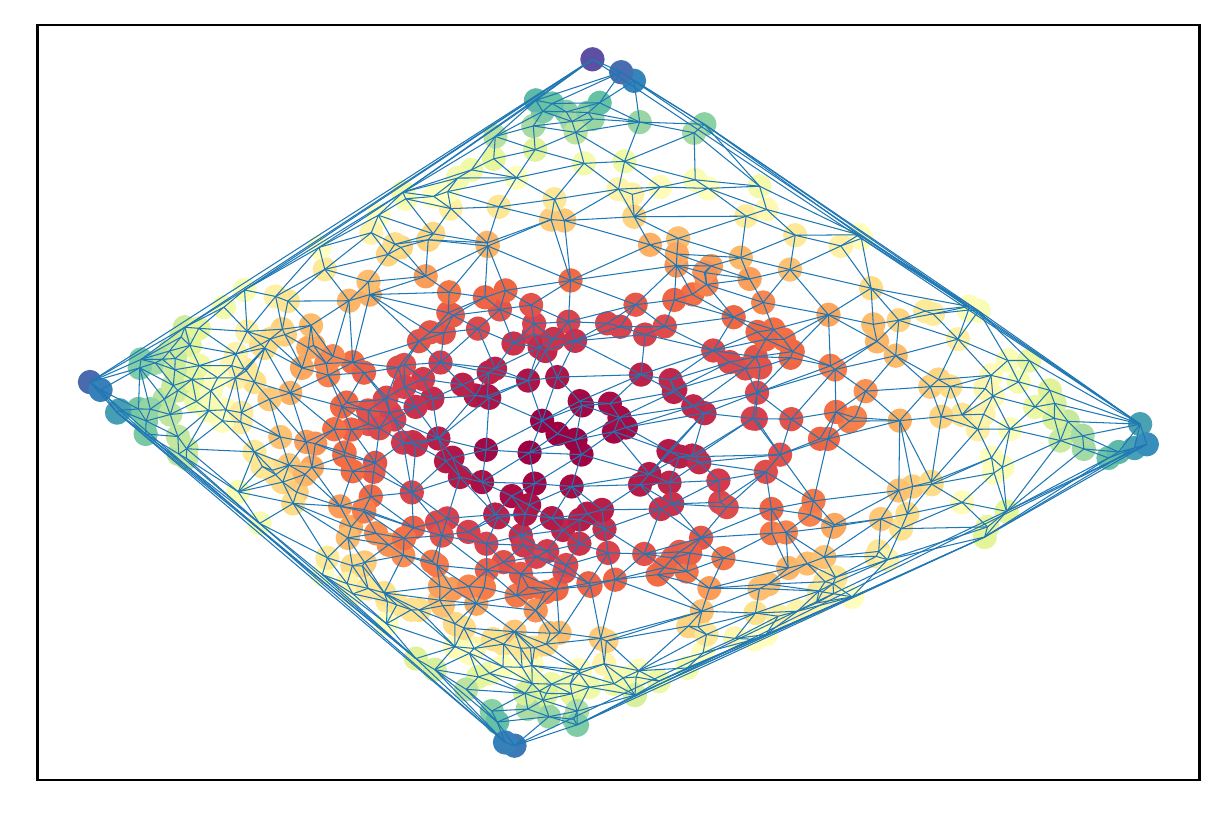}} \;\;\;
     \subfloat[Isomap ]{\includegraphics[width =0.11\textwidth, height = 0.13\textwidth]{ 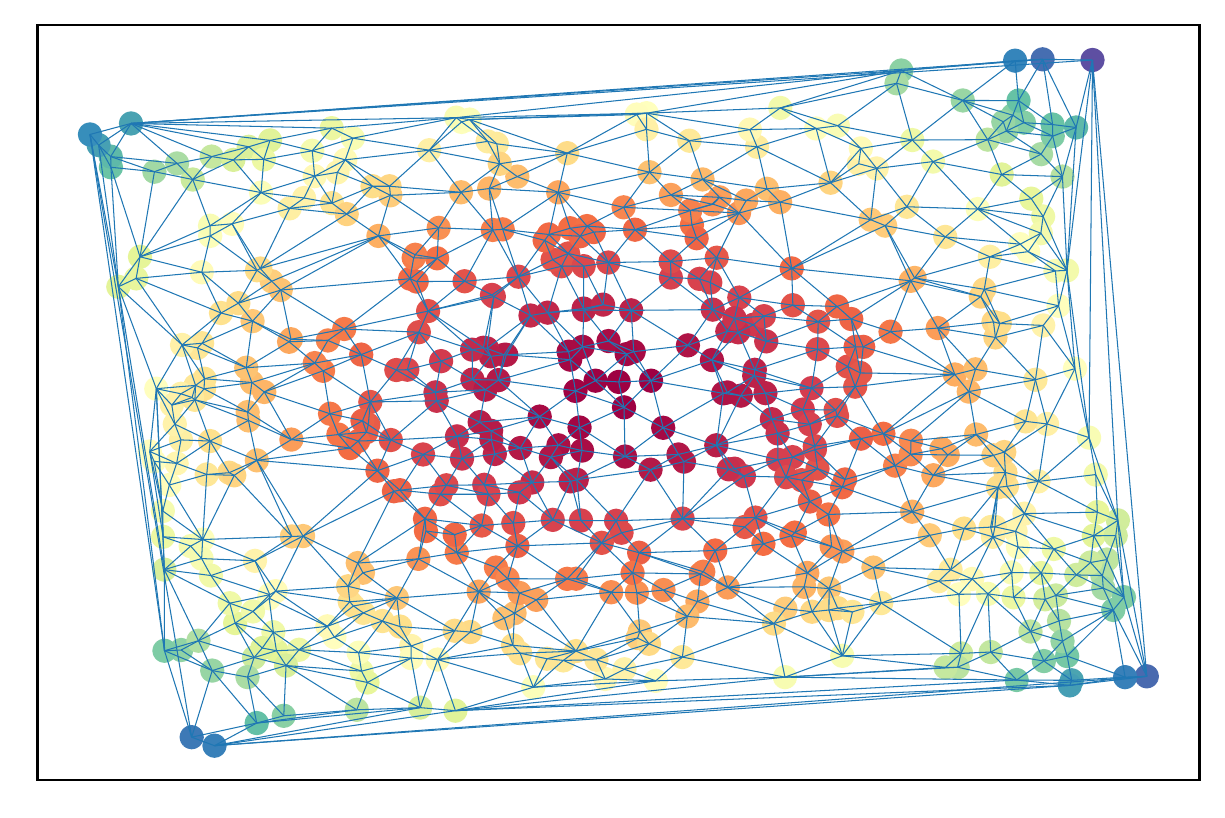}}\;\;\;
     \subfloat[LE ]{\includegraphics[width =0.11\textwidth, height = 0.13\textwidth]{ 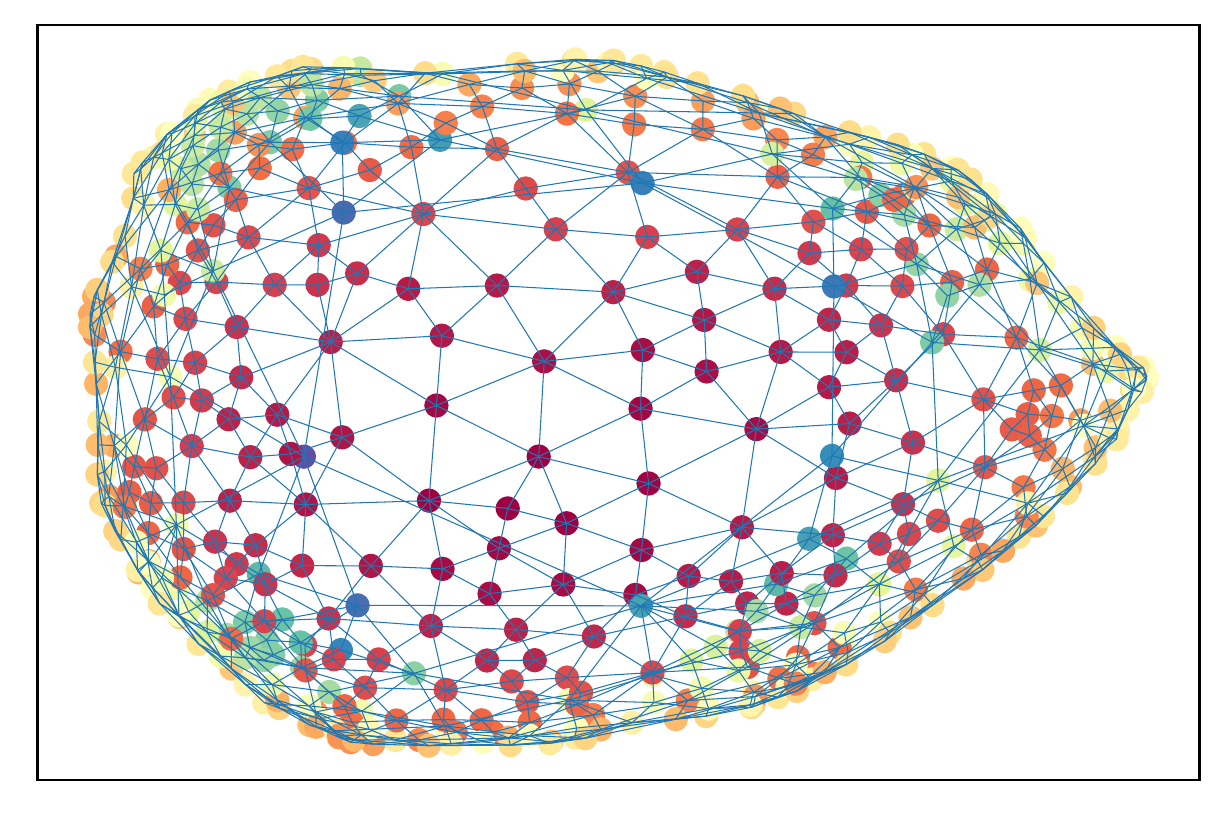}}\;\;\;
     \subfloat[LLE ]{\includegraphics[width =0.11\textwidth, height = 0.13\textwidth]{ 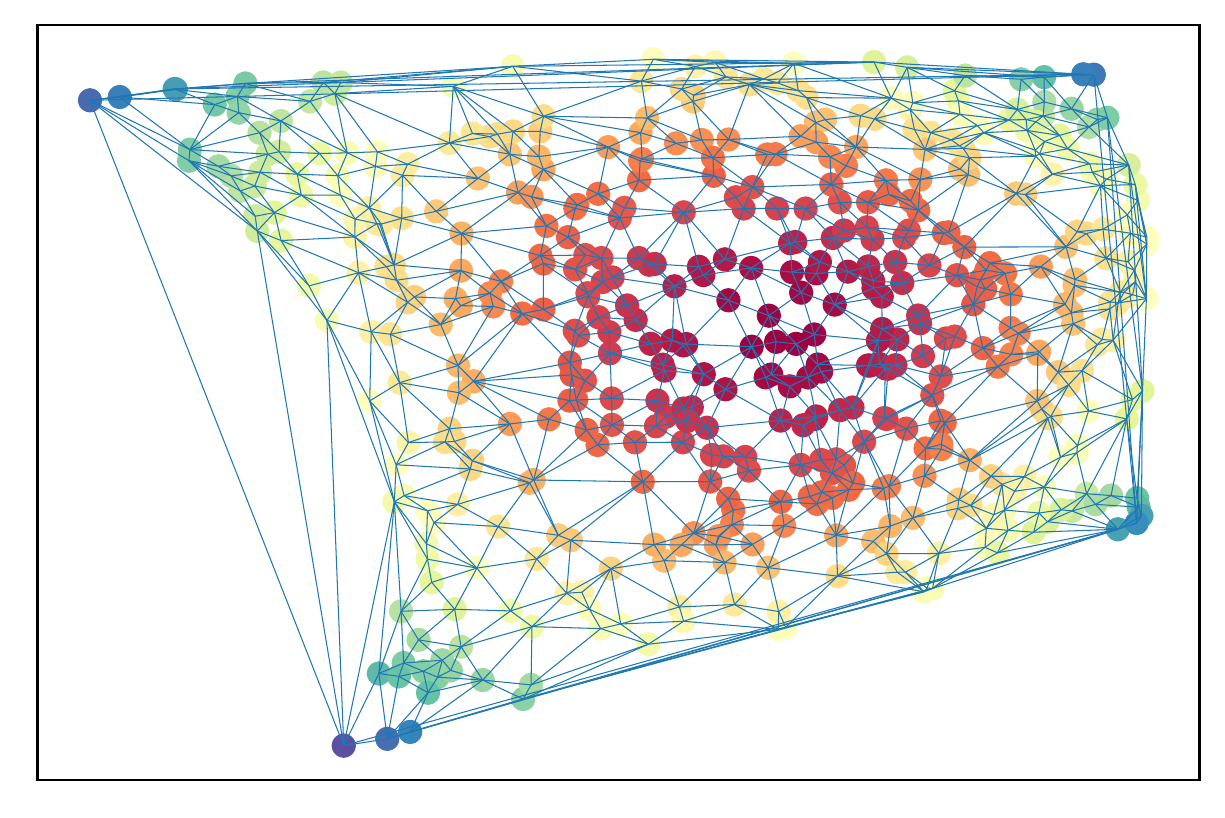}}\;\;\;
     \subfloat[LTSA ]{\includegraphics[width =0.11\textwidth, height = 0.13\textwidth]{ 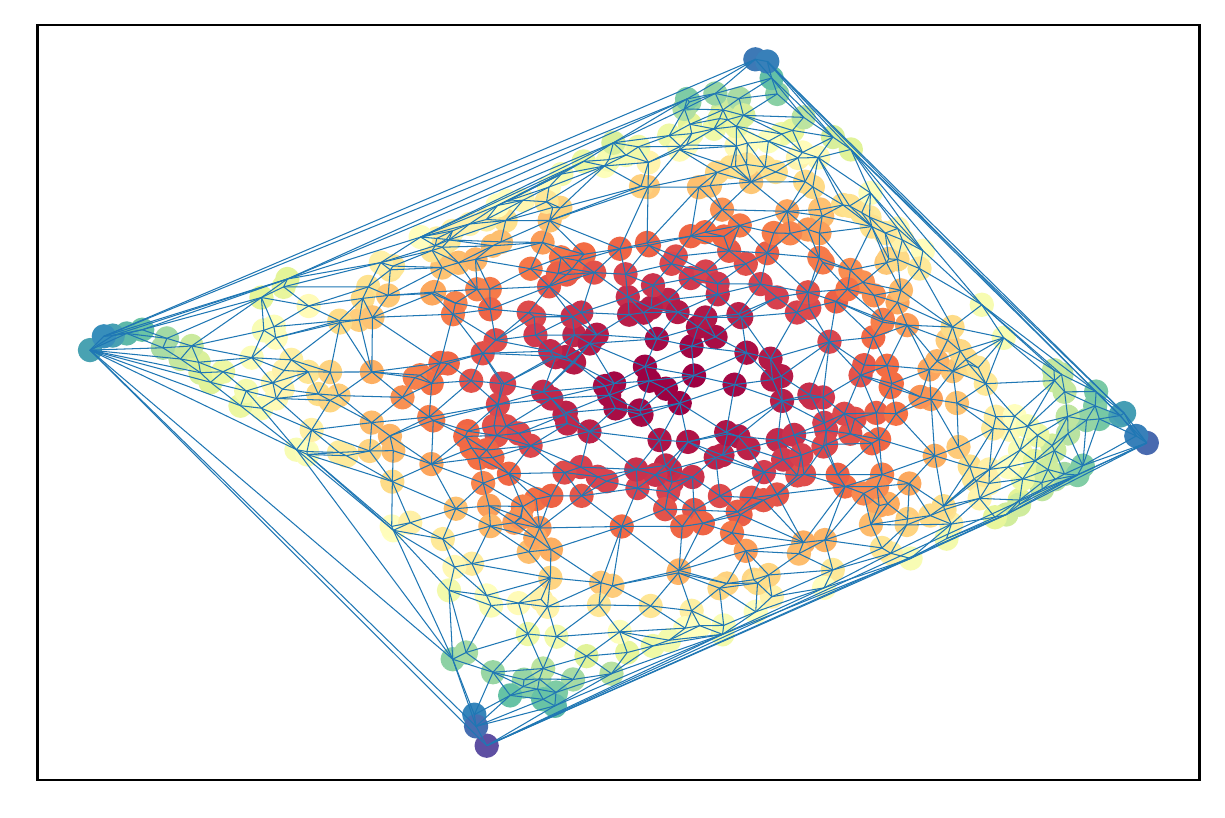}}\;\;\;
     \subfloat[MDS ]{\includegraphics[width =0.11\textwidth, height = 0.13\textwidth]{ 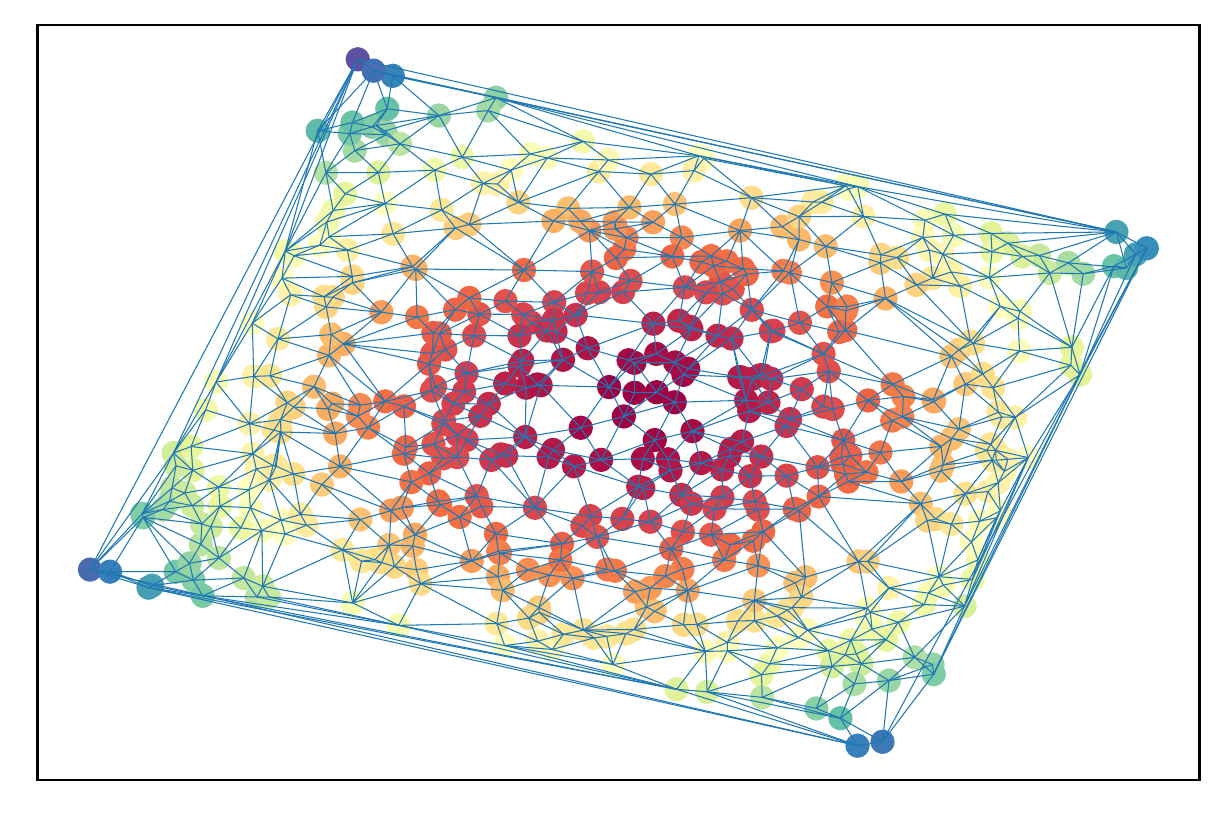}}\;\;\;
     \subfloat[t-SNE ]{\includegraphics[width =0.11\textwidth, height = 0.13\textwidth]{ 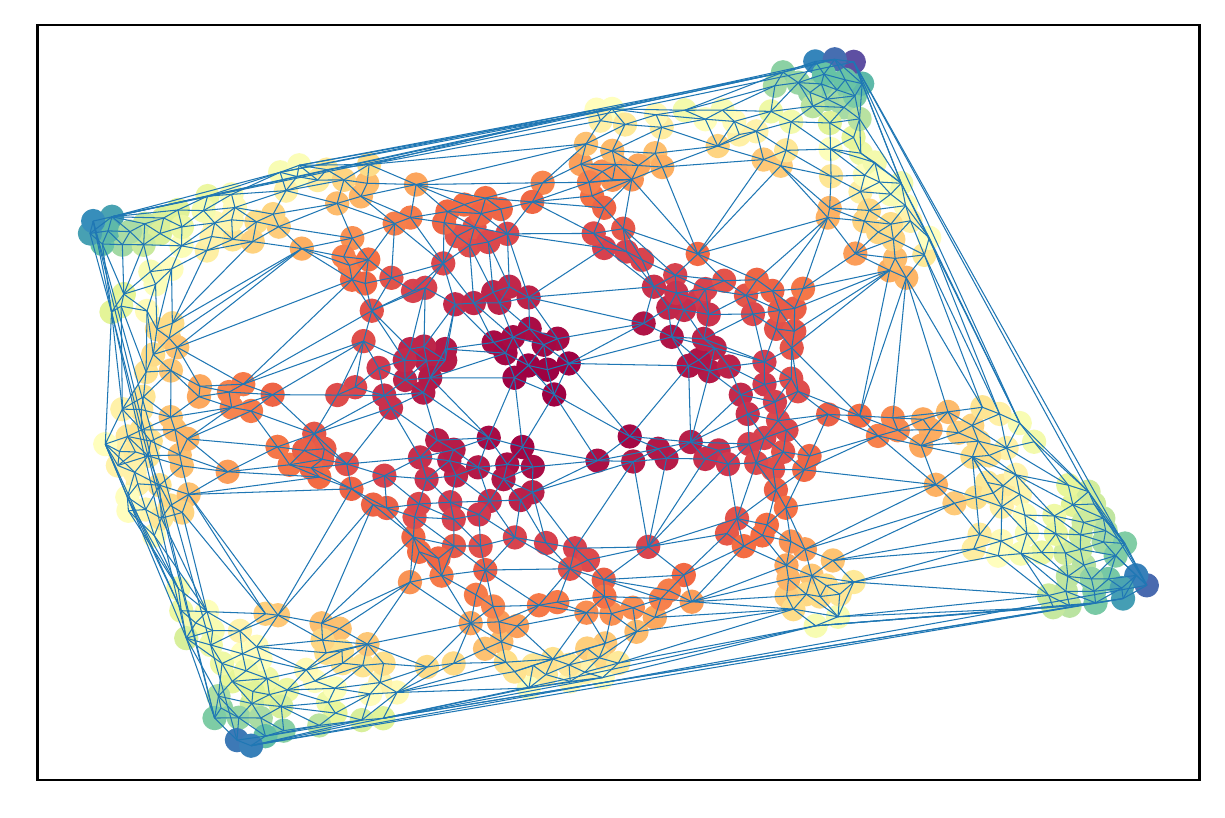}}\;\;\;
\caption{Other methods on Paraboloid: (a) 122 crosses,(b) 39 crosses, (c) 1468 crosses, (d) 309 cross, (e) 30 crosses, (f) 57 crosses, (g) 282 crosses}   
\label{other_methods_Paraboloid}
\end{figure}

Manifold: Twin peaks
\begin{figure}[H]
     \centering
     \subfloat[]{\includegraphics[width = 0.13\textwidth, height = 0.13\textwidth]{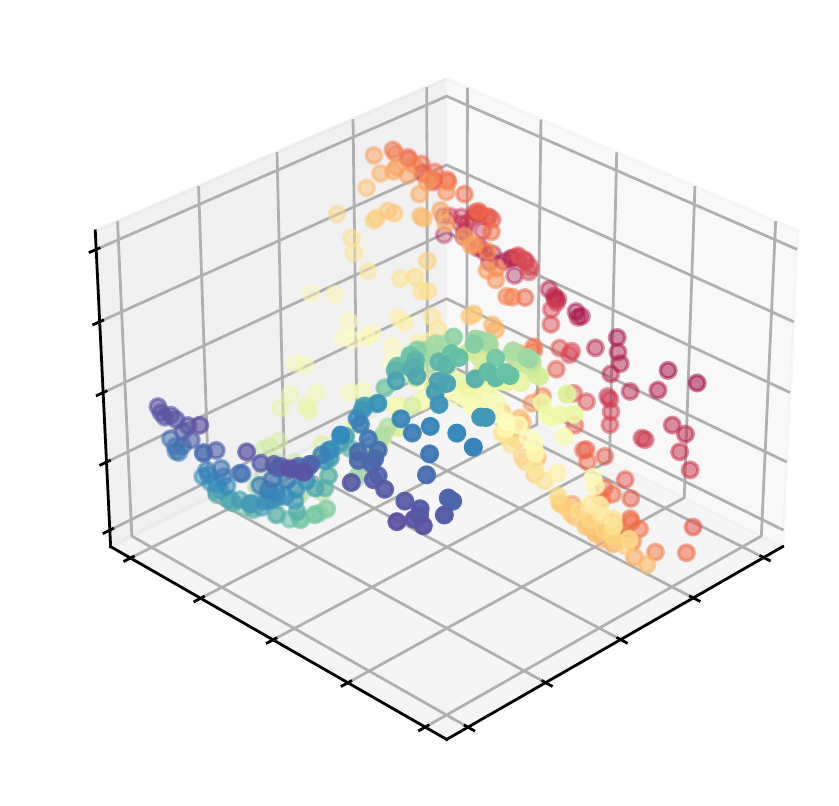}} \;\;\;
     \subfloat[]{\includegraphics[width =0.13\textwidth, height = 0.13\textwidth]{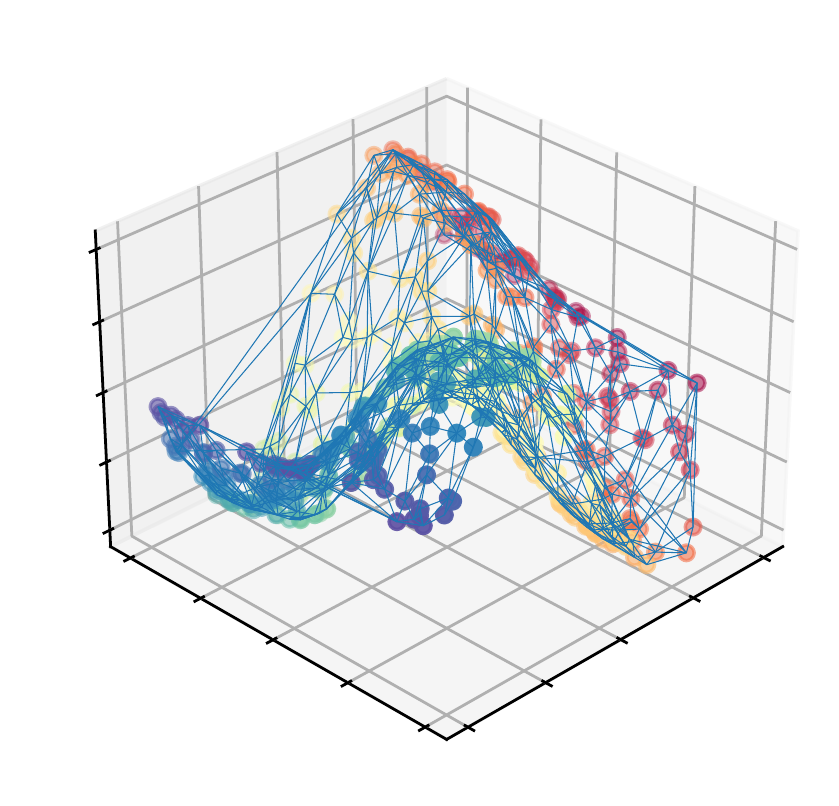}}\;\;\;
     \subfloat[]{\includegraphics[width =0.13\textwidth, height = 0.13\textwidth]{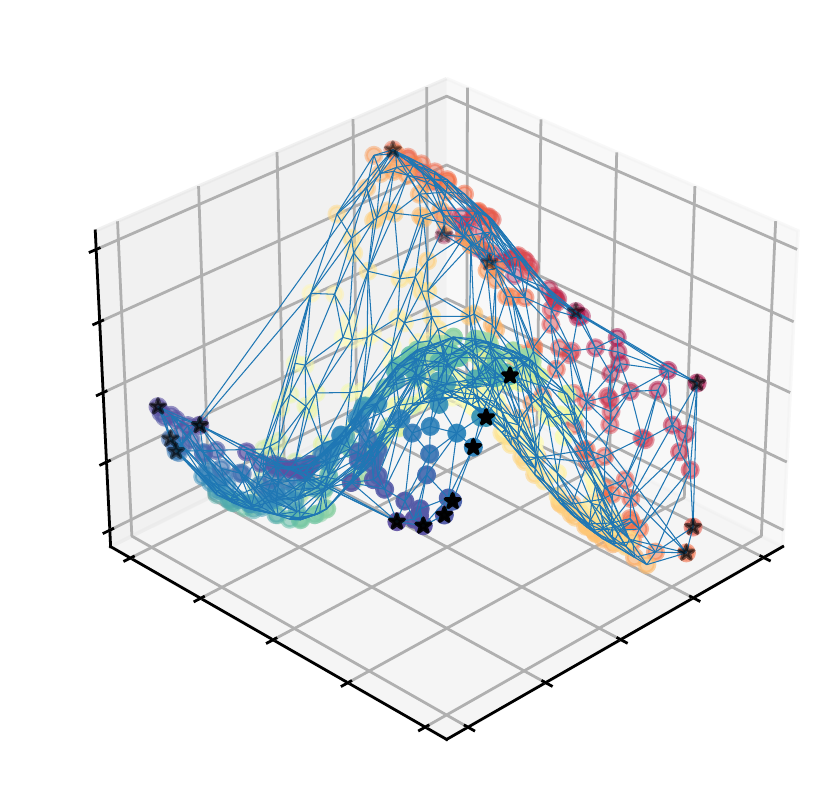}}\;\;\;
     \subfloat[]{\includegraphics[width =0.13\textwidth, height = 0.13\textwidth]{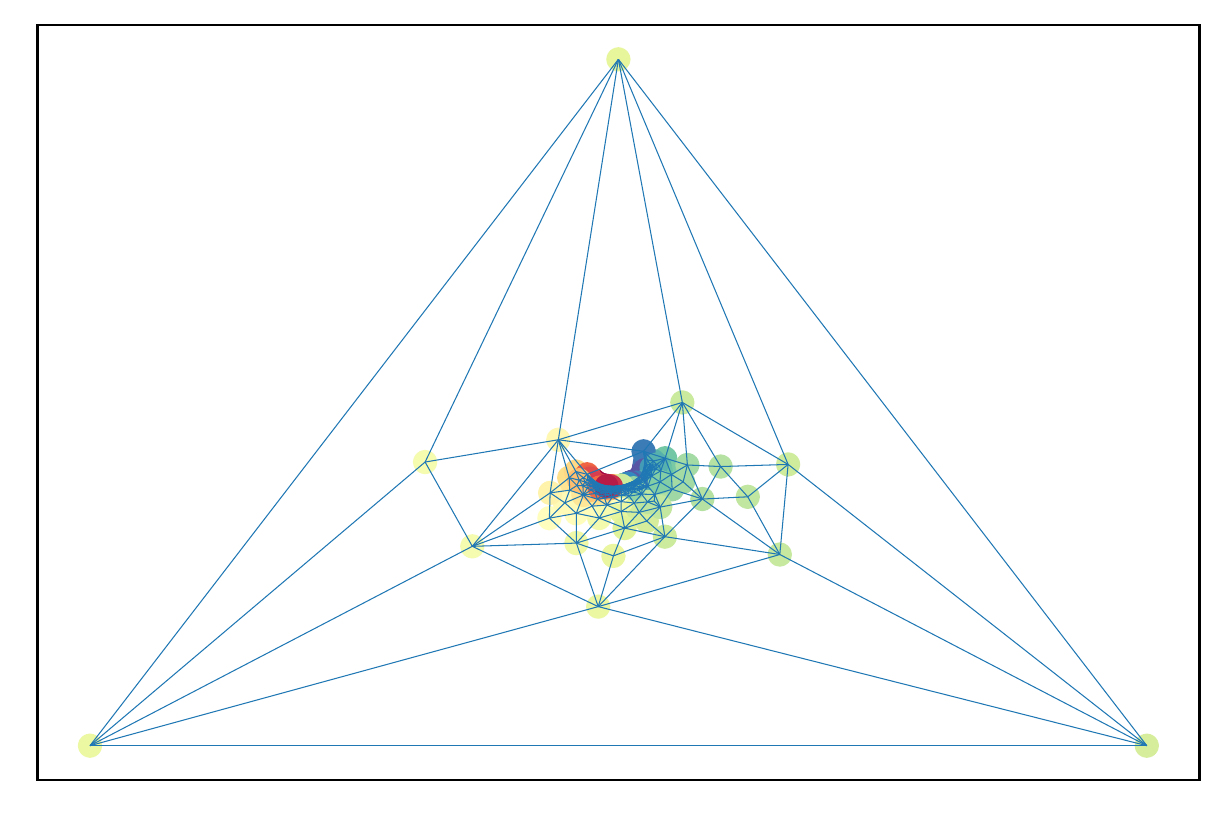}}\;\;\;
     \subfloat[]{\includegraphics[width =0.13\textwidth, height = 0.13\textwidth]{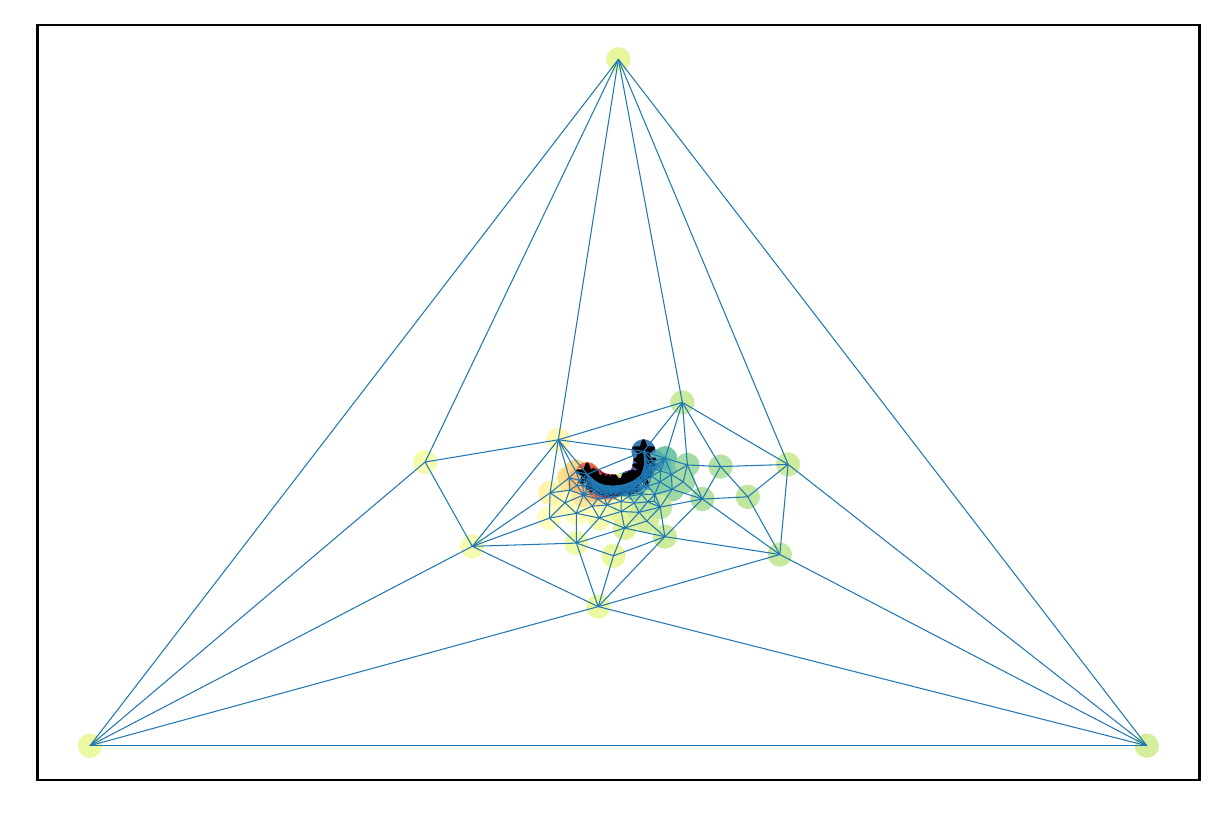}}\;\;\;
     \subfloat[]{\includegraphics[width =0.13\textwidth, height = 0.13\textwidth]{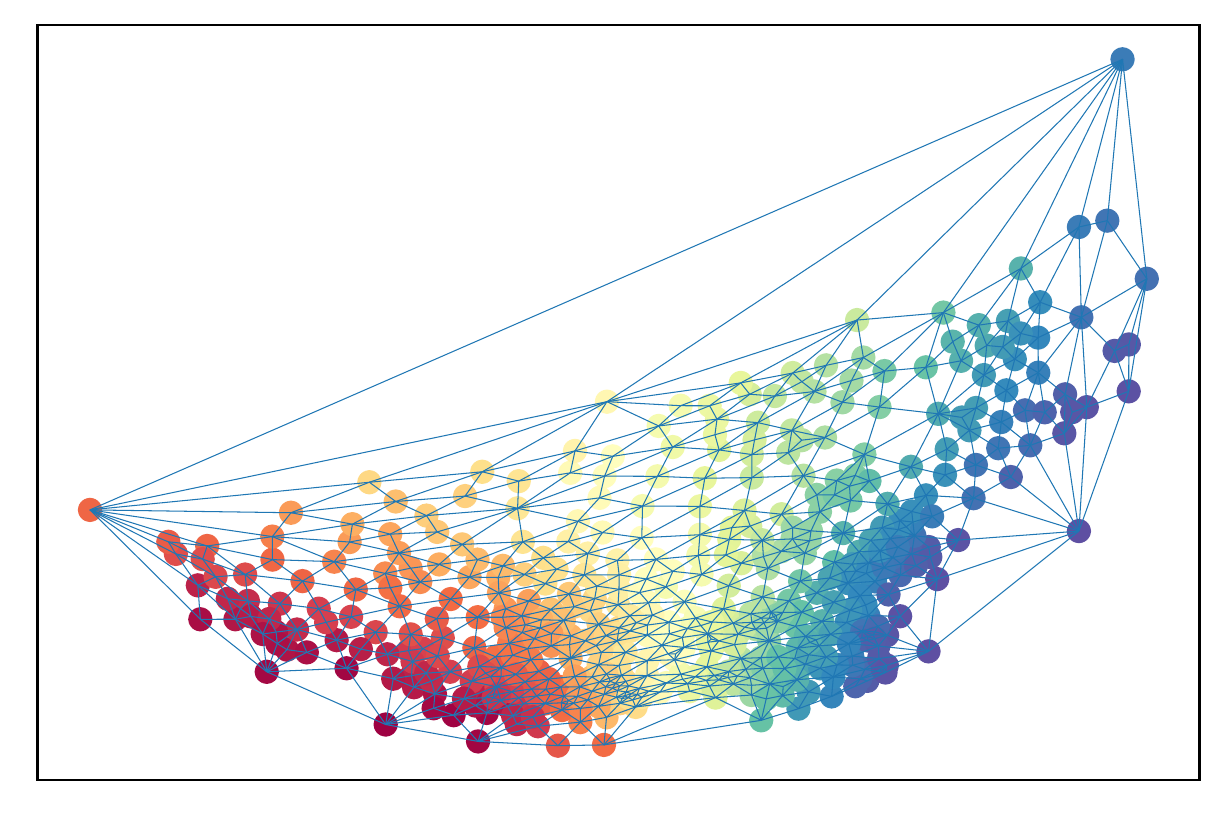}}\;\;\;
\caption{FPLM on Twinpeaks }   
\label{FPLM_twinpeaks}
\end{figure}

\begin{figure}[H]
     \centering
     \subfloat[AE ]{\includegraphics[width = 0.11\textwidth, height = 0.13\textwidth]{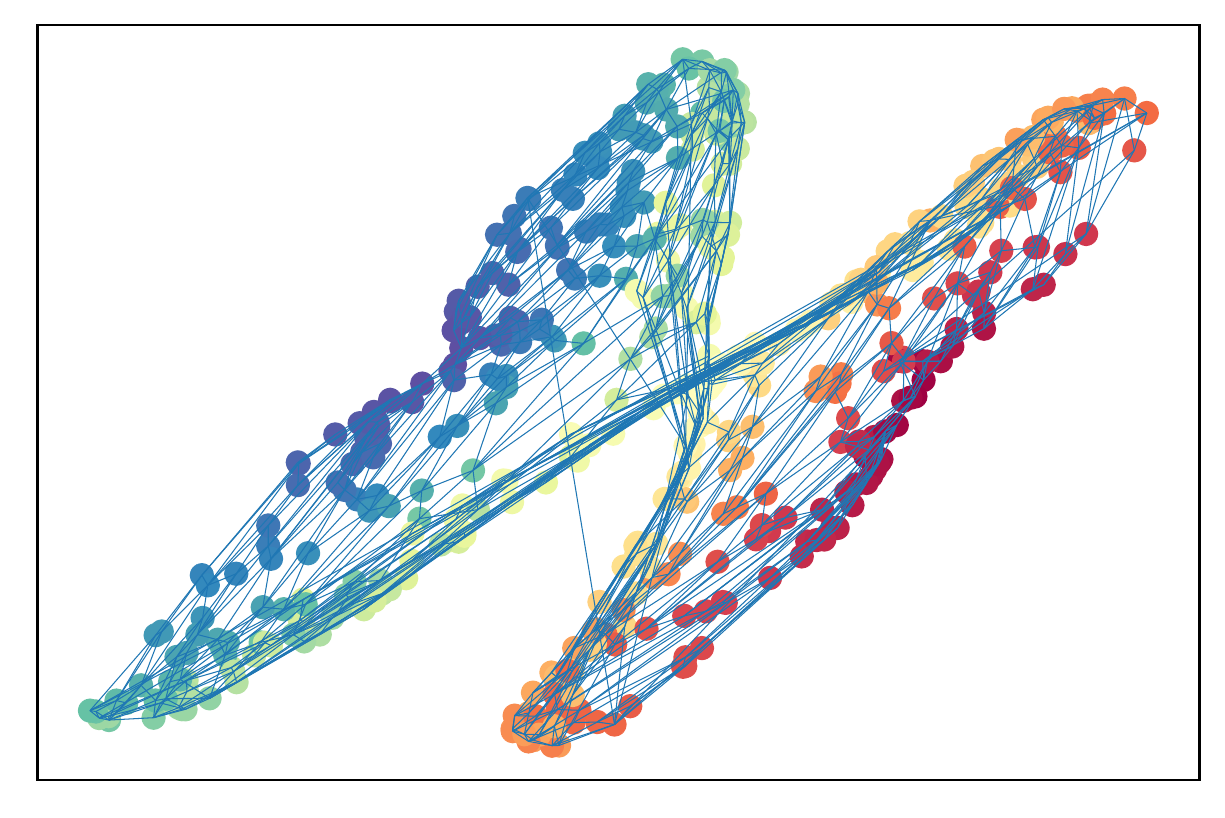}} \;\;\;
     \subfloat[Isomap ]{\includegraphics[width =0.11\textwidth, height = 0.13\textwidth]{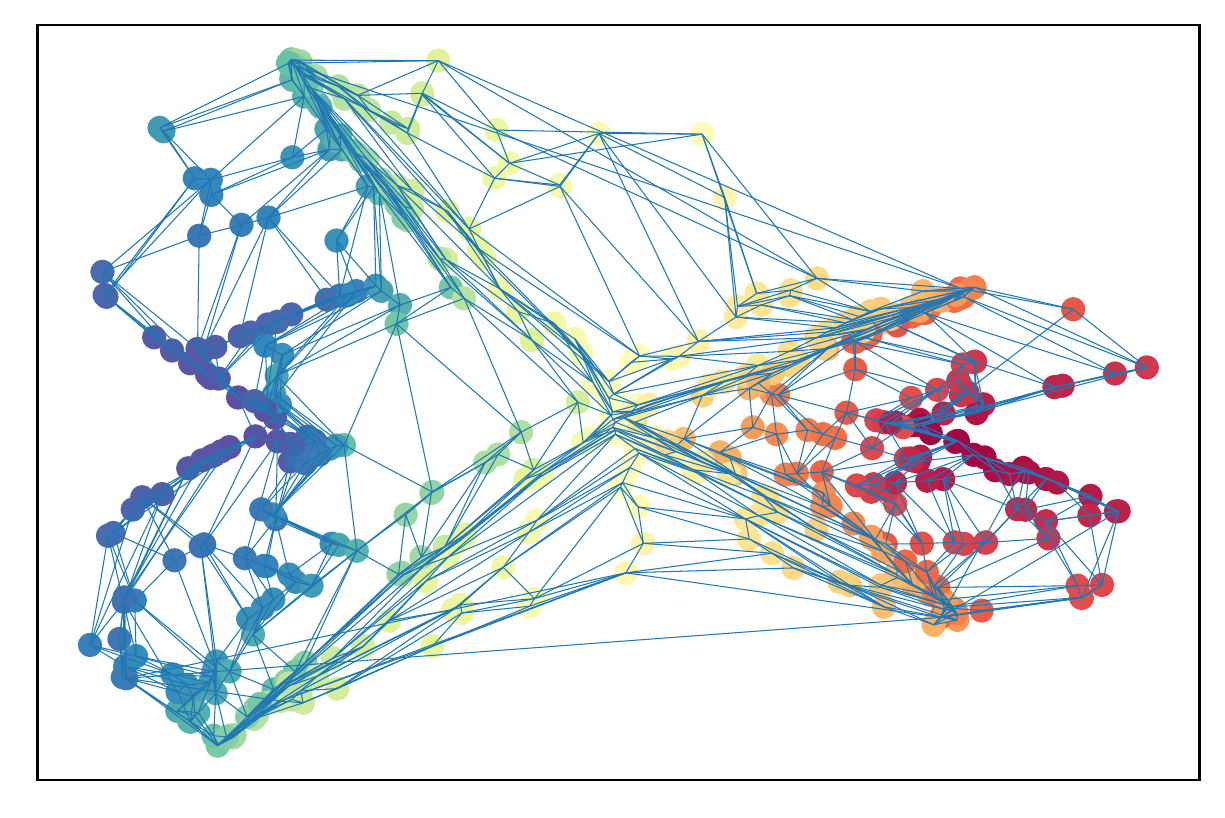}}\;\;\;
     \subfloat[LE ]{\includegraphics[width =0.11\textwidth, height = 0.13\textwidth]{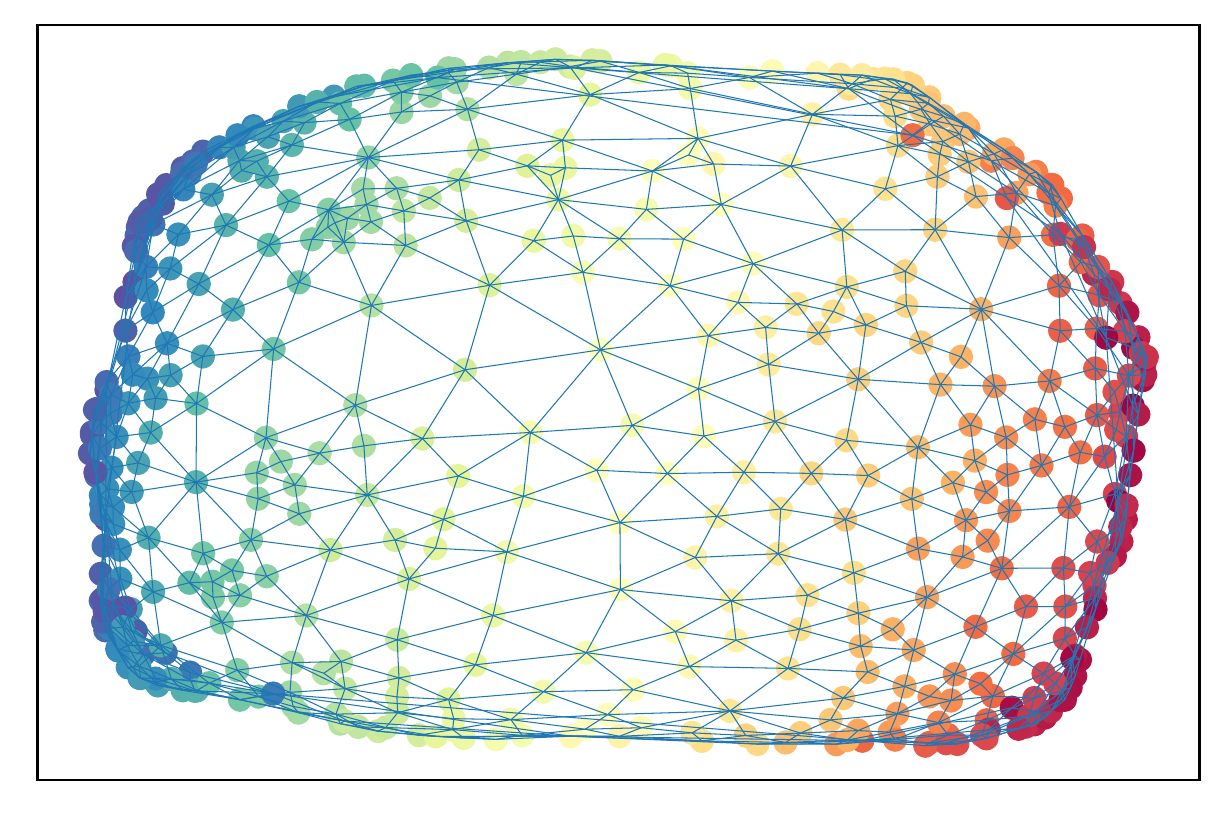}}\;\;\;
     \subfloat[LLE ]{\includegraphics[width =0.11\textwidth, height = 0.13\textwidth]{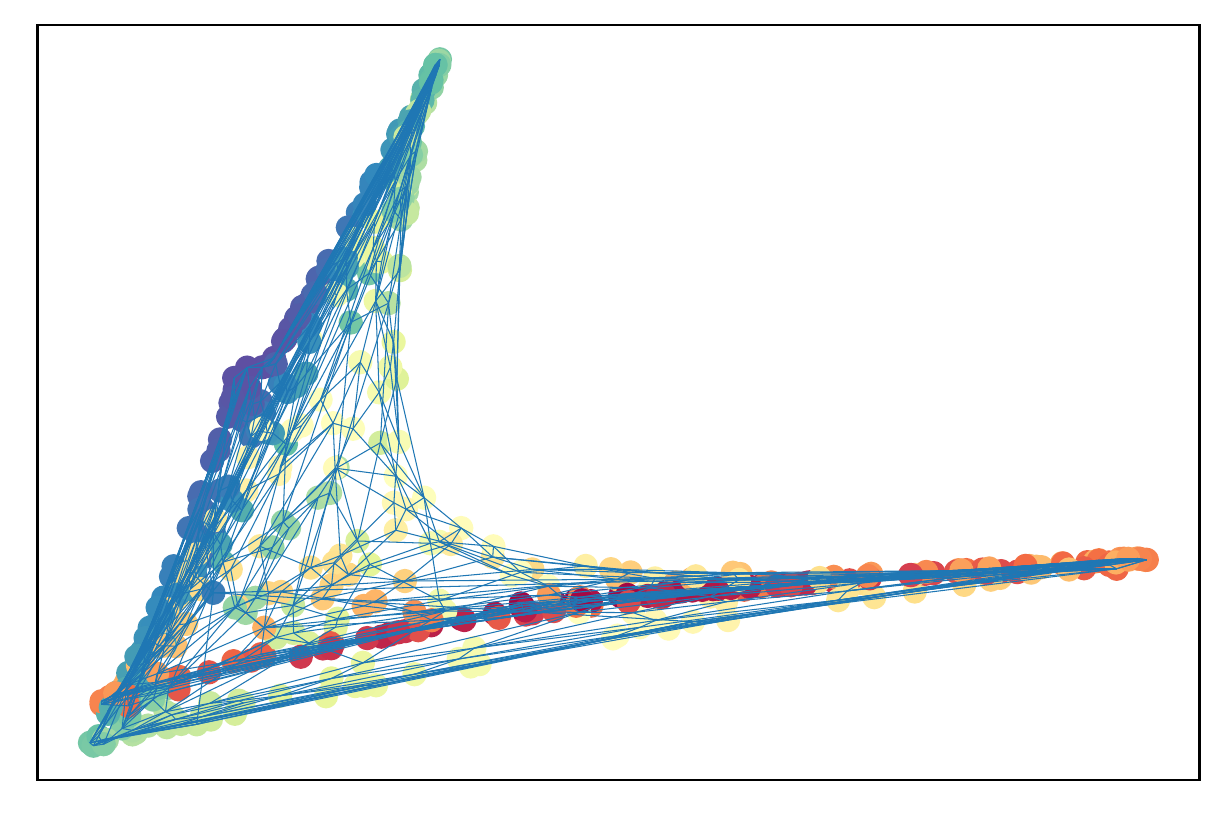}}\;\;\;
     \subfloat[LTSA ]{\includegraphics[width =0.11\textwidth, height = 0.13\textwidth]{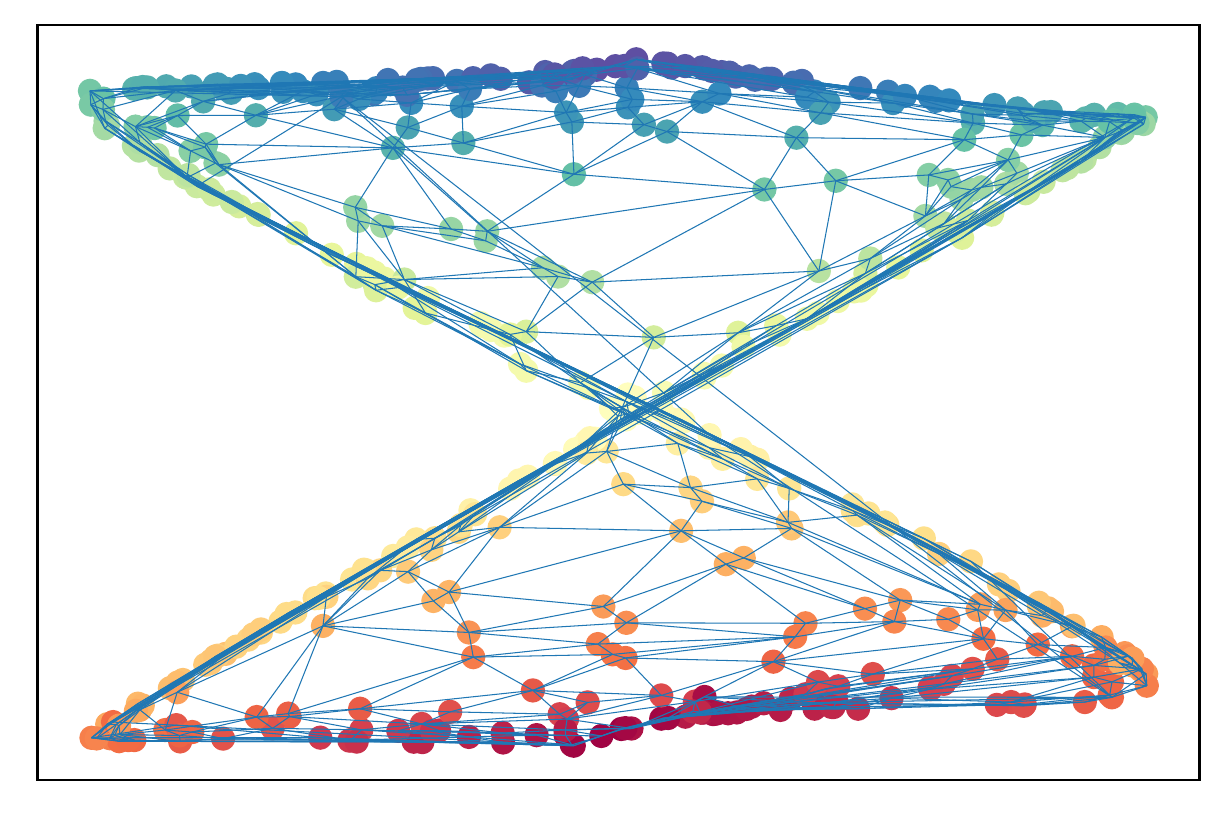}}\;\;\;
     \subfloat[MDS ]{\includegraphics[width =0.11\textwidth, height = 0.13\textwidth]{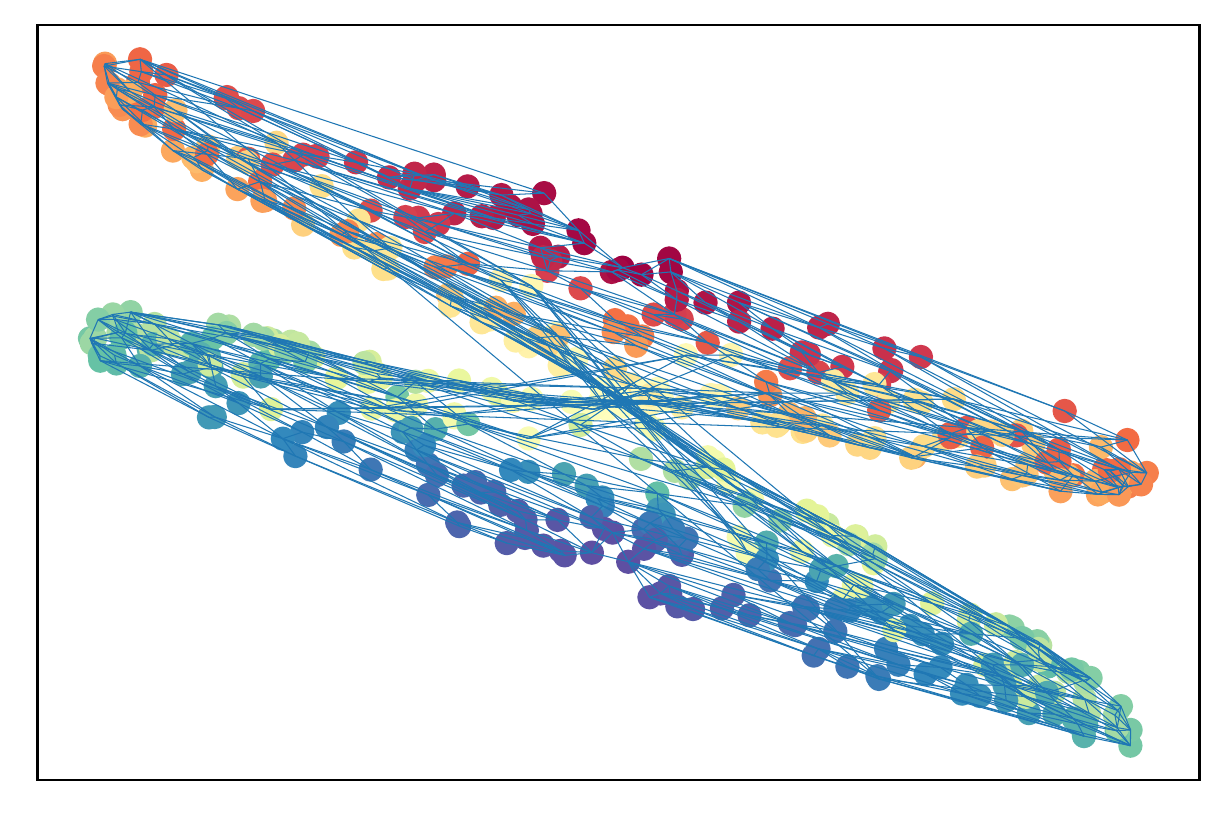}}\;\;\;
     \subfloat[t-SNE ]{\includegraphics[width =0.11\textwidth, height = 0.13\textwidth]{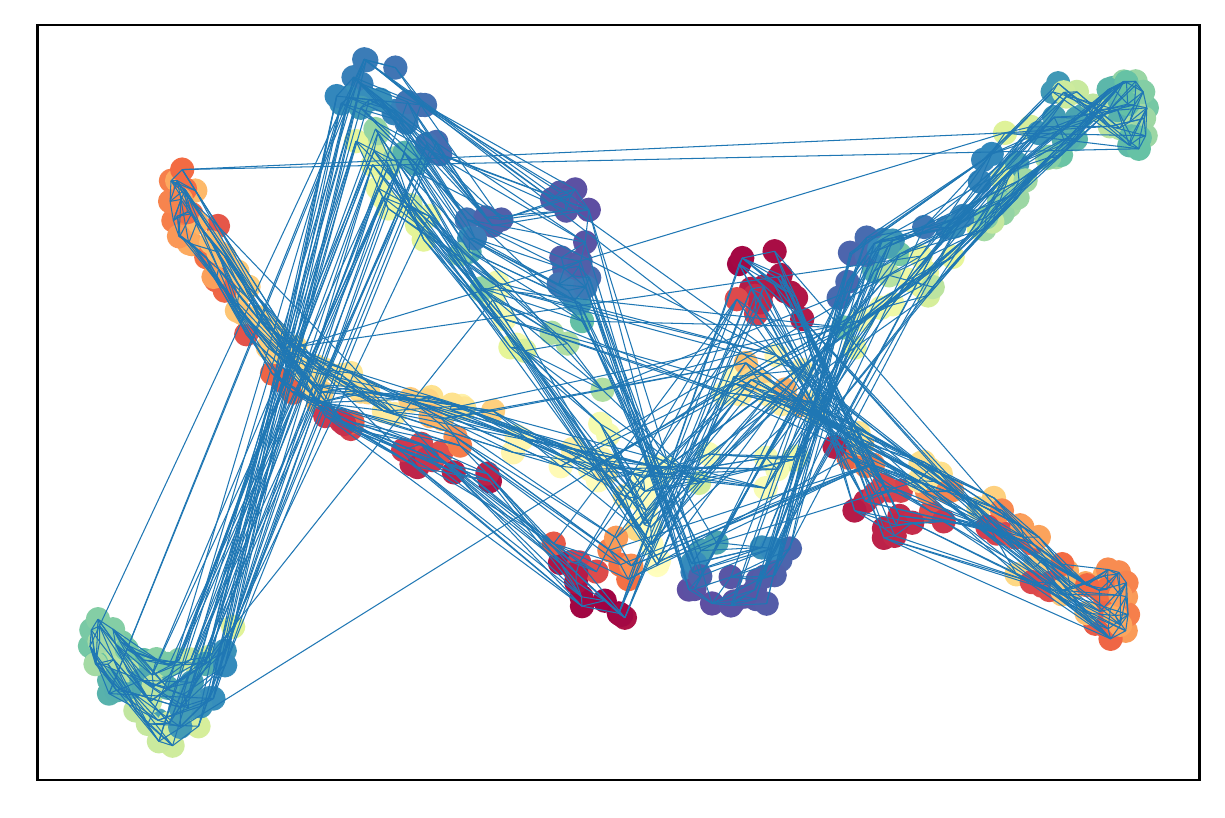}}\;\;\;
\caption{Other methods on Twinpeaks: (a) 3332 crosses,(b) 964 crosses, (c) 764 crosses, (d) 3751 cross, (e) 2976 crosses, (f) 3584 crosses, (g) 18282 crosses}   
\label{other_methods_twinpeaks}
\end{figure}


A summary of all manifolds included in the experiment and the number of line crosses generated from the methods other than FPLM are included in the following table: 
\begin{table}[H]
\centering
\begin{tabular}{|l|l|l|l|l|l|l|l|}
\hline
\textbf{Manifolds} &             &        & \multicolumn{3}{l|}{\textbf{Methods and Line   crosses}} &      &       \\ \hline
                  & Autoencoder & Isomap & LE                & LLE              & LTSA              & MDS  & TSNE  \\ \hline
Monkey Saddle      & 54          & 46     & 815               & 57               & 33                & 47   & 735   \\ \hline
Swiss Roll         & 4585        & 1942   & 937               & 3623             & 36773             & 3804 & 10088 \\ \hline
Sphere             & 770         & 1786   & 1683              & 1883             & 1795              & 1667 & 2329  \\ \hline
Twin Peaks         & 100         & 114    & 903               & 2806             & 696               & 91   & 235   \\ \hline
Paraboloid           & 56          & 39     & 1468              & 309              & 309               & 38   & 282   \\ \hline
\end{tabular}
\end{table}

\subsection{Additional result on tetrahedral meshes}
We additionally provide this result to show that FPLM can deal with large number of tetrahedral mesh in $\mathbb R^3$. Note that the boundary of manifold (i.e. in $\mathbb R^4$ or higher)will generally be different compared with the boundary detected in $\mathbb R^3$, since there are many types of embedding functions to map 3 dimensional tetrahedral mesh into $\mathbb R^4$. We select the famous SHARK tetrahedral mesh \cite{sullivan2019pyvista} that contains 17061 tetrahedrons to check the efficiency of FPLM. The results are as follows: 

\begin{figure}[H]
     \centering
     \subfloat[]{\includegraphics[width = 0.15\textwidth, height = 0.15\textwidth]{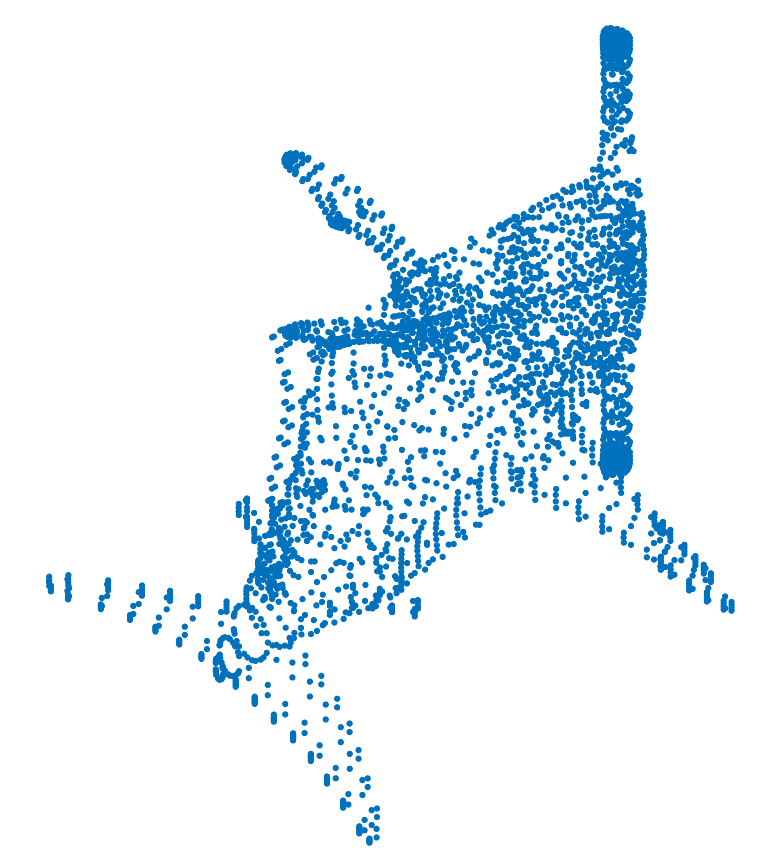}} \;\;\;
     \subfloat[]{\includegraphics[width =0.15\textwidth, height = 0.15\textwidth]{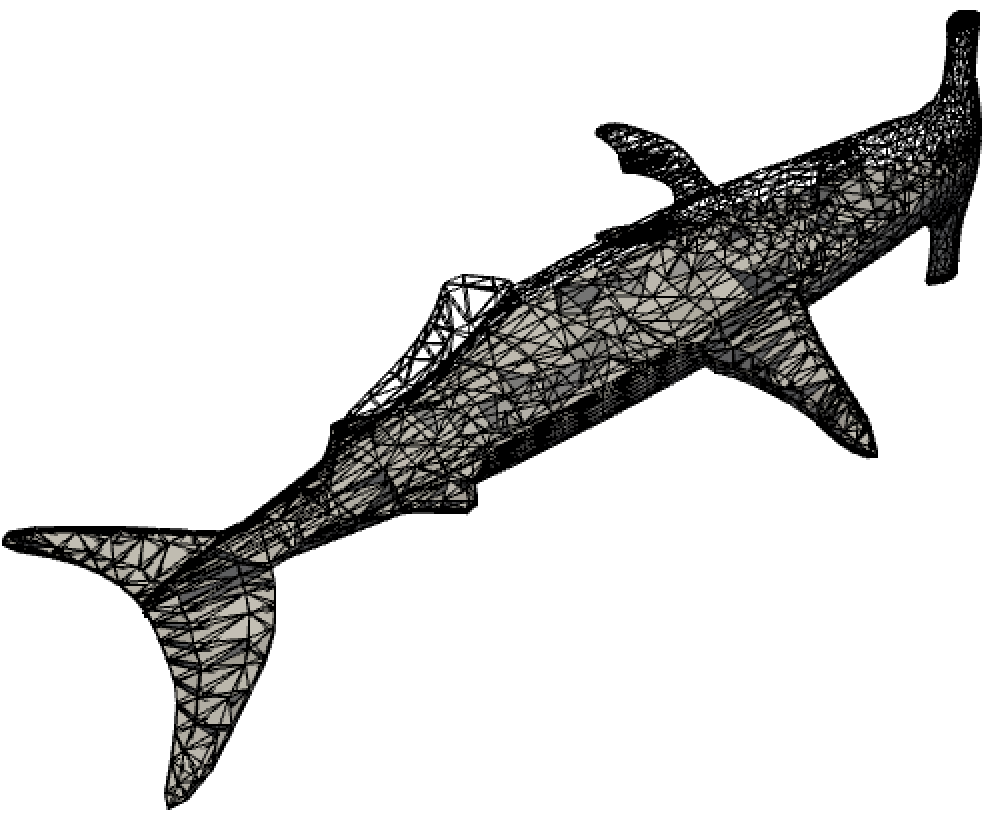}}\;\;\;
      \subfloat[]{\includegraphics[width =0.15\textwidth, height = 0.15\textwidth]{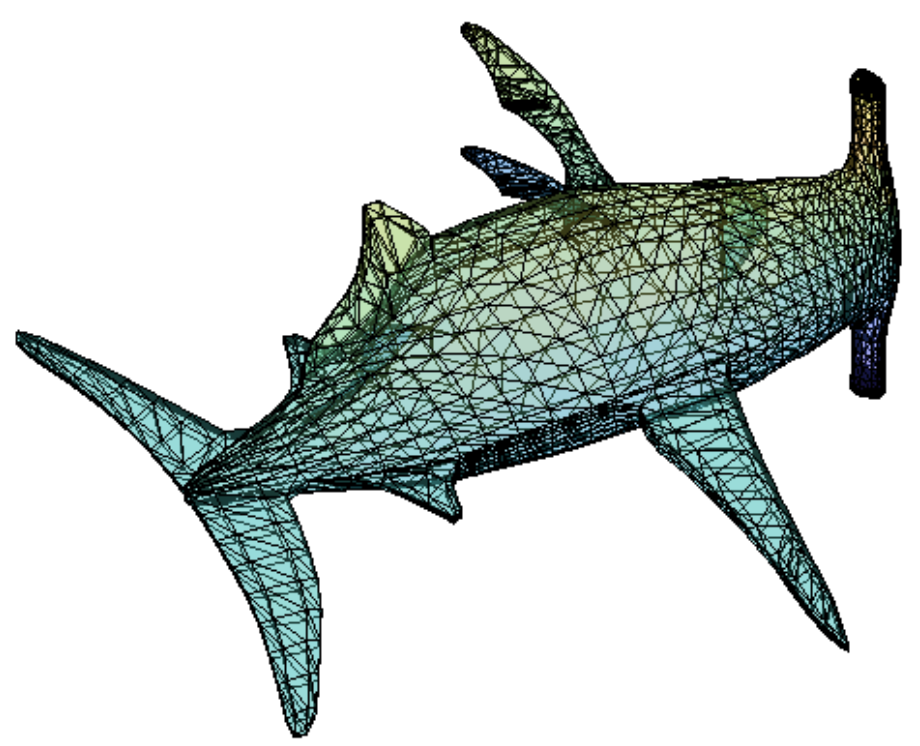}}\;\;\;
     \subfloat[]{\includegraphics[width =0.15\textwidth, height = 0.15\textwidth]{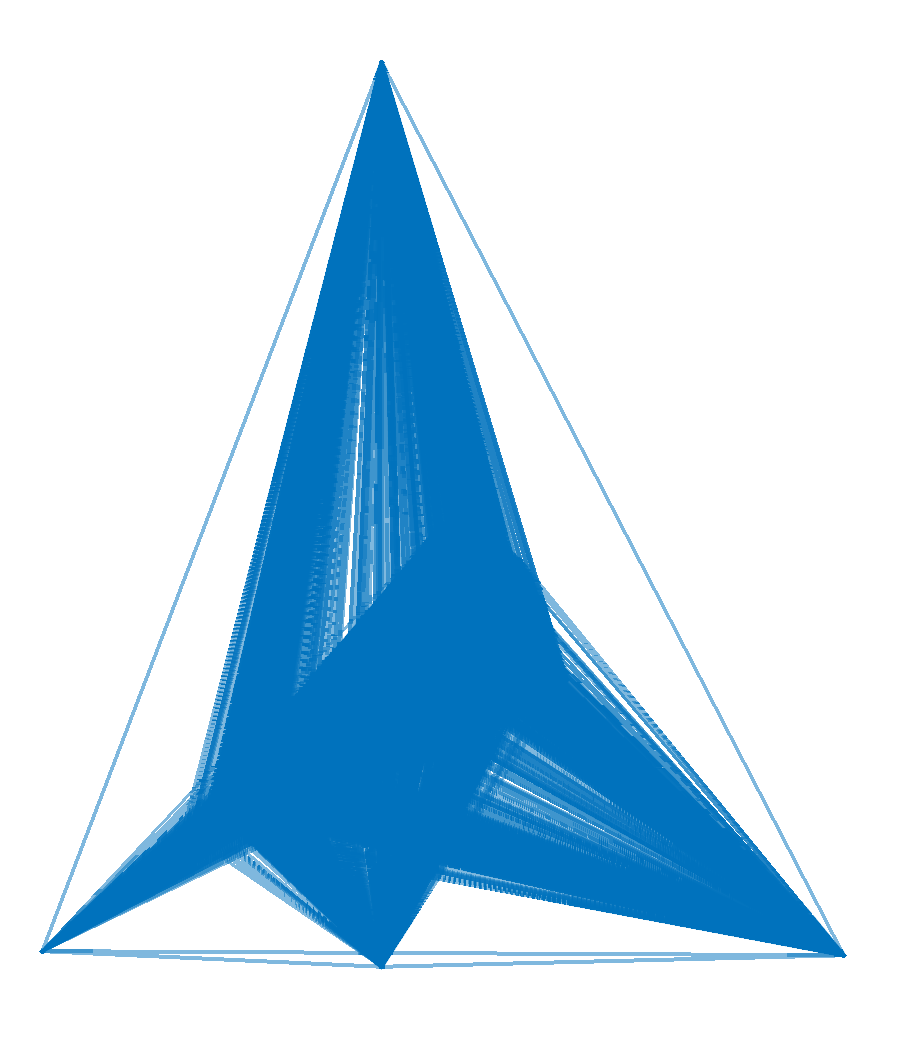}}\;\;\;
     \subfloat[]{\includegraphics[width =0.15\textwidth, height = 0.15\textwidth]{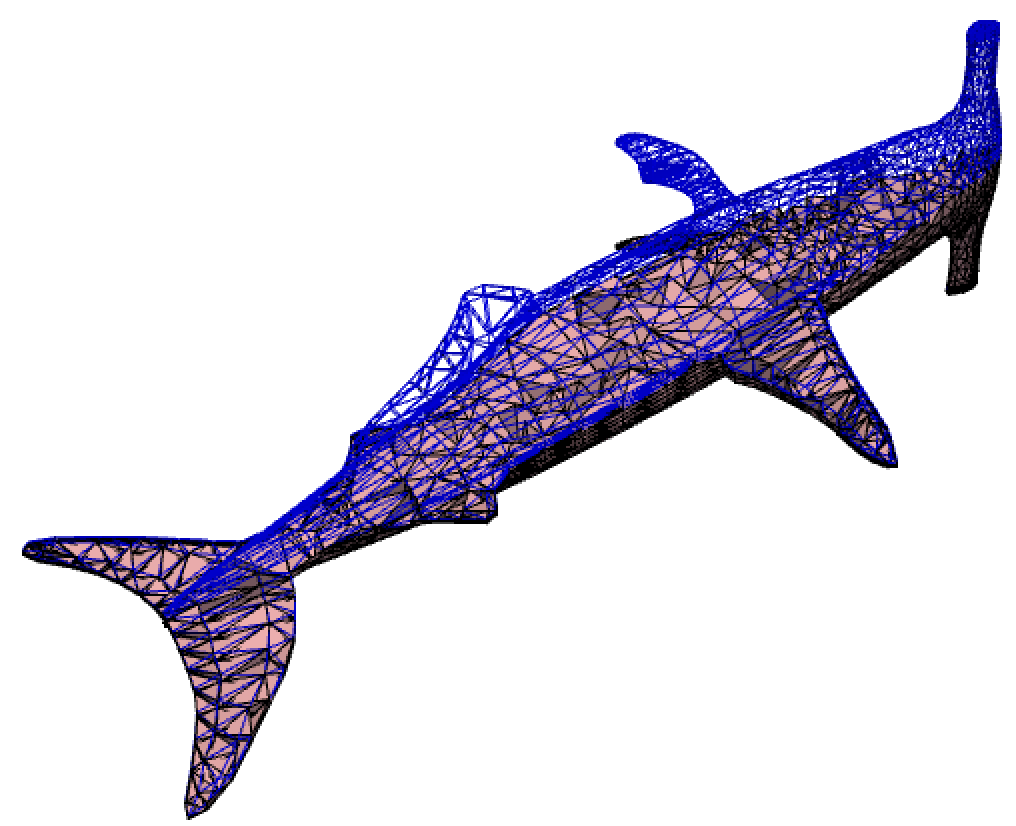}}\;\;\;
\caption{FPLM on shark sharp manifold example: 1. Point scatter 2. tetrahedralization on scatters 3. Boundary detection (faces) 4. First round FPLM 5. Second round FPLM. Total running time:38.5s}  
\label{FPLM_shark}
\end{figure}

\end{document}